\documentclass[11pt]{article}
\usepackage{fullpage}
\usepackage{comment}
\usepackage{caption}
\usepackage{subcaption}
\usepackage{xcolor}
\usepackage{multirow}
\usepackage{authblk}
\usepackage{dsfont}
\usepackage{mathtools}
\usepackage{amsthm,amsmath,amssymb}
\usepackage{bm}
\usepackage{bbm}

\usepackage[colorlinks,linkcolor=blue,citecolor=cyan]{hyperref}
\usepackage[capitalize,nameinlink]{cleveref}
\newtheorem{theorem}{Theorem}
\newtheorem{lemma}{Lemma}

\newtheorem{assumption}{Assumption}
\theoremstyle{definition}

\theoremstyle{remark}

\newcommand{\eps}{\epsilon}
\newcommand{\veps}{\varepsilon}

\newcommand{\Cov}{\mathrm{Cov}}
\newcommand{\OPT}{\mathrm{OPT}}
\newcommand{\CSBM}{\mathrm{CSBM}}
\newcommand{\Ber}{\mathrm{Ber}}
\newcommand{\R}{\mathbb{R}}
\newcommand{\Prob}{\mathds{P}}
\newcommand{\zero}{\mathbf{0}}
\newcommand{\one}{\mathbf{1}}
\newcommand{\bt}{\tilde{b}}
\newcommand{\at}{\tilde{a}}
\newcommand{\w}{\mathbf{w}}
\newcommand{\wt}{\tilde{\w}}
\newcommand{\uv}{\mathbf{u}}
\newcommand{\vv}{\mathbf{v}}
\newcommand{\yv}{\mathbf{y}}
\newcommand{\muv}{\bm{\mu}}
\newcommand{\nuv}{\bm{\nu}}
\newcommand{\Nc}{\mathcal{N}}
\newcommand{\At}{\tilde{A}}
\newcommand{\Xt}{\tilde{X}}
\newcommand{\ie}{{i.e.}}

\newcommand{\iid}{{i.i.d.}}

\newcommand{\poly}{{\rm poly}}
\newcommand{\E}{\mathds{E}}
\newcommand{\setc}{{\mathsf{c}}}
\newcommand{\inner}[1]{\left\langle #1 \right\rangle}
\newcommand{\train}{\text{\it train}}
\newcommand{\test}{\text{\it test}}
\newcommand{\bignorm}[1]{\left\lVert #1 \right\rVert}
\newcommand{\floor}[1]{\lfloor #1 \rfloor}
\newcommand\bigsubseteq[1][1.19]{%
   \mathrel{\vcenter{\hbox{\scalebox{#1}{$\subseteq$}}}}}

\newcommand{\abs}[1]{\left\lvert #1 \right\rvert}
\newcommand{\norm}[1]{\left\lVert #1 \right\rVert}
\newcommand{\bS}{\mathbb{S}}

\usepackage[normalem]{ulem}

\title{Graph Convolution for Semi-Supervised Classification: Improved Linear Separability and Out-of-Distribution Generalization}
\author[1]{Aseem Baranwal}
\author[1]{Kimon Fountoulakis}
\author[2]{Aukosh Jagannath}
\affil[1]{David R. Cheriton School of Computer Science, University of Waterloo, Waterloo, Canada}
\affil[2]{Department of Statistics and Actuarial Science, 
Department of Applied Mathematics, University of Waterloo, Waterloo, Canada}

\date{}

\begin{document}
\graphicspath{ {./images/} }

\maketitle
\frenchspacing

\begin{abstract}
Recently there has been increased interest in semi-supervised classification in the presence of graphical information.
A new class of learning models has emerged that relies, at its most basic level, on classifying the data after first applying a graph convolution.
To understand the merits of this approach, we study the classification of a mixture of Gaussians, where the data corresponds to the node attributes of a stochastic block model. We show that graph convolution extends the regime in which the data is linearly separable by a factor of roughly $1/\sqrt{D}$, where $D$ is the expected degree of a node, 
as compared to the mixture model data on its own. Furthermore, we find that the linear classifier obtained by minimizing the cross-entropy loss after the graph convolution  generalizes to out-of-distribution data where the unseen data can have different intra- and inter-class edge probabilities from the training data.
\end{abstract}

\section{Introduction}
\label{intro}
Semi-supervised classification is one of the most important topics in machine learning and artificial intelligence. Recently, researchers extended classification models to include relational information~\cite{HamilBook}, where relations are captured by a graph. The attributes of the nodes capture information about the nodes, while the edges of the graph capture relations among the nodes. 
The reason behind this trend is that many applications require 
the combination of both the graph and the node attributes, such as recommendation systems~\cite{YHCEHL18}, predicting the properties of compounds or molecules~\cite{gilmer:quantum, scarselli:gnn}, predicting states of physical objects~\cite{battaglia:graphnets}, and classifying types of nodes in knowledge graphs~\cite{kipf:gcn}.

The most popular models use graph convolution~\cite{kipf:gcn} where one averages the attributes of a node with those of its neighbors.\footnote{Other types of graph convolution exist, for simplicity we focus on averaging since it's one of the most popular.}
This allows the model to make predictions about a node
using the attributes of its neighbors instead of only using  the node's attributes.
Despite the common perception among practitioners~\cite{CLB19} that graph convolution can improve the performance of models for semi-supervised classification, we are not aware of any work that studies the benefits of graph convolution in improving classifiability of the data as compared to traditional classification methods, such as logistic regression, nor are we aware of work on its generalization performance on out-of-distribution data for semi-supervised classification. 

To understand these issues, we study the performance of a graph convolution on a simple classification model with node attributes that are correlated with the class information, namely semi-supervised classification for the contextual stochastic block model~\cite{BVR17,DSM18}.
The contextual stochastic block model (CSBM) is a coupling of the standard stochastic block model 
(SBM)~\cite{holland83stochastic} with a Gaussian mixture model.
In this model, each class in the graph corresponds to a different Gaussian component of the mixture model, which yields the distribution for the attributes of the nodes. For a precise definition of the model see \Cref{model}. The CSBM allows us to explore a range of questions related to linear separability and, in particular, to probe how various methods perform as one varies both the noise level of the mixture model, namely the distance between the means, and the noise level of the underlying graph, namely the difference between intra- and inter-class edge probabilities.
We focus here on the simple case of two classes where the key issues are particularly transparent. We expect that our methods apply readily to the multi-class setting (see \Cref{conclusion} for more on this).

\subsection{Previous work}
\label{prev_work}
Computer scientists and statisticians have taken a fresh perspective on semi-supervised classification by coupling the graph structure and node attributes, see, e.g., ~\cite{scarselli:gnn,CZY2011,GVB2012,DV2012,GFRT13,YML13,HYL17,JLLHZ19,pmlr-v97-mehta19a}. These papers focus largely on practical aspects of these problems and new graph-based machine learning models.

On the other hand, there is a vast body of theoretical work on unsupervised learning for stochastic block models, see, e.g., \cite{decelle2011asymptotic,massoulie2014community,mossel2018proof,mossel2015consistency,abbe2015community,abbe2015exact,bordenave2015non,deshpande2015asymptotic,montanari2016semidefinite,banks2016information,abbe2018proof}, as well as the recent surveys \cite{Abbe2018,moore2017csphysics}. 
More recently, there has been work on the related problem of unsupervised classification using the contextual stochastic block model~\cite{BVR17,DSM18}. In their work, \cite{DSM18} explore the fundamental thresholds for correctly classifying a macroscopic fraction of the nodes in the regime of linear sample complexity and large but finite degree. Furthermore they present a conjecture for the sharp threshold. Their study, however, is largely focused on the fundamental limits of unsupervised learning whereas the work here is focused on understanding the relative merits of graph convolutions over traditional learning methods for semi-supervised learning and is thus not directly comparable. 

Another line of work has been studying the power of graph convolution models to distinguish graphs~\cite{XHLJ19,GJJ20,A2020}, and the universality of models that use graph convolution~\cite{ALoukas2020}. In this last paper and the references therein, the authors study the expressive power of graph neural networks, i.e., the ability to learn a hypothesis set.
This, however, does not guarantee generalization for unseen data.
Another relevant work that also studies semi-supervised classification using the graphs generated by the SBM is~\cite{CLB19}. There, the authors show that all local minima of cross entropy are approximately global minima if the graphs follow an SBM distribution.  Their work, however, does not provide theoretical evidence for the learning benefits of graph convolution in improving linear separability of data, neither do they show generalization bounds for out-of-distribution data. More recently, \cite{yan:2021:two-sides} showed the changes in the mean and the variance of the data after applying a graph convolution, however, their theoretical analysis is limited to networks without non-linear activations and they primarily address oversmoothing and heterophily problems instead of linear separability.

\subsection{Our contributions}
Let us now briefly summarize our main findings. In the following, let $d$ be the dimension of the mixture model (the number of attributes of a node in the graph), $n$ the number of nodes, $p$ and $q$ the intra- and inter-class edge probabilities respectively, and $D \approx n(p+q)/2$ the expected degree of a node. In our analysis we find the following:
\begin{itemize}
    \item If the means of the mixture model are at most $O(1/\sqrt{d})$ apart, then the data from the mixture model is not linearly separable and the minimal value of the binary cross entropy loss on the sphere of radius $R$ is bounded away from $0$ in a way that depends quantitatively on this distance and the sizes of the labeled data-sets uniformly over $R>0$.\footnote{It is easy to see that this is essentially  sharp, that is, if the means are $\omega(\sqrt{\log d/d}))$ apart then the data is linearly separable.}  
    \item If the  means are at least $\tilde{\omega}(1/\sqrt{d\cdot D})$ apart, the graph is not too sparse ($p,q=\tilde{\omega}(\log^2(n)/n)$), and the noise level is not too large ($(p-q)/(p+q)=\Omega(1)$), then the graph convolution of the data is linearly separable with high probability.
    \item Furthermore, if these conditions hold, then the minimizer of the training loss achieves exponentially small binary-cross entropy even for out-of-sample data with high probability. 
    \item On the other hand, if the means are $O(1/\sqrt{d D})$, then the convolved data is not linearly separable as well and one obtains the same lower bound on the loss for a fixed radius as in the non-convolved setting.
\end{itemize}
In particular, we see that if the average degree is large enough, then there is a substantial gain in the scale on which the corresponding graph convolution can perform well as compared to logistic regression. On the other hand, it is important to note that if the noise level of the graph is very high and the noise level of the data is small then the graph convolution can be disadvantageous. This is also shown empirically in our experiments in \Cref{subsec:ood_synthetic,subsec:ood_real}.

The rest of the paper is organized as follows: we give a precise definition of the semi-supervised contextual stochastic block model in \Cref{model}. In \Cref{results} we present our results along with a discussion. In \Cref{experiments}, we present extensive experiments which illustrate our results. Next, in \Cref{degree-concentration} we note the elementary concentration results for the class sizes and the degrees of nodes, that are used throughout the proofs. We then provide the proof for the separability thresholds in \Cref{separability-thresholds} and the proof for out-of-sample generalization in \Cref{generalization}.

\subsection{The model}
\label{model}
In this section we describe the CSBM~\cite{DSM18}, which is a simple coupling of a stochastic block model with a Gaussian mixture model.

Let $(\veps_{k})_{k\in[n]}$ be \iid\ $\Ber(\frac12)$ random variables. Corresponding to these, consider a stochastic block model consisting of two classes
$C_{0}=\{i\in[n]:\veps_{i}=0\}$ and $C_{1}=C_{0}^\setc$ with inter-class edge probability $q$ and intra-class edge probability $p$ with no self-loops.
In particular, conditionally on $(\veps_{k})$ the adjacency matrix $A=(a_{ij})$ is Bernoulli with $a_{ij}\sim \Ber(p)$ if $i,j$ are in the same class  and $a_{ij}\sim \Ber(q)$ if they are in distinct classes. 
Along with this, consider $X\in\R^{n\times d}$ to be the feature matrix such that each row $X_{i}$ is an independent $d$-dimensional Gaussian random vector with
$X_{i}\sim N(\muv,\frac{1}{d}I)$ if $i\in C_{0}$ and
$X_{i}\sim N(\nuv,\frac{1}{d}I)$ if $i\in C_{1}$.
Here $\muv,\nuv\in\R^{d}$ are fixed vectors with $\norm{\muv}_2,\norm{\nuv}_2\leq1$ and $I$ is the identity matrix.
Denote by $\CSBM(n,p,q,\muv,\nuv)$ the coupling of a stochastic block model with a two component Gaussian mixture model with means $\muv,\nuv$ and covariance $\frac1dI$ as described above and we denote a sample by $(A,X)\sim \CSBM(n,p,q,\muv,\nuv)$.\footnote{We note here that, we could also have considered $\sigma^2 I$ instead of $I/d$, in which case all of our results still hold after rescaling the thresholds appropriately. For example, if we took $\sigma^2 = 1$, then the relevant critical thresholds for linear separability become $\norm{\muv-\nuv} \sim 1$ and $\norm{\muv-\nuv}\sim 1/\sqrt{D}$ for the mixture model and the CSBM respectively.}
Observe that the marginal distribution for $A$ is a stochastic block model and that the marginal distribution for X is a two-component Gaussian mixture model.
Finally, define $\At=(\at_{ij})=A + I$ and $D$, the diagonal degree matrix for $\At$ where $D_{ii} = \sum_{j\in[n]}\at_{ij}$ for all $i\in[n]$. Then the graph convolution of some data $X$ is given by $\tilde{X}=D^{-1}\At X$.

For parameters $\w\in\R^d$ and $b\in \R$, the label predictions are given by $\hat{\yv} = \sigma(D^{-1}\At X \w + b\one)$,
where $\sigma(x)=(1+e^{-x})^{-1}$ is the sigmoid function applied element-wise in the usual sense. Note that we work in the semi-supervised setting where only a fraction of the labels are available. In particular, we will assume that for some fixed $0<\beta_0,\beta_1\le \frac12$, the number of labels available for class $C_0$ is $\beta_0 n$ and for class $C_1$ is $\beta_1 n$. Let $S = \{i: y_i \text{ is available} \}$ so that $|S|=(\beta_0+\beta_1)n$. The loss function we use is the binary cross entropy,
\begin{equation}
    L(A,X,\w,b) = -\frac1{|S|}\sum_{i\in T} y_i \log \hat{y_i} + (1-y_i)\log (1-\hat{y_i}), \label{eq:BCE-loss}
\end{equation}
where $y_i$ is the given label of node $i$, and $\hat{y}_i$ is the predicted label of node $i$ (also, the $i$-th component of vector $\hat{\yv}$).
Observe that the binary cross-entropy loss used in Logistic regression can be written as $L(I,X,\w,b)$.

\subsection{Results}
\label{results}
In this paper we have two main results.
Our first result is regarding the relative performance of the graph convolution as compared to classical logistic regression. Here, there are two types of questions to ask. The first is geometric in nature, namely when is the data linearly separable with high probability?\
This is a statement about the fundamental limit of logistic regression for this data. The second is about the output of the corresponding optimization procedure, the minimizer of  \eqref{eq:BCE-loss}, namely whether or not it performs well in classifying out-of-sample data .

Note that the objective function, while convex, is non-coercive when the data are linearly separable. Therefore, we introduce a norm-ball constraint and consider the following problem:
\begin{equation}\label{eq:OPT}
    \OPT_d(A,X,R)=\min_{\substack{\norm{\w} \le R,\\b\in\R}} ~ L(A,X,\w,b),
\end{equation}
where $\norm{\cdot}$ is the $\ell_2$-norm.
The analogous optimization problem in the setting without graph structure, i.e., logistic regression, is then $\OPT_d(I,X,R)$.
We find that graph convolutions can dramatically improve the separability of a dataset and thus the performance of the regression. In particular, we find that by adding the graph structure to a dataset and using the corresponding convolution, i.e., working with $AX$ as opposed to simply $X$, can make a dataset linearly separable when it was not previously. 

Our second result is about the related question of generalization on out-of-distribution data. Here we take the optimizer, $(\w^*,b^*)$, of problem \eqref{eq:OPT} and we are interested in how well it classifies data coming from a CSBM with the same means but with a different number of nodes, $n'$, and different intra- and inter-class edge probabilities, $p'$ and $q'$ respectively. We find that $(\w^*,b^*)$ performs nearly optimally, even when the values of  $n'$, $p'$, and $q'$ are substantially different from those in the training set.

Let us now state our results more precisely.
 Given a sample $(A,X)\sim \CSBM(n,p,q,\muv,\nuv)$, we say that $(X_i)_{i=1}^n$ is \emph{linearly separable} if there is some unit vector
$\vv$ and scalar $b$ such that
$\inner{X_i,\vv}+b < 0$ for all $i\in C_0$ and $\inner{X_i,\vv} +b > 0$
for all $i\in C_1$, i.e., there is some half-space which correctly classifies the data. We say that $(\Xt_i)_{i=1}^n$ is linearly separable if the same holds for $\tilde{X}$. Let us now define the scaling assumptions under which we work. Define the following quantity:
\[
\Gamma(p,q)=\frac{p-q}{p+q}.
\]
\begin{assumption}\label{assumption:1}
We say that $n$ satisfies \emph{Assumption 1}  if 
\[
\omega(d \log d)\leq n\leq O(\mathrm{poly}(d)).
\]
\end{assumption}
\begin{assumption}\label{assumption:2}
We say that $(p,q)$ satisfies \emph{Assumption 2} if
\[
p,q=\omega(\log^2 (n)/n) \quad\text{ and }\quad\Gamma(p,q)=\Omega(1).
\]
\end{assumption}
\Cref{assumption:1} states that we have at least quasilinearly many samples (\ie, nodes) and at most polynomially many such samples in the dimension of the data.\footnote{The need for the $poly(d)$ upper bound is, heuristically, for the following simple reasons: if $n\sim \exp(Cd)$ for $C$ sufficiently large then the dataset will hit essentially any point in the support of the two Gaussians, even large deviation regions. As such since there should be a large number of points from either community which will lie on the ``wrong'' side of any linear classifier. In particular, our arguments will apply if we relax this assumption to taking $n$ to be subexponential in $d$.} \Cref{assumption:2} states that the CSBM is not too sparse but such that there is a notable difference between the amount of edges within a class as opposed to between different classes. Assumptions of this latter type are similar to those in the stochastic block model literature, see, e.g., \cite{Abbe2018}.

Finally, let $\mathbb{B}^d=\{x\in\R^d:\norm{x}\leq 1\}$ denote the unit ball, let $\Phi(x)$ denote the cumulative distribution function of a standard Gaussian. We then have the following.

\begin{theorem}\label{thm:thresholds}
Suppose that $n$ satisfies \Cref{assumption:1} and that $(p,q)$ satisfies \Cref{assumption:2}. Fix $0<\beta_0,\beta_1\leq 1/2$ and let $\muv,\nuv\in \mathbb{B}^d$. For any $(A,X)\sim \CSBM(n,p,q,\muv,\nuv)$, we have the following:
\begin{enumerate}
    \item For any $K\geq 0$ if $\norm{\muv-\nuv}\leq K/\sqrt{d}$, then there are some $C,c>0$ such that for $d\geq 1$ 
    \[
    \Prob((X_i)_{i\in S} \text{ is linearly separable})\leq C\exp(-cd).
    \]
    Furthermore, for any $t>0$  there is a $c>0$ such that for every $R>0$,
    \[
    \begin{aligned}
     \OPT_d(I,X,R)
    \geq
    2(\beta_0\wedge\beta_1)\Phi\left(-\frac{K}{2}(1+t)\right)\log(2)
    \end{aligned}
    \]
    with probability $1-\exp(-c d)$.
    \item If  $\norm{\muv-\nuv} = \omega(\frac{\log n}{\sqrt{dn(p+q)/2}})$, then 
    \[
    \Prob((\tilde{X}_i)_{i\in S} \text{ is linearly separable}) = 1-o_d(1),
    \]
    where $o_d(1)$ denotes a quantity that converges to $0$ as $d\to\infty$.
    Furthermore, with probability $1-o_d(1)$, we have for all $R>0$
    \[
    \OPT_d(A,X,R) \leq \exp\left(-\frac{R}{2}\Gamma(p,q) \norm{\muv-\nuv}(1-o_d(1))\right).
    \]
    \item Consider a dataset with $N$ CSBM examples drawn independently from the same model, with each example associated with a graph with $n$ nodes. Assume further that $N\log n = \omega(d\log d)$ and that $nN = O(\poly(d))$. Then if the distance between the means $\norm{\muv-\nuv}\leq K/\sqrt{dn(p+q)/2}$ for some constant $K$, then
    \[
    \Prob((\tilde{X}_{k,i})_{i\in S_k} \text{ is linearly separable}\;\forall k\in [N])= o_d(1).
    \]
\end{enumerate}
\end{theorem}

Let us briefly discuss the meaning of \Cref{thm:thresholds}. The first part of this theorem shows that if we consider a two-component mixture of Gaussians in $\R^d$ with the same variances but different means, then if the means are $O(1/\sqrt{d})$ apart, it is impossible to linearly separate the data and the minimal loss is order $1$ with high probability.
For the second part we find that the convolved data, $\tilde{X}=D^{-1}\tilde{A}X$, is linearly separable provided the means are a bit more than $\Omega(1/\sqrt{d(n(p+q)/2)})$ apart and furthermore, on this scale the loss decays exponentially in $R\norm{\muv-\nuv}\Gamma$. Consequently, as $n(p+q)/2$ is diverging this regime contains the regime in which the data $(X_i)$ is not linearly separable and logistic regression fails to classify well. We note here that our arguments show that this bound is essentially sharp, provided $R$ is chosen to be at least $\Omega(\sqrt{d (n(p+q))/2)}$. Finally the third part shows that, analogously, the convolved data is not linearly separable below the $1/\sqrt{d n(p+q)/2}$ threshold.

We note here that these results hold here under Assumption 2, and in particular, under the assumption of $\Gamma(p,q)=\Omega(1)$. This is to be compared to the work on community detection for stochastic block models  and CSBMs \cite{abbe2015exact,mossel2015consistency,massoulie2014community,mossel2018proof,DSM18}  where the sharp threshold is at $(p-q)\Gamma(p,q)=1$. Those works, however, are for the (presumably) harder problems of unsupervised learning and hold in a much sparser regime.

Let us now turn to the related question of generalization. Here we are interested in the performance of the optimizer of \eqref{eq:OPT}, call it $(\w^*,b^*)$ on out-of-distribution data and, in particular, we are interested in an upper bound on the loss achieved with respect to new data $(A',X')$.
We find that the graph convolution performs well on any out-of-distribution example. In particular, given that the attributes of the test example are drawn from the same distribution as the attributes of the training sample, the graph convolution makes accurate predictions with high probability even when the graph is sampled from a different distribution. More precisely, we have the following theorem.
\begin{theorem}\label{thm:generalization}
Suppose that $n$ and $n'$ satisfy \Cref{assumption:1}. Suppose furthermore that the pairs $(p,q)$ and $(p',q')$ satisfy \Cref{assumption:2}. Fix $0<\beta_1,\beta_2\leq 1/2$ and $\muv,\nuv\in \mathbb{B}^d$.  
Let $(A,X)\sim\CSBM(n,p,q,\muv,\nuv)$. Let $(\w^*(R),b^*(R))$ be the optimizer of \eqref{eq:OPT}.
Then for any sample $(A',X')\sim\CSBM(n',p',q',\muv,\nuv)$ independent of $(A,X)$, there is a $C>0$ such that with probability $1-o_d(1)$ we have that for all $R>0$
\[
L(A',X',\w^*(R),b^*(R)) \le C\exp\Big(-\frac{R}{2}\norm{\muv-\nuv} \Gamma(p',q')(1 - o(1))\Big)
\]
where the loss \eqref{eq:BCE-loss} is with respect to the full test set $S=[n']$.
\end{theorem}

Let us end by noting here that while we have stated our result for generalization in terms of the binary-cross entropy, our arguments immediately yield that the number of nodes misclassified by the half-space  classifier defined by $(\w^*,b^*)$ must vanish with probability tending to 1.

\subsection{Proof sketch for \texorpdfstring{\Cref{thm:thresholds,thm:generalization}}{Theorems 1 and 2}}
We now briefly sketch the main ideas of the proof of \Cref{thm:thresholds,thm:generalization}. Let us start with the first.
To show that the data $(X_i)_{i=1}^n$ is not linearly separable, we observe that we can decompose the data in the form 
\[
X_i = (1-\varepsilon_i)\muv+\varepsilon_i\nuv + \frac{Z_i}{\sqrt{d}},
\]
where $Z_i\sim N(\zero,I)$ are \iid. The key observation is that when the means are $O(1/\sqrt{d})$ apart
then the intersection of the high probability regions of the two components of the mixture is most of the mass of both, so that no plane can separate the high probability regions.\footnote{We expect that the sharp threshold here is when the distance between the means is $K\sqrt{\log{d}/d}$ for some $K$.} To make this precise, consider the Gaussian processes, $g_i(\vv) = \inner{Z_i,\vv}$. Linear separability can be reduced to showing that for some unit vector $\vv$, either 
the maximum of $g_i(\vv)$ for $i\in S_0$ or the minimum of $g_i(v)$ for $i\in S_1$
is bounded above or below respectively by an order 1 quantity over the entire sphere. This is exponentially unlikely by direct calculation using standard concentration arguments via an $\epsilon-$net argument. In fact, this calculation also shows that for $0<t<\Phi(-K/2)$, every hyperplane misclassifies at least $nt$ of the data points from each class with high probability, which yields the corresponding loss lower bound.

For the convolved data, the key observation is that 
\[
\tilde{X}_i \approx
\begin{cases}
\frac{p \muv +q\nuv}{p+q} + \frac{Z_i}{\sqrt{d D_{ii}}} &i\in C_0\\
\frac{q\muv+p\nuv}{p+q} + \frac{Z_i}{\sqrt{d D_{ii}}} & i\in C_1.
\end{cases}
\]
From this we see that, while the means move closer to each other by a factor of $(p-q)/(p+q)$, the variance has reduced by a factor of
$D_{ii}\approx (n(p+q)/2)^{-1}$.
This lowers the threshold for non-separability by the same factor. Consequently, if the distance between the means is a bit larger than $1/\sqrt{dn(p+q)/2}$ apart then  we can separate the data by the plane through the mid-point of the two means whose normal vector is the direction vector from $\muv$ to $\nuv$ with overwhelming probability. More precisely, it suffices to take as ansatz $(\wt,\bt)$ given by
\[
\wt \propto \nuv -\muv\qquad
\bt = \inner{\muv+\nuv,\wt}/2.
\]
To obtain a training loss upper bound, it suffices to evaluate $L(A,X,\wt,\bt)$. A direct calculation shows that this decays exponentially fast with rate $-R\Gamma\norm{\muv-\nuv}/2$ . 

Let us now turn to \Cref{thm:generalization}. The key point here is to observe that the preceding argument in fact shows two things. Firstly, the optimizer of the training loss, $\w^*$, must be close to this ansatz and the corresponding $b^*$ must be such that the pair $(\w^*,b^*)$ separates the data better than the ansatz. Secondly, the ansatz we chose does not depend on the particular values of $p$ and $q$. As such, it can be shown that $(\wt,\bt)$ performs well on out-of-distribution data corresponding to different values of $p'>q'$. Combining these two observations then shows that $(\w^*,b^*)$ also performs well on the out-of-distribution data.

\section{Experiments}
\label{experiments}

In this section we provide experiments to demonstrate our theoretical results in \Cref{results}. To solve problem~\eqref{eq:OPT} we used CVX, a package for specifying and solving convex programs~\cite{cvx,GB08}. 
Throughout the section we set $R=d$ in~\eqref{eq:OPT} for all our experiments.

\subsection{Training and test loss against distance of means}\label{subsec:train_loss_vs_distance}
In our first experiment we illustrate how the training and test losses scale as the distance between the means increases from nearly zero to $2/\sqrt{d}$. Note that according to Part 1 and Part 3 of \Cref{thm:thresholds}, $1/\sqrt{dn(p+q)}$ and $1/\sqrt{d}$ are the thresholds for the distance between the means, below which the data with and without graph convolution are not linearly separable with high probability, respectively. For this experiment we train and test on CSBM samples with $p=0.5$, $q=0.1$, $d=60$, and $n=400, N=10$ so that $n$ is roughly equal to $0.85\cdot d^{3/2}$,
and each class has $200$ nodes. We present results averaged over $10$ trials for the training data and $10$ trials for the test data. This means that for each value of the distance between the means we have $100$ combinations of train and test data. The results for training loss are shown in \Cref{fig:1a} and the results of the test loss are shown in \Cref{fig:1b}. We observe that graph convolution results in smaller training and test loss when the distance of the means is larger than $\log n / \sqrt{d n(p+q)} \approx 0.035$, which is the threshold such that graph convolution is able to linearly separate the data (Part 2 of \Cref{thm:thresholds}). 
\begin{figure}[ht!]
    \centering
	\begin{subfigure}[t]{3in}
		\centering
		\includegraphics[width=\columnwidth]{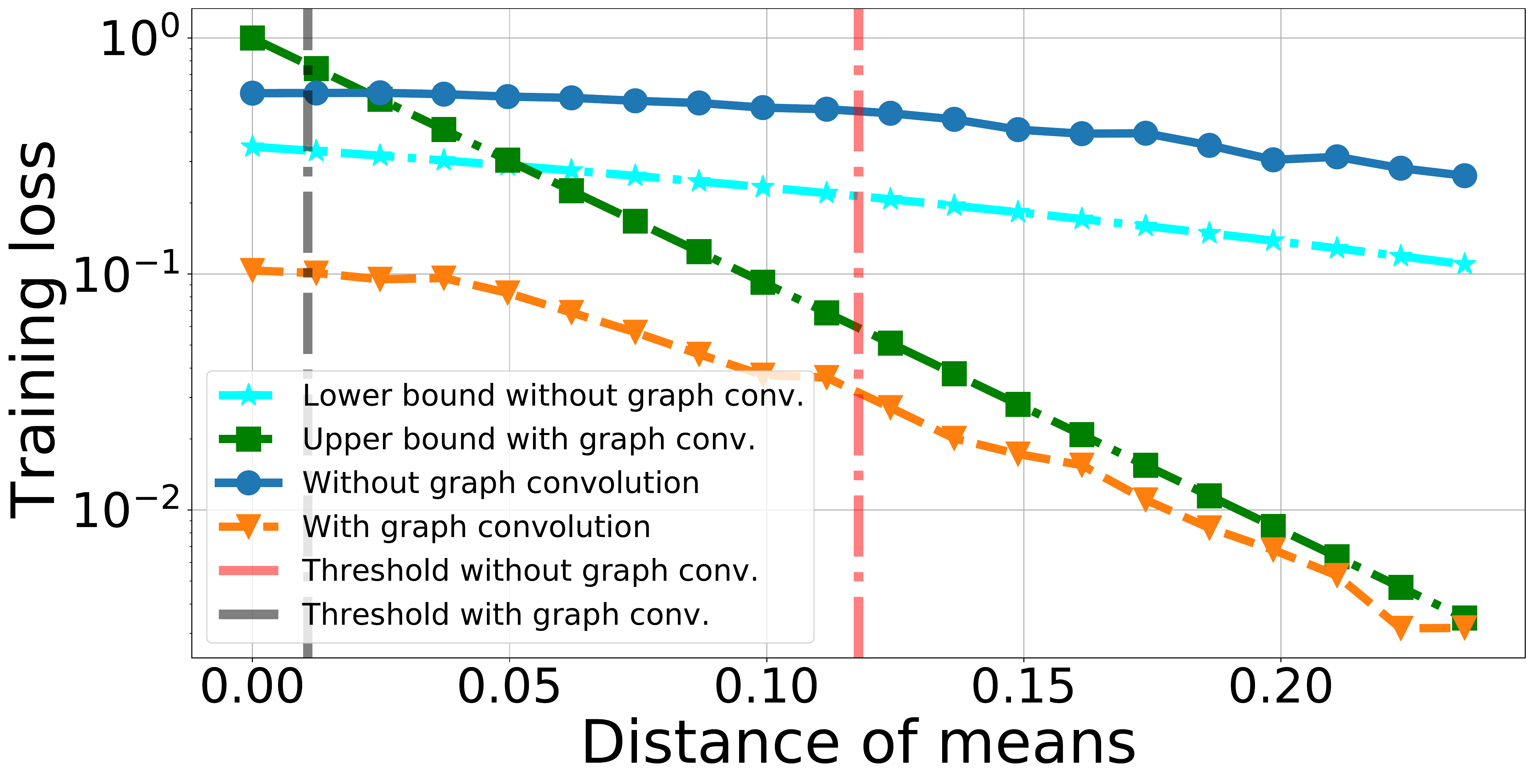}
		\caption{Training loss vs distance of means}\label{fig:1a}		
	\end{subfigure}
	\begin{subfigure}[t]{3in}
		\includegraphics[width=\columnwidth]{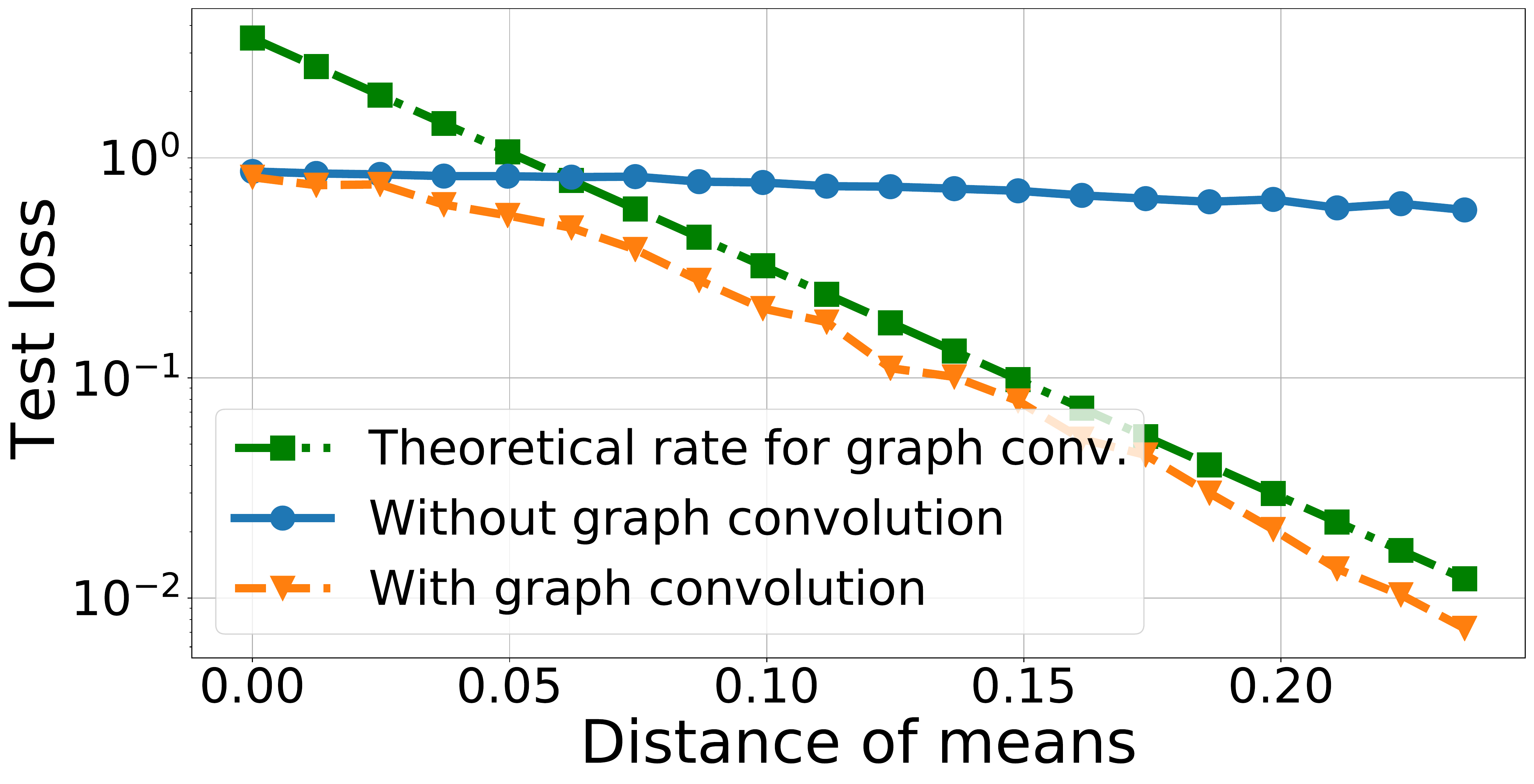}
		\caption{Test loss vs distance of means}\label{fig:1b}
	\end{subfigure}
	\caption{Training and test loss with/without graph convolution for increasing distance between the means. The vertical dashed red and black lines correspond to the separability thresholds from Parts 1 and 3 of \Cref{thm:thresholds}, respectively. The green dashed line with square markers illustrates the theoretical rate from \Cref{thm:generalization}. The cyan dashed line with star markers corresponds to the lower bound from Part 1 of \Cref{thm:thresholds}. We train and test on a CSBM with $p=0.5$, $q=0.1$, $n=400$ and $d=60$. The $y$-axis is in log-scale.}\label{fig:1}
\end{figure}

\subsection{Training and test loss against density of graph}
\label{subsec:train_loss_vs_density}
In our second experiment, we illustrate how the training and test losses scale as the density of the graph increases while maintaining the same signal to noise ratio for the graph. By density we mean the value of the intra- and inter-class edge probabilities $p$ and $q$, since they both control the average degree of each node in the graph. It is important to note that our theoretical results are based on \Cref{assumption:2}, which states lower bounds for $p$, $q$ and $\Gamma(p,q)$. For this experiment we train and test on a CSBM with $q=0.2 p$ where $p$ varies from $1/n$ to $0.5$ and $\Gamma(p,q)\approx 0.6$, $d=60$, $n=400$, and $N=10$, where $n$ is roughly equal to $0.85\cdot d^{3/2}$, and each class has $200$ nodes. For this experiment we set the distance between the means to $2/\sqrt{d}$.
The results for training loss are shown in \Cref{fig:2a} and the results of the test loss are shown in \Cref{fig:2b}. In these figures we observe that the performance of graph convolution improves as density increases. We also observe that for $p,q \le \log^2 n /n$, the performance of graph convolution is as poor as that of standard logistic regression.
\begin{figure}[ht!]
	\centering
	\begin{subfigure}[t]{3in}
		\centering
		\includegraphics[width=\columnwidth]{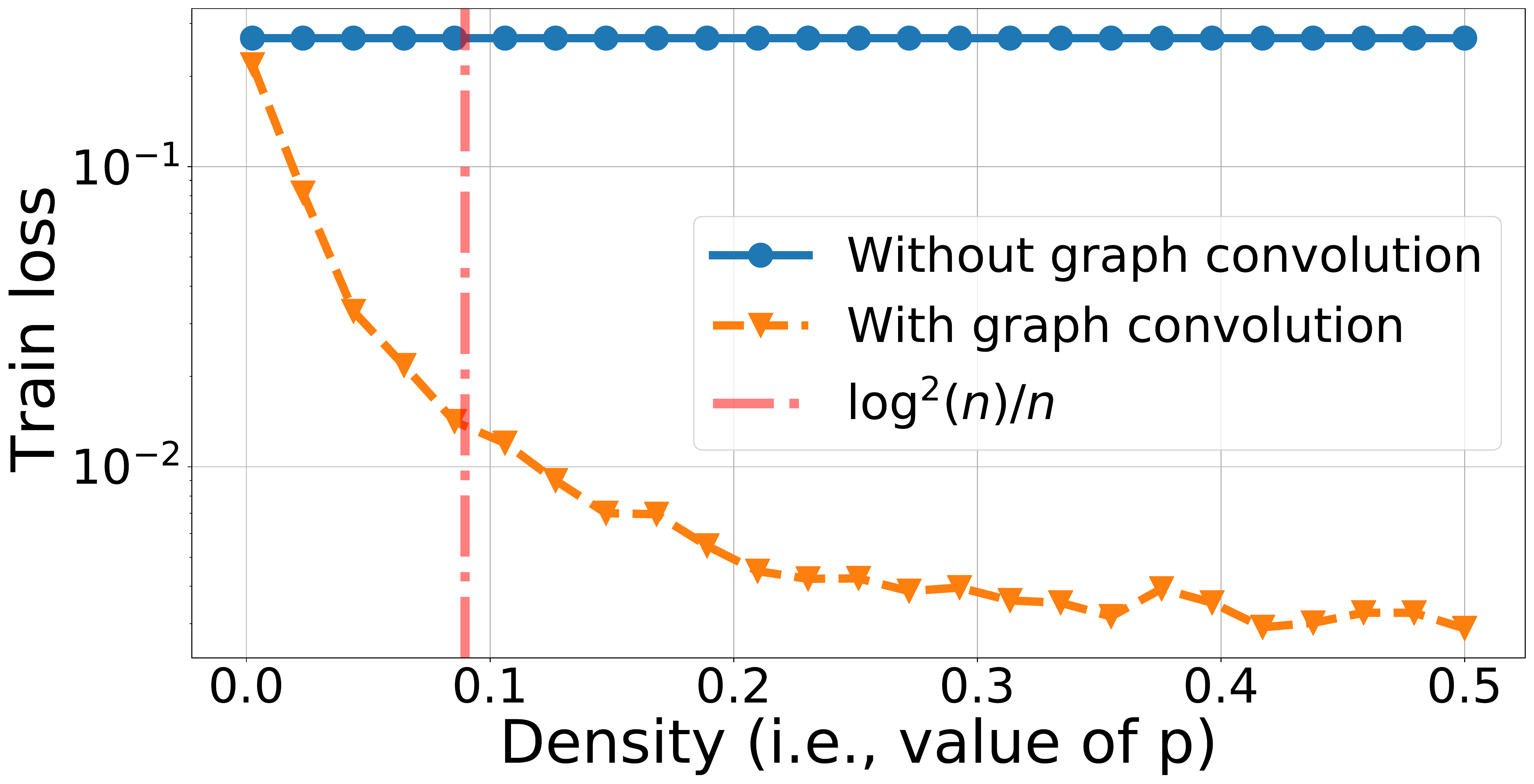}
		\caption{Training loss vs density}\label{fig:2a}		
	\end{subfigure}
	\begin{subfigure}[t]{3in}
		\includegraphics[width=\columnwidth]{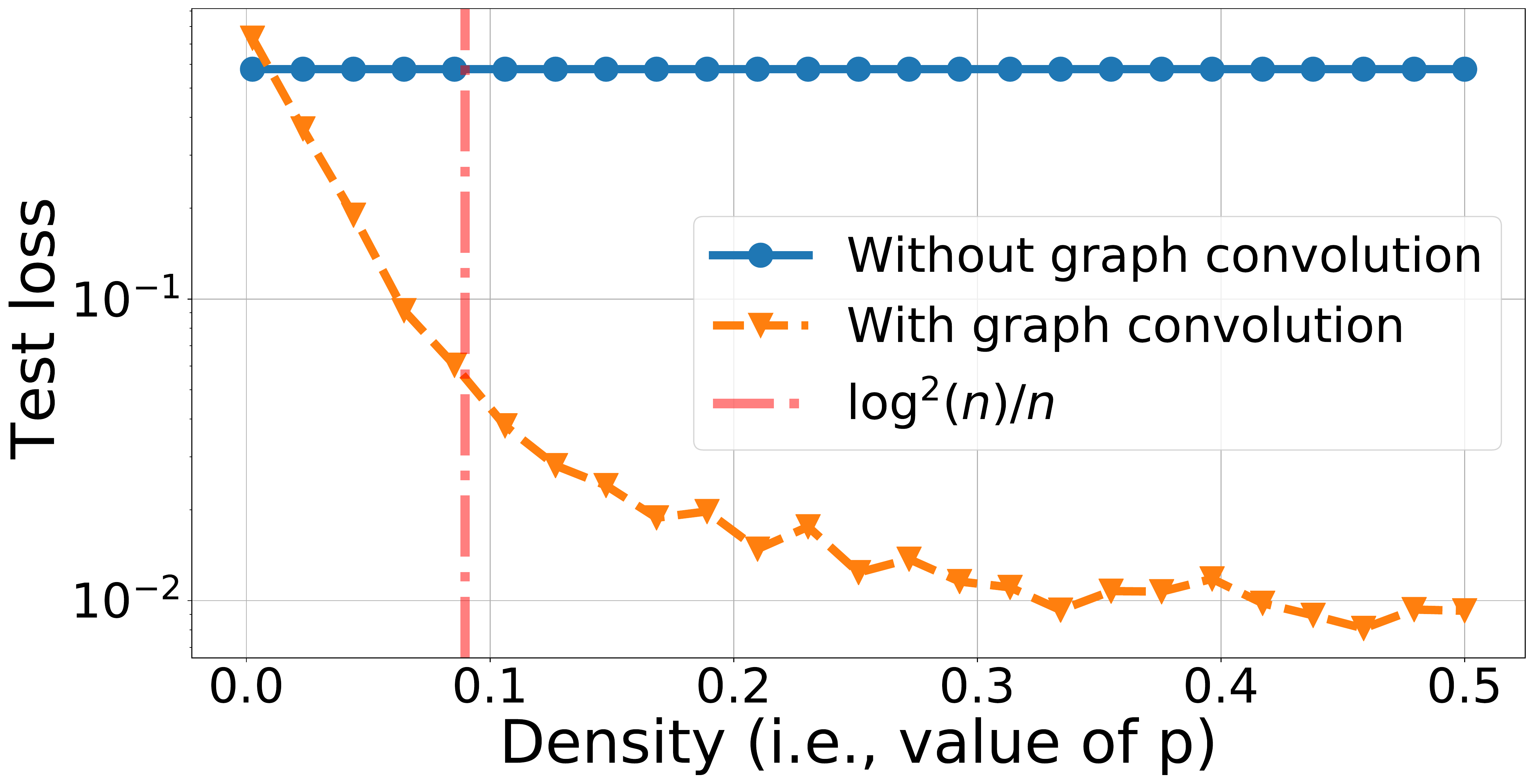}
		\caption{Test loss vs density}\label{fig:2b}
	\end{subfigure}
	\caption{Training and test loss with/without graph convolution for increasing density. The vertical dashed red line corresponds to the lower bound of $p$ and $q$ from \Cref{assumption:2}. See the main text for a detailed description of the experiment's parameters. The $y$-axis is in log-scale.}\label{fig:2}
\end{figure}

\subsection{Out-of-distribution generalization}\label{subsec:ood_synthetic}
In this experiment we test the performance of the trained classifier on out-of-distribution datasets.
We perform this experiment for two different distances between the means, $16/\sqrt{d}$ and $2/\sqrt{d}$. We train on a CSBM with $p_{\train}=0.5$, $q_{\train}=0.1$, $n=400$ and $d=60$, and we test on CSBMs with $n=400$, $d=60$ and varying $p_{\test}$ and $q_{\test}$ while $p_{\test}>q_{\test}$. The results are shown in \Cref{fig:3}\footnote{Note that the x-axis is $q$. Another option, that is more aligned with \Cref{thm:generalization} would be $\Gamma(p_{\test},q_{\test})$. However, the log-scale collapses all lines to one and the result is less visually informative.}. In this figure we observe what was studied in \Cref{thm:generalization} that is, out-of-distribution generalization to CSBMs with the same means but different $p$ and $q$ pairs. In particular, for small distance between the means, i.e., $2/\sqrt{d}$, where the data are close to being not linearly separable with high probability (Part 1 \Cref{thm:thresholds}), \Cref{fig:3a} shows that graph convolution results in much lower test error than not using the graph. This happens even when $q_{\test}$ is close to $p_{\test}$ in the figure, i.e., $\Gamma(p_{\test},q_{\test})$ from the bound in \Cref{thm:generalization} is small. Furthermore, in \Cref{fig:3b}, we observe that for large distance between the means, i.e., $16/\sqrt{d}$, where the data are linearly separable with high probability (Part 1 \Cref{thm:thresholds}), and $q_{\test}$ is much smaller than $p_{\test}$  (i.e., $\Gamma(p_{\test},q_{\test})$ is large), then graph convolution has low test error, and this error is lower than that obtained without using the graph. On the other hand, in this regime for the means, as $q_{\test}$ approaches $p_{\test}$ (i.e, as $\Gamma(p_{\test},q_{\test})$ decreases), the test error increases and eventually it becomes larger than without the graph. 

In summary, we observe that in the difficult regime where the data are close to linearly inseparable, i.e., the means are close but larger than $1/\sqrt{d}$, then graph convolution can be very beneficial. However,
if the data are linearly separable and their means are far apart, then we get good performance without the graph. Furthermore, if $\Gamma(p_{\test},q_{\test})$ is small then the graph convolution can actually result in worse training and test errors than logistic regression on the data alone. In \cref{additional-experiments}, we provide similar plots for various training pairs $p_{\test}$ and $q_{\test}$. We observe similar trends in those experiments.
\begin{figure}[ht!]
	\centering
	\begin{subfigure}[t]{3in}
		\centering
		\includegraphics[width=\columnwidth]{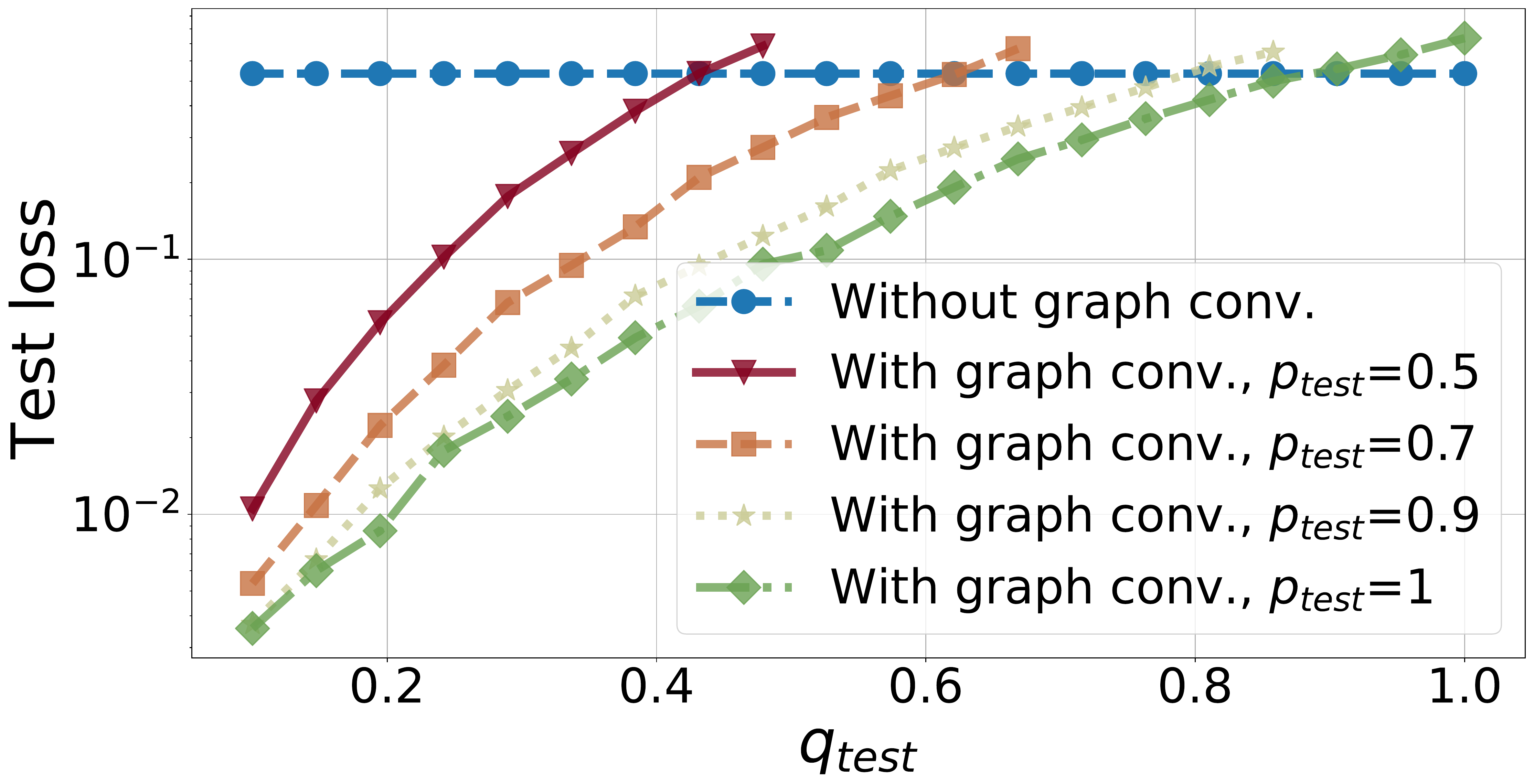}
		\caption{$\norm{\mu_d-\nu_d}=2/\sqrt{d}$}\label{fig:3a}	
	\end{subfigure}
	\begin{subfigure}[t]{3in}
		\includegraphics[width=\columnwidth]{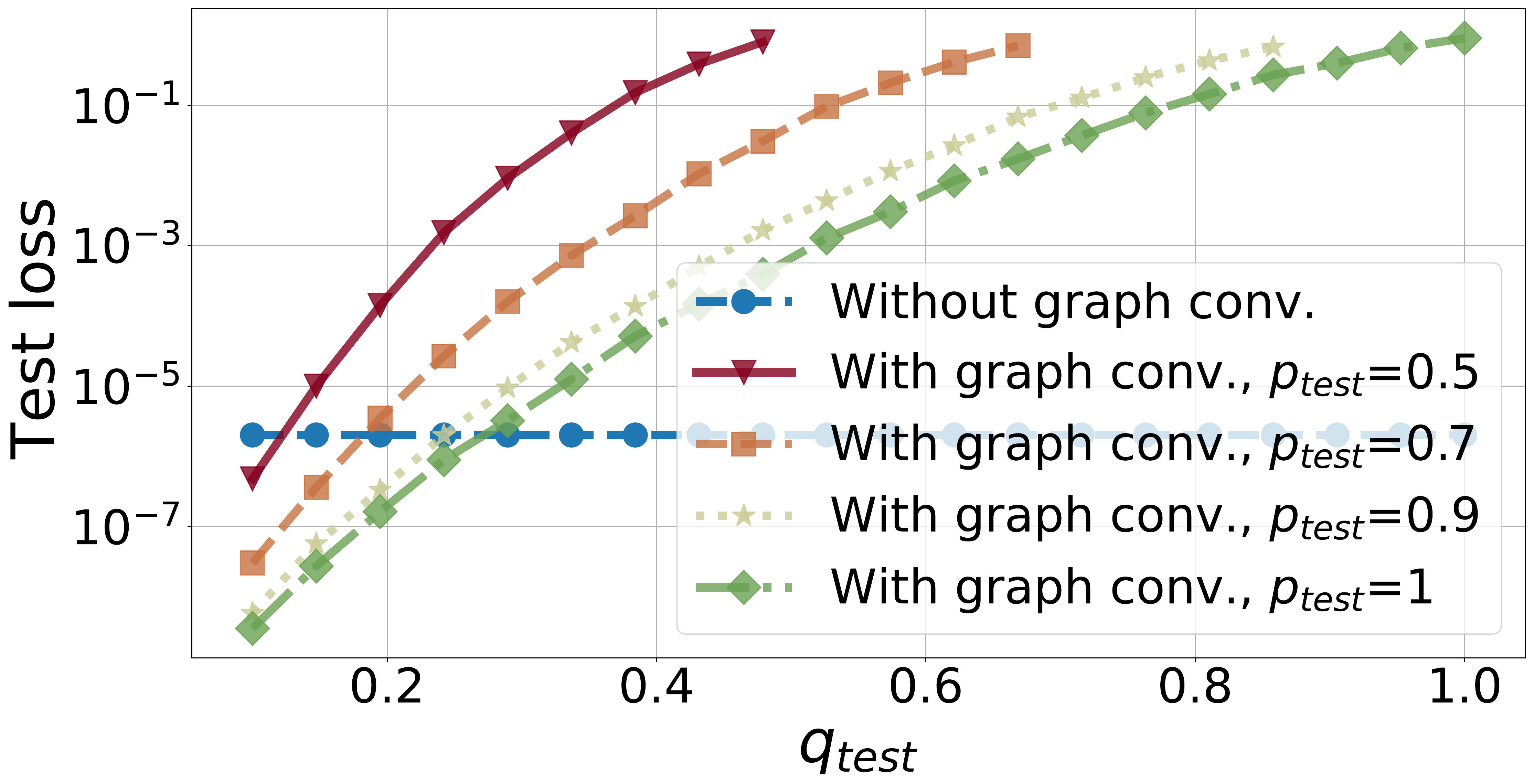}
		\caption{$\norm{\mu_d-\nu_d}=16/\sqrt{d}$}\label{fig:3b}
	\end{subfigure}
	\caption{Out-of-distribution generalization. We train on a CSBM with $p_{\train}=0.5$, $q_{\train}=0.1$, $n=400$ and $d=60$. We test on CSBMs with $n=400$, $d=60$ and varying $p_{\test}$ and $q_{\test}$ while $p_{\test}>q_{\test}$ and fixed means. The $y$-axis is in log-scale.}\label{fig:3}
\end{figure}

\subsection{Out-of-distribution generalization on real data}\label{subsec:ood_real}
In this experiment we illustrate the generalization performance on real data for the linear classifier obtained by solving~\eqref{eq:OPT}. In particular, we use the partially labelled real data to train two linear classifiers, with and without graph convolution. We generate new graphs by adding inter-class edges uniformly at random. Then we test the performance of the trained classifiers on the noisy graphs with the original attributes. Therefore, the only thing that changes in the new unseen data are the graphs, the attributes remain the same.
Note that our goal in this experiment is not to beat current baselines, but rather to demonstrate out-of-distribution generalization for real data when we use graph convolution. 

We use the popular real data Cora, PubMed and WikipediaNetwork. These data are publicly available and can be downloaded from~\cite{FL2019}. The datasets come with multiple classes, however, for each of our experiments we do a one-v.s.-all  classification for a single class. WikipediaNetwork comes with multiple masks for the labels, in our experiments we use the first mask. Moreover, this is a semi-supervised problem, meaning that only a fraction of the training nodes have labels. Details about the datasets are given in \Cref{table:1}. 
\begin{table}[ht!]
\caption{Information about the datasets, $\beta_0$ and $\beta_1$ are defined in \Cref{model}. Note that for each dataset we only consider classes $A$ and $B$ and we perform linear classification in a one-v.s.-all fashion. Here, $A$ and $B$ refer to the original classes of the dataset. Results for other classes are given in \cref{additional-experiments}.} \vspace{0.1cm}
\centering
 \begin{tabular}{||c c c c||} 
 \hline
 Info./Dataset & Cora & PubMed & Wiki.Net. \\ [0.5ex] 
 \hline\hline
 \# nodes                        & $2708$ & $19717$ & $2277$ \\ 
 \hline
 \# attributes                     & 1433  & $500$ & $2325$ \\
 \hline
 $\beta_0$, class $A$ & $5.0$e$-2$ & $2.5$e$-3$ & $4.7$e$-1$ \\
 \hline
 $\beta_1$, class $A$ & $5.6$e$-2$ & $4.8$e$-3$ & $4.9$e$-1$ \\
 \hline
 $\beta_0$, class $B$ & $4.8$e$-2$  & $3.3$e$-3$ & $4.7$e$-1$  \\
 \hline
 $\beta_1$, class $B$ & $9.2$e$-2$  & $2.5$e$-3$ & $4.7$e$-1$  \\
 \hline
 $\norm{\muv-\nuv}$, class $A$   & $7.0$e$-1$ & $1.0$e$-1$  & $3.6$e$-1$ \\
 \hline
 $\norm{\muv-\nuv}$, class $B$   & $9.4$e$-1$ & $7.2$e$-2$ & $3.0$e$-1$ \\ [1ex] 
 \hline
\end{tabular}
\label{table:1}
\end{table}

The results for this experiments are presented in \Cref{fig:4}.
We present results for classes $A$ and $B$ for each dataset. This set of experiments is enough to demonstrate good and bad performance when using graph convolution. The results for the rest of the classes are presented in \Cref{additional-experiments}. The performance for other classes is similar. Note in the plots that in this figure the y-axis (Test error) measures the number of misclassified nodes\footnote{The reason that we do not use the loss in the y-axis is because the test loss did not differ much between using and not using graph convolution. However, the misclassified nodes had large differences as shown in \Cref{fig:4}. As noted after \Cref{thm:generalization}, our argument for the bound on the loss immediately yields a bound on the number of misclassified nodes. } over the number of nodes in the graph. In all sub-figures in \Cref{fig:4} except for \Cref{fig:4c} we observe that graph convolution has lower test error than without the graph convolution. However, as we add inter-class edges (noise increases), then graph convolution can be disadvantageous. Also, there can be cases like in \Cref{fig:4c} where graph convolution is disadvantageous for any level of noise. Interestingly, in the experiment in \Cref{fig:4c} the test errors with and without graph convolution are low (roughly $\sim 0.080$). This seems to imply that the dataset is close to being linearly separable with respect to the given labels. However, the dataset seems to be nearly non-separable after the graph convolution, since adding noise to the graph results in larger test error.
\begin{figure}[ht!]
	\centering
	\begin{subfigure}[t]{3in}
		\centering
		\includegraphics[width=\columnwidth]{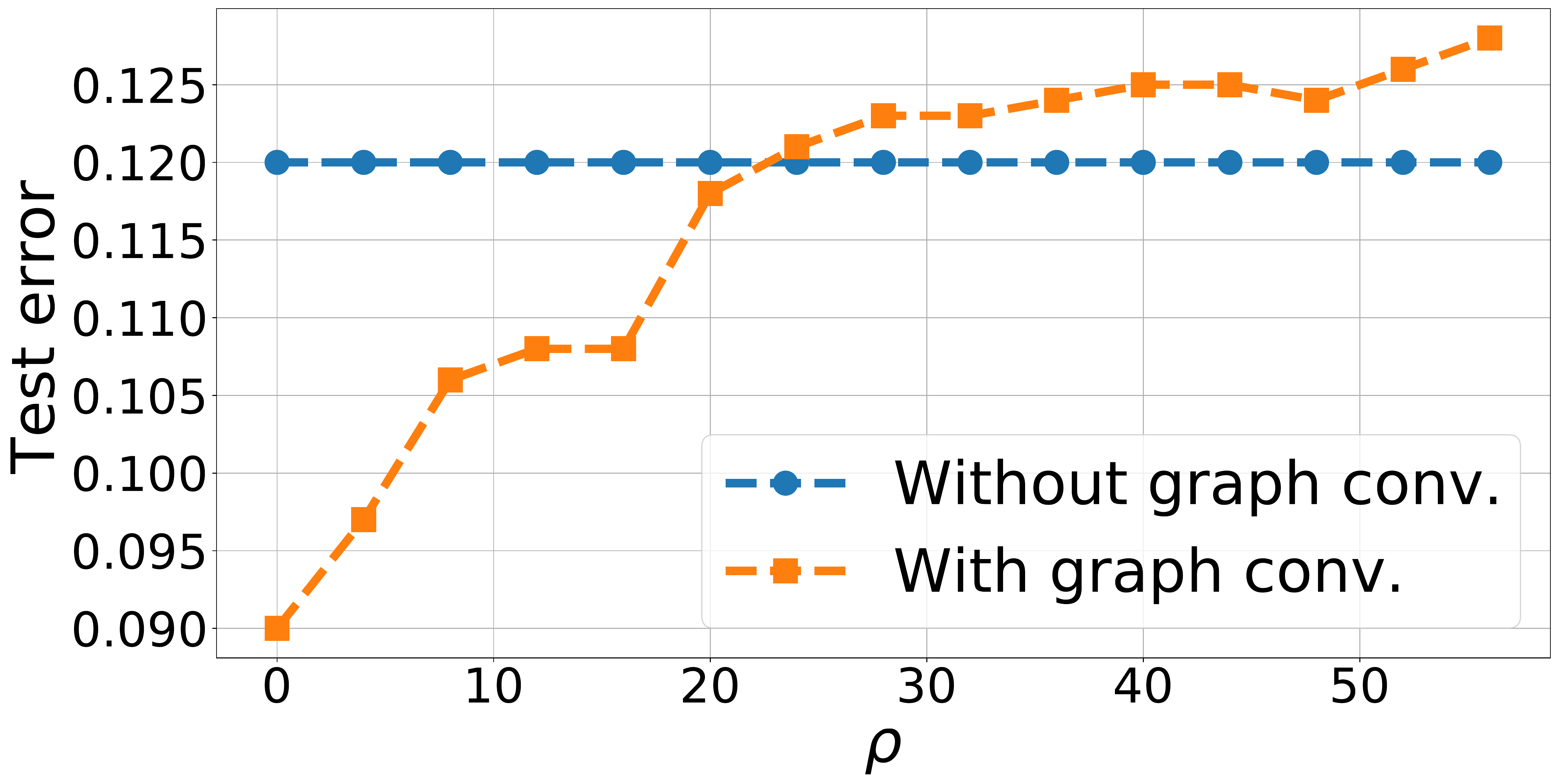}
		\caption{Cora, class $A$}\label{fig:4a}
	\end{subfigure}
	\begin{subfigure}[t]{3in}
		\centering
		\includegraphics[width=\columnwidth]{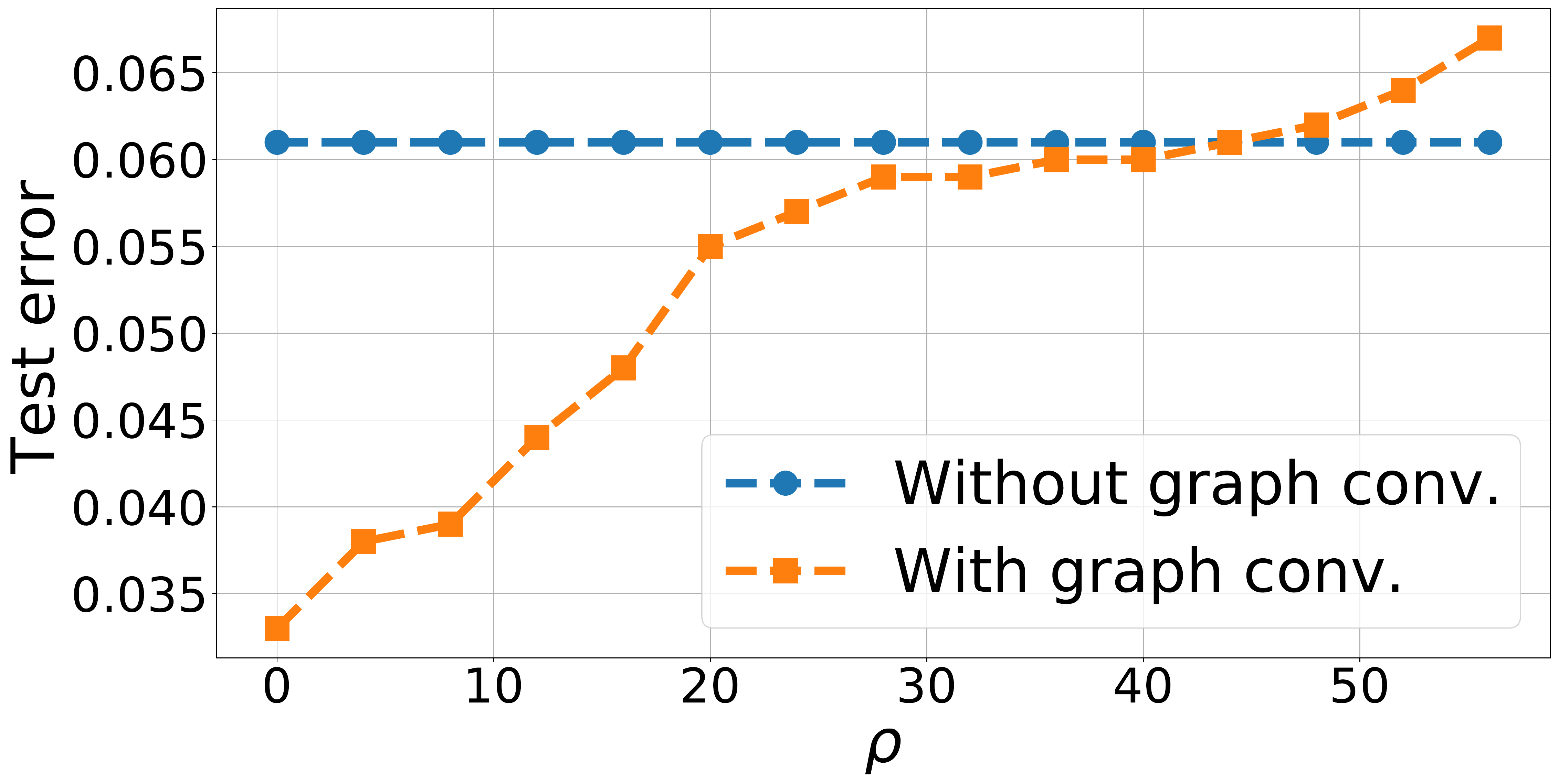}
		\caption{Cora, class $B$}\label{fig:4b}
	\end{subfigure}
	\\\vspace{0.2cm}
	\begin{subfigure}[t]{3in}
		\centering
		\includegraphics[width=\columnwidth]{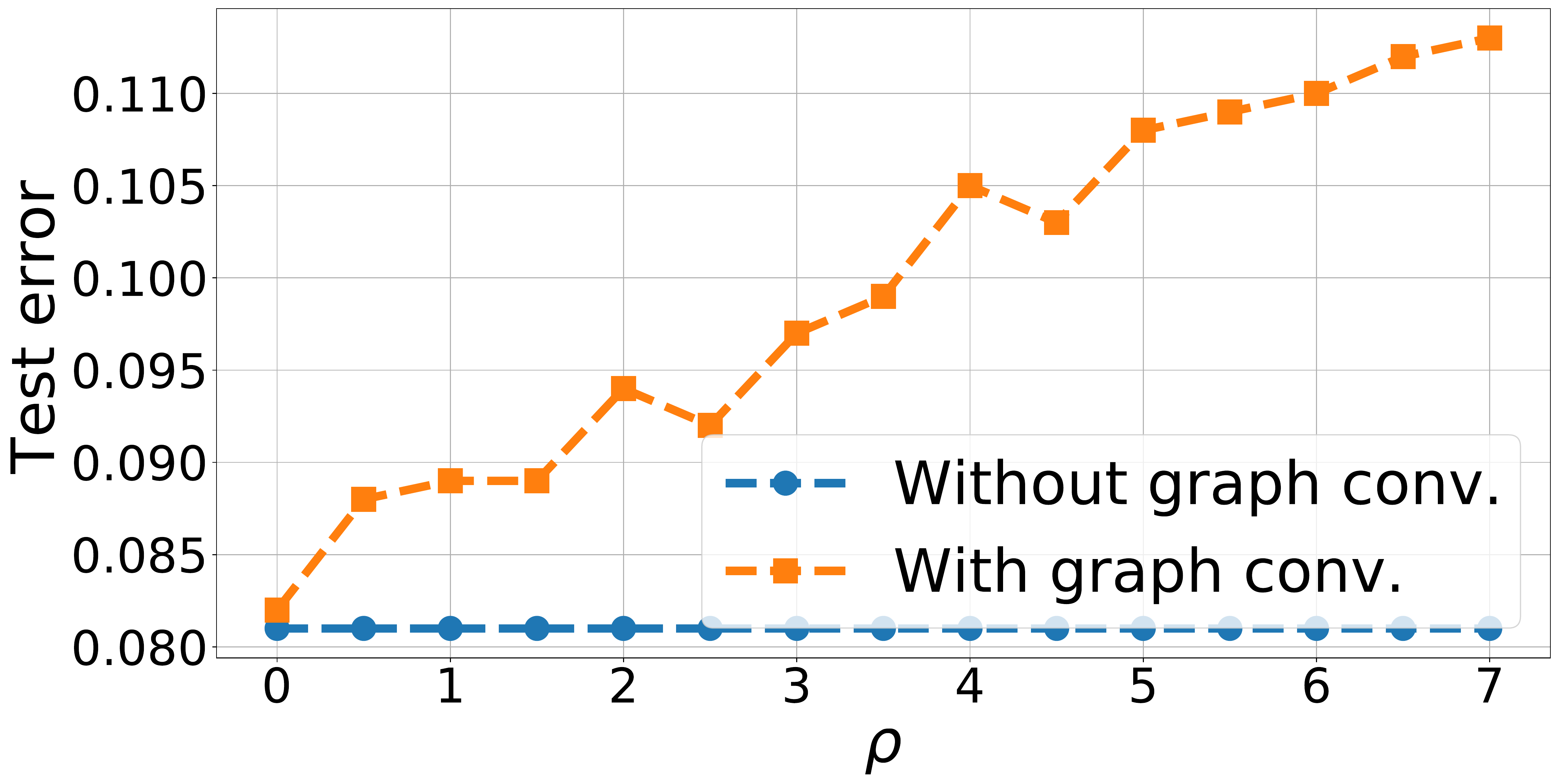}
		\caption{PubMed, class $A$}\label{fig:4c}
	\end{subfigure}
	\begin{subfigure}[t]{3in}
		\centering
		\includegraphics[width=\columnwidth]{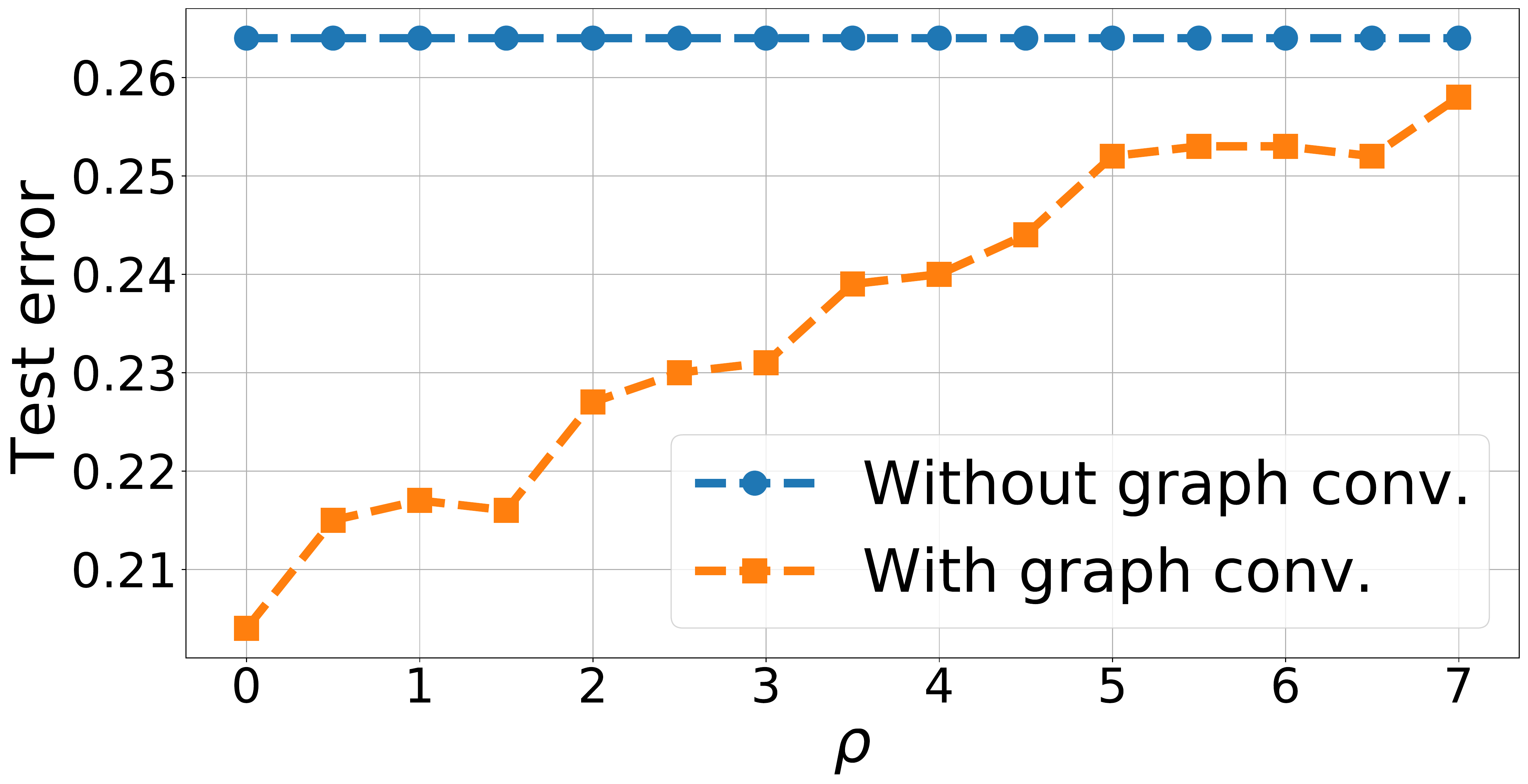}
		\caption{PubMed, class $B$}\label{fig:4d}
	\end{subfigure}
	\\\vspace{0.2cm}
	\begin{subfigure}[t]{3in}
		\centering
		\includegraphics[width=\columnwidth]{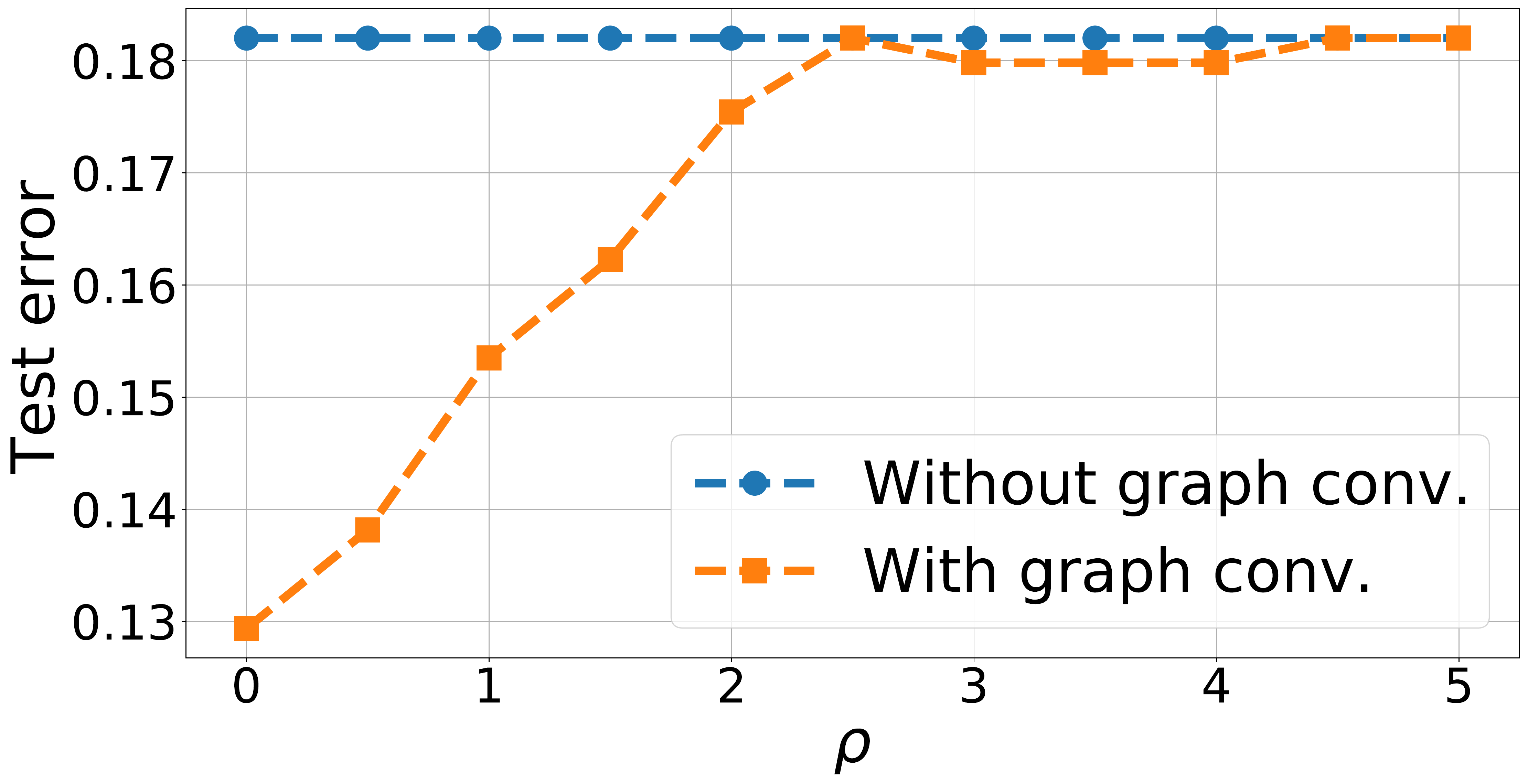}
		\caption{WikipediaNetwork, class $A$}\label{fig:4e}
	\end{subfigure}
	\begin{subfigure}[t]{3in}
		\centering
		\includegraphics[width=\columnwidth]{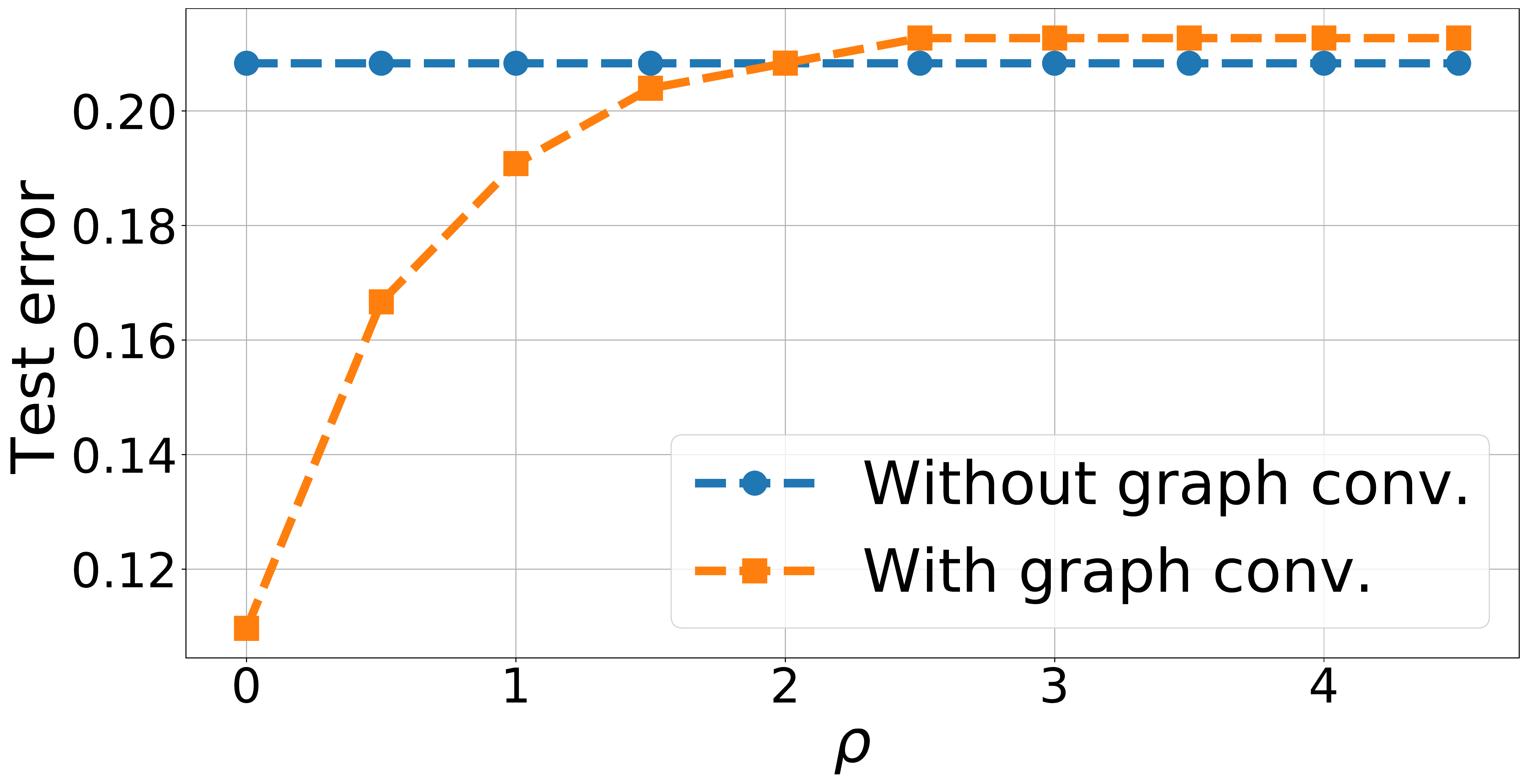}
		\caption{WikipediaNetwork, class $B$}\label{fig:4f}
	\end{subfigure}
	\caption{Test loss as the number of nodes increases. The test error measures the number of misclassified nodes over the number of nodes in the graph. Here, $\rho$ denotes the ratio of added inter-class edges over the number of inter-class edges of the original graph. The $y$-axis is in log-scale.}\label{fig:4}
\end{figure}

\section{Degree concentration}
\label{degree-concentration}
We note here the following elementary concentration results for the class size and degrees, which are all straightforward consequences of the Chernoff bound for sums of independent Bernoulli random variables, see, e.g., \cite[Theorem 2.3.1]{Vershynin:2018}.

Since $(\veps_i)_{i\in[n]}\sim\Ber(\frac12)$, by the Chernoff bound we have for any $\delta>0$ that
\[
\Prob\Big(\Big|\frac1n\sum_{i=1}^n\varepsilon_i - \frac12\Big| \ge \delta/2\Big) \le 2\exp(-n\delta^2/6).
\]
In particular, we have that for any $\delta>0$ the number of nodes in each class satisfies
\begin{equation}\label{eq:class-size-concentration}
    \Prob\left(\frac{\abs{C_0}}{n},\frac{\abs{C_1}}{n}\in\left[\frac{1}{2}-\delta,\frac{1}{2}+\delta\right]\right)\geq1-C\exp(-cn\delta^2),
\end{equation}
for some $C,c>0$.

The degrees are sums of Bernoulli random variables. Hence, by the Chernoff bound, for each $i$, we have for $\delta\in(0,1)$ that
\begin{align*}
    \Prob\left(\Big|D_{ii} - \E[D_{ii}]\Big| \ge \delta\E[D_{ii}]\right) \le 2\exp(-\E[D_{ii}]\delta^2/3),
\end{align*}
where for any $i$,
\[
\E[D_{ii}] = \frac12(\E[D_{ii}\mid\veps_i=0] + \E[D_{ii}\mid\veps_i=1]) = 1 + \frac{n-1}2(p+q).
\]
In particular, it follows that for any $\delta\in(0,1)$ we have
\begin{equation}\label{eq:degree-concentration}
\Prob\left(\frac{D_{ii}}n \in \left[\frac12(p+q)(1-\delta), \frac12(p+q)(1+\delta)\right]^\setc\right)\leq C\exp(-cn(p+q)\delta^2),
\end{equation}
for some $C,c>0$. As we will frequently work on the event that the degrees and class sizes concentrate, for fixed $\delta,\delta'>0$ we define the event 
\begin{equation}\label{eq:B-def}
B(\delta,\delta') = 
\left\{\frac{n}{2}(1-\delta)\leq |C_0|,|C_1|\leq \frac{n}{2}(1+\delta)\right\}
\bigcap_{i\in[n]}
\left\{\frac n2(p+q)(1-\delta')\leq D_{ii}\leq \frac n2(p+q)(1+\delta')\right\}.
\end{equation}

Since $p,q=\omega(\frac{\log^2 n}{n})$, by the union bound, if we choose $\delta=n^{-1/2+\eps}$ and $\delta'=(\log n)^{-1/2+\eps}$ then for $\eps>0$ small enough, for any $c>1$ there is  $C>0$ such that
\begin{equation}\label{eq:B-probability}
\Prob(B(n^{-1/2+\eps},(\log n)^{-1/2+\eps})) \ge 1 - \frac{C}{n^c}.
\end{equation}
Let $N_i$ denote those vertices connected to $i$ (including $i$), and let 
\[
\begin{aligned}
\tilde{B}(\delta,\delta')&=B(\delta,\delta')\bigcap_{i\in[n]}
\left\{\frac{(1-\veps_i)p+\veps_i q}{p+q}(1-\delta')\leq \frac{\abs{C_0\cap N_i}}{D_{ii}}\leq \frac{(1-\veps_i)p+\veps_i q}{p+q}(1+\delta')\right\}\\
&\bigcap_{i\in[n]}
\left\{\frac{\veps_i p+(1-\veps_i) q}{p+q}(1-\delta')\leq \frac{\abs{C_1\cap N_i}}{D_{ii}}\leq \frac{\veps_i p+(1-\veps_i) q}{p+q}(1+\delta')\right\}
\end{aligned}
\]
by similar reasoning, a Chernoff bound and union bound yields, for  $\eps>0$ small enough, we have that for any $c>0$, and some $C>0$,
\begin{equation}\label{eq:B-prob-with-ngbd}
 \Prob\left(\tilde{B}(n^{-1/2+\eps},(\log n)^{-1/2+\eps})\right)
\geq 1-\frac{C}{n^c}.
\end{equation}

\section{Separability thresholds}
\label{separability-thresholds}
In this section, we prove \Cref{thm:thresholds}. We begin by first proving a bound on a certain Gaussian process. We then develop concentration bounds for the convolved data. We end the section by proving the three parts of the theorem in turn.
\subsection{Bounds for the isonormal process}
Consider the Gaussian process,  $g(\vv)=\inner{Z,\vv}$ for all $\vv\in\R^d$, for some standard Gaussian vector $Z\sim N(\zero,I)$. The process $g$ is sometimes called the \emph{isonormal process} or \emph{canonical Gaussian process}. Controlling its behaviour will be an essential step in showing that the mixture model or CSBM data is not linearly separable below a certain threshold.  Let $g_i(\vv)$ denote \iid\ copies of this process and Define the events
\[
\begin{aligned}
  A_{k,\gamma,n}(\vv) &=\{\exists J\subseteq [n]:
  g_i(\vv) > \gamma \text{ for } i\in J,\quad
  |J| = k\},\\
  \tilde{A}_{k,\gamma,n}(\vv) &= \{\exists J\subseteq [n]:
  g_i(\vv) > \gamma \text{ for } i\in J,
  g_i(\vv)<\gamma  \text{ for } i\in J^c \quad
  |J| = k\}
\end{aligned}
\]
Observe that for each $k$, $A_{k,\gamma,n}(\vv)$ is the event that at $\vv$ the $k$-th largest of the $g_i(\vv)$ is large and the tilded version is the event that this occurs and the remaining $g_i(\vv)$ are all small. 
Let $$H(x)=-x \log x -(1-x)\log(1-x).$$
Finally for $\eps>0$, let $\Sigma_{\eps,d}$ denote an $\eps$-net of the unit sphere, $\bS^{d-1}$. We begin by showing the following result about the isonormal process.

\begin{lemma}\label{lem:isonormal-process}
Suppose that $n$ satisfies \Cref{assumption:1}. Then for any $\gamma>0$ and $0 < t < 1-\Phi(\gamma)$, there is a $C>0$ such that for $d\geq 1$,
\begin{equation}
     \frac{1}{d}\log \Prob\Big(\bigcup_{\vv\in \bS^{d-1}}A_{\floor{tn},\gamma,n}(\vv)^c\Big) \leq -C, \label{eq:union-bound-separable-1}
\end{equation}
and for any $0<t\leq 1$, there is a $C>0$ such that for $d\geq 1$
\begin{equation}
     \frac{1}{d}\log \Prob\Big(\bigcup_{\vv\in \bS^{d-1}}A_{1,\gamma,\floor{tn}}(\vv)^c\Big) \leq -C. \label{eq:union-bound-separable-2}
\end{equation}
\end{lemma}
\begin{proof}
It will suffice to consider only the first case as the second case clearly follows by the same argument.
For $L$ sufficiently large and fixed $\eps<\gamma$, let $\eps'=\frac{\eps}{L\sqrt{d}}$. Then
we have that
\[
\begin{aligned}
\Prob\Big(\bigcup_{\vv\in \bS^{d-1}} A_{\floor{tn},\gamma,n}(\vv)^c\Big) &\leq
n \Prob(\norm{Z_1} > L\sqrt{d})
+\Prob\Big(\bigcup_{\vv\in \bS^{d-1}} A_{\floor{tn},\gamma,n}(\vv)^c\cap\{\norm{Z_i}\leq L\sqrt d ~~ \forall i\in [n]\}\Big)\\
&\leq n \Prob(\norm{Z_1} > L\sqrt{d}) + \Prob\Big(\bigcup_{\vv\in\Sigma_{\eps',d}} A_{\floor{tn},\gamma+\eps,n}(\vv)^c
\cap\{\norm{Z_i}\leq L\sqrt d ~~ \forall i\in [n]\}\Big) \\
&\leq n \Prob(\norm{Z_1} > L\sqrt{d}) + \Prob\Big(\bigcup_{\vv\in\Sigma_{\eps',d}} A_{\floor{tn},\gamma+\eps,n}(\vv)^c\Big)= A+B,
\end{aligned}
\]
The first inequality above follows from the law of total probability and then a union bound over all $i\in[n]$. For the second inequality, observe that since $\Sigma_{\eps',d}$ is an $\eps'$-net, we have that for a fixed $\vv\in \bS^{d-1}$ if $\uv\in\Sigma_{\eps',d}$ is the vector in the $\eps'$-net nearest to $\vv$ then if we let $E=\{\norm{Z_i}\leq L\sqrt{d}\,i\in [n]\}$ then $A_{\floor{tn},\gamma,n}(\vv)^c\cap E \subseteq A_{\floor{tn},\gamma+\eps,n}(\uv)^c\cap E$.

We bound these terms in turn. Let us begin with $A$. Recall that by the norm concentration of a standard Gaussian vector \cite[Theorem 3.1.1]{Vershynin:2018}, there exist $C,c'>0$ such that for any $L>1$ and $d\geq 1$,
\[
\Prob(\norm{Z_1}>L \sqrt{d}) \le  C\exp(-c'dL^2).
\]
Recall from \cref{assumption:1} that $n=O(\poly(d))$. Then we have that for some constant $c>0$,
\[
A = n\Prob(\norm{Z_1}>L \sqrt{d}) \le \exp(-cdL^2).
\]
On the other hand, for $B$, we have that for some $C'>0$
and any fixed $\vv\in\bS^{d-1}$
\begin{align*}
B &\le \abs{\Sigma_{\eps',d}}\Prob(A_{\floor{nt},\gamma+\eps,d,n}(\vv)^c)
\leq \exp( C' d\log (d/\eps)) \Prob(A_{\floor{nt},\gamma+\eps,n}(\vv)^c)\\
&\leq \exp( C' d\log (d/\eps))
\sum_{s<n t} \Prob(\tilde{A}_{s,\gamma+\eps,n}(\vv))\\
&\leq \exp( C' d\log (d/\eps))\sum_{s< n t} \binom{n}{s}\Prob(g_1(\vv)>\gamma+\eps)^s \Prob(g_1(\vv)<\gamma+\eps)^{n-s} \\
&\leq \exp( C' d\log (d/\eps))\sum_{s< n t} \exp\Big(nH\Big(\frac sn\Big)\Big)(1-\Phi(\gamma+\eps))^s \Phi(\gamma+\eps)^{n-s} \\
&= \exp(C'd\log(d/\eps))\sum_{s< nt} \exp\left(n\left[H\Big(\frac sn\Big) + \frac sn\log(1-\Phi(\gamma+\eps)) + (1-\frac sn)\log\Phi(\gamma+\eps)\right]\right) \\
&\leq n t \exp\{n [H(t)+t\log(1-\Phi(\gamma+\eps))+(1-t)\log\Phi(\gamma+\eps)]+O(d\log(d/\eps))\}.
\end{align*}
The first inequality follows by a union bound. The second inequality follows from $\abs{\Sigma_{\eps,d}}\leq (2/\eps + 1)^d$ for any $\eps\in(0,\frac12)$ \cite[Corollary 4.2.13]{Vershynin:2018}, the third follows by union bound since $A_{\floor{nt},\gamma+\epsilon,n}^c\subseteq \cup_{s<nt} \tilde{A}_{s,\gamma+\eps,n}$, the fourth follows since $g_i(\vv)$ are \iid, and the fifth by the Stirling bound 
$ \binom{n}{nt}\leq\exp(n H(t))$. For the final inequality, note that since $\gamma+\eps>0$, the function 
\[
f(x)=H(x)+x\log(1-\Phi(\gamma+\eps))+(1-x)\log\Phi(\gamma+\eps)
\]
is negative and increasing for $0<x<1-\Phi(\gamma+\eps)$ so that each summand is bounded above by the value  at $s=nt$ since $t<1-\Phi(\gamma+\eps)$.  Since by \Cref{assumption:1}, $n=\omega(d\log d)$, we have that there is some $C>0$ such that
\[
B\leq C\exp(-cd\log d)
\]
Consequently, $0\leq B/A\leq C$ for all $d$ for some $C>0$
Combining the bounds on $A$ and $B$ we obtain
\[
 \frac{1}{d}\log(A+B)\leq \frac{1}{d}\log A+\frac{1}{d}\log(1+C) =-cL^2+O\left(\frac{1}{d}\right),
\]
from which the result follows.
\end{proof} 

\subsection{Proof of part 1 of  Theorem 1}
We are now ready to prove part 1 of \Cref{thm:thresholds}, which shows the threshold for data to be linearly separable, along with a corresponding lower bound for the loss.
\begin{proof}[Proof of \Cref{thm:thresholds} part 1] Observe that $X_i$ can be written as
\[
X_i=(1-\veps_i)\muv + \veps_i \nuv + \frac{Z_i}{\sqrt{d}}
\]
where $Z_i$ are \iid\ standard Gaussian vectors. 

By \eqref{eq:class-size-concentration}, it suffices to bound these terms on the event from \eqref{eq:class-size-concentration}.
If $(X_i)$ are linearly separable, then there is a unit vector $\vv$ and $b\in\R$ such that
\begin{equation}\label{eq:non-sep-mixture-1}
\inner{\muv,\vv} + \frac{\inner{Z_i,\vv}}{\sqrt{d}} + b < 0, \quad i\in S_0\quad \text{ and }\quad
\inner{\nuv,\vv} + \frac{\inner{Z_i,\vv}}{\sqrt{d}} + b > 0, \quad i\in S_1.
\end{equation}
Recall that $\norm{\muv-\nuv}\le K/\sqrt{d}$. Hence, writing $b = b' - \frac{\inner{\muv+\nuv,\vv}}2$, we see that
if the above holds then there is a pair $(\vv,b')$ such that
\begin{equation}\label{eq:non-sep-mixture-2}
\max_{i\in S_0}\frac{\inner{Z_i,\vv}}{\sqrt{d}} + b' < \frac{K}{2\sqrt{d}}
\quad\text{ and } \quad
\min_{i\in S_1}\frac{\inner{Z_i,\vv}}{\sqrt{d}} + b' > - \frac{K}{2\sqrt{d}}.    
\end{equation}
Such a pair $(\vv,b')$ exists only if at least one of the above two holds with $b'=0$.
Conditionally on the event $|S_0|=k$, the probability of this occurring is at most sum of the probability of these two events:
\[
\begin{aligned}
& \Prob\Big(\exists \vv\in\bS^{d-1}: \max_{i\leq k} g_i(\vv)<K/2\Big) +\Prob\Big(\exists \vv\in\bS^{d-1}:\min_{ i\in [|S|-k,|S|]} g_i(\vv)>-K/2\Big)\\
& \leq  2\Prob\Big(\exists \vv\in\bS^{d-1}: \max_{i\leq k\wedge |S|-k} g_i(\vv)<K/2\Big) = I_k
\end{aligned}
\]
As this function is decreasing in $k$, it suffices to bound it in the case that $k=(\frac12-\delta)\beta_0 n$, by \eqref{eq:class-size-concentration}. Note that
\[
I_{tn} = 2\Prob\Big(\bigcup_{\vv\in \bS^{d-1}}A_{1,\gamma,\floor{tn}}(\vv)^c\Big)
\]
with $\gamma=K/2$ and $t=(\frac12 - \delta)\beta_0$.
Thus, using \eqref{eq:union-bound-separable-2} we have
that 
\[
\frac{1}{d}\log(I_{tn})\leq -C.
\]
The first result then follows by combining this with \eqref{eq:class-size-concentration}.

Let us now turn to the lower bound on the loss. Take $t<1-\Phi\Big(\frac{K}2(1+\eps)\Big)$ for some $\eps>0$. Since $K>0$ and $\beta_0,\beta_1\leq 1/2$, using \eqref{eq:union-bound-separable-1} with $\gamma=\frac K2(1+\eps)$ we have that with probability at least $1-Ce^{-cd}$, for all $\vv$ with $\norm{\vv} = 1$ there are $t\beta_0 n$ choices of $i\in S_0$ and $t\beta_1 n$ choices of $i\in S_1$ with
\begin{equation}\label{eq:gv-bounds}
\inner{Z_i,\vv} > (1+\eps)\frac{K}{2}\qquad\text{ and }\qquad \inner{Z_i,\vv} < -(1+\eps)\frac{K}{2}
\end{equation}
respectively. Let these sets of indices be denoted by $J(\vv),J'(\vv)$, and let $l(X_i,\veps_i,\vv,b)$ denote the loss, given by
\[
l(X_i,\veps_i,\vv,b) = -\veps_i\log(\sigma(\inner{X_i,\vv}+b)) - (1 - \veps_i)\log(1-\sigma(\inner{X_i,\vv}+b)).
\]
Then using \eqref{eq:gv-bounds} we have  that for each $\vv\in \R^d$
\begin{align*}
\min_{i\in J(\vv/\norm{\vv})}l(X_i,\veps_i,\vv,b)
&= -\log(1-\sigma(\inner{X_i,\vv}+b))\\
&= -\log\left(1-\sigma\left(\inner{\muv,\vv}+b + \frac{\inner{Z_i,\vv}}{\sqrt d}\right)\right)\\
&\geq -\log\left(1-\sigma\left(\eps\frac{ K \norm{\vv}}{2\sqrt{d}} + b'\right)\right).
\end{align*}
Similarly,
\begin{align*}
\min_{i\in J'(\vv/\norm{\vv})}l(X_i,\varepsilon_i,\vv,b) 
&\geq -\log \sigma\left(-\eps\frac{K\norm{\vv}}{2\sqrt{d}} + b'\right).
\end{align*}
Thus, using \eqref{eq:BCE-loss} we have that
\[
L(I,X,\vv,b) \geq t f\left(\eps\frac{K\norm{\vv}}{2\sqrt{d}}, b'\right),
\]
where $f(x,y)=-\beta_0\log(1-\sigma(x+y))-\beta_1\log \sigma(-x+y)
=\beta_0\log(1+e^{x+y})+\beta_1\log(1+e^{x-y})$.\\
Note that by optimizing in $x,y$, we see that for $\beta_0=\beta_1$ and $x\geq0$, we have 
\[f(x,y)\geq f(0,0)= 2\beta_0\log (2),\]
so that for any $0<\beta_0,\beta_1\leq \frac12$ and $x\geq0$
we have $f(x,y)\geq (\beta_0\wedge\beta_1) 2\log 2$, and thus,
\[
L(I,X,\vv,b)\geq 2t\cdot \beta_0\wedge\beta_1 \cdot \log 2.
\]

Combining the above and minimizing in $\vv,b$, we see that for every $0<t<1-\Phi\Big(\frac{K}2(1+\eps)\Big),$ there is some $c>0$ such that 
\[\min_{\vv\in\R^d,b\in\R} ~ L(I,X,\vv,b) \geq 2t \beta_0\wedge\beta_1 \cdot \log2\]
with probability at least $1-\exp(-cd)$ as desired.
\end{proof}

\subsection{Decomposition of the convolved data}
In this subsection we provide a decomposition of $\Xt$ which we will use frequently throughout the rest of this paper. Note that conditionally on $(\veps_i)$, we have that $X_j\sim\Nc(\muv_j,\frac1dI)$ where
$\muv_j=\muv$ if $j\in C_0$ and $\muv_j=\nuv$ if $j\in C_1$. Thus, we can write
\begin{equation}\label{eq:xj-decomp}
X_j = (1-\veps_j)\muv+\veps_j\nuv + \frac{g_j}{\sqrt{d}},
\end{equation}
where $g_j\sim\Nc(\zero,I)$ are \iid\ copies of a standard normal vector.

\begin{lemma}\label{lem:xt-i-concentration}
Conditionally on $A$ and $(\veps_k)$, we have that for any $c>0$ and some $C>0$, with probability at least $1 - Cn^{-c}$, for every $i\in[n]$ and any unit vector $\w$,
\begin{align*}
    \bignorm{\inner{\Xt_i - \frac{p\muv + q\nuv}{p+q},\w}(1 + o(1))}_2 &=
    O\left(\sqrt{\frac{\log n}{dn(p+q)}}\right)\; \text{for}\; \veps_i = 0,\\
    \bignorm{\inner{\Xt_i - \frac{q\muv + q\nuv}{p+q},\w}(1 + o(1))}_2 &=
    O\left(\sqrt{\frac{\log n}{dn(p+q)}}\right)\; \text{for}\; \veps_i = 1.
\end{align*}
\end{lemma}
\begin{proof}
Consider the random variables $\Xt_i=[D^{-1}\At X]_i$.
For any fixed $i$, we define $m(i)$ to be the conditional mean of $\Xt_i$ on the adjacency matrix $A$ and class memberships $(\veps_j)$,
\[
m(i) = \E[\Xt_i\mid A,\veps] = \frac1{D_{ii}}\sum_{j\in[n]}\at_{ij}\muv_j.
\]
From \eqref{eq:xj-decomp} we can write for any unit vector $\w$ that
\begin{equation}\label{eq:x-tilde}
\Xt_i\cdot\w = \frac1{D_{ii}}\sum_{j\in[n]}\at_{ij}(X_j\cdot\w)
= m(i)\cdot\w + \frac1{D_{ii}\sqrt{d}}\sum_{j\in[n]}\at_{ij}(\inner{g_j,\w}),
\end{equation}
When $\veps_i=0$, we have that 

\begin{align}
    m(i) = \frac1{D_{ii}}\left(\sum_{j\in C_0}\at_{ij}\muv + \sum_{j\in C_1}\at_{ij}\nuv \right)
    = \frac1{D_{ii}}(|C_0\cap N_i|\muv + |C_1\cap N_i|\nuv),\label{eq:mi}
\end{align}
and similarly when $\veps_i=1$.

Note that by \eqref{eq:B-prob-with-ngbd}, we have that with probability $1-1/n^c$ for $c>0$ large enough,
\[
\begin{aligned}
    \frac{\abs{C_0\cap N_i}}{D_{ii}} &= \left[(1-\veps_i)\frac{p}{p+q}+\veps_i \frac{q}{p+q}\right](1+o(1)),\\
    \frac{\abs{C_1\cap N_i}}{D_{ii}} &= \left[\veps_i\frac{p}{p+q}+(1-\veps_i) \frac{q}{p+q}\right](1+o(1))\\
    \frac{1}{D_{ii}} &=\frac{2}{n(p+q)}(1+o(1))
\end{aligned}
\]
for all $i\in [n]$, and so we have that
\begin{align}
    m(i) &= \frac{p\muv + q\nuv}{p+q}(1 + o(1))
    \qquad\text{ for } \veps_i=0,\label{eq:mi-0}\\
    m(i) &= \frac{q\muv + p\nuv}{p+q}(1 + o(1))
    \qquad\text { for } \veps_i=1.\label{eq:mi-1}
\end{align}

Next, we consider $F_i=\frac1{D_{ii}\sqrt{d}}\sum_{j\in[n]}\at_{ij}(\inner{g_j,\w})$. Note that for a given adjacency matrix $A$, we have that $F_i\sim\Nc(0,\frac{1}{dD_{ii}})$. Then by Gaussian concentration we have
\begin{equation}\label{eq:Gaussian-norm}
    \Prob(|F_i|>\delta\mid A)\le  2\exp(-\delta^2dD_{ii}/2).
\end{equation}
Define the event $Q=Q(t)=|F_i| \le t\; \forall i\in[n]\}$ and note that if we let the event $\tilde{B}$ to be $\tilde{B}(n^{-1/2+\eps},(\log n)^{\eps-1/2})$ then
\begin{align*}
    \Prob(Q^\setc) 
    \le \Prob(\tilde{B}\cap Q^\setc) + \Prob(\tilde{B}^\setc)
    {\le 2n\exp(-c't^2dn(p+q)) + \frac{1}{n^{c}}},
\end{align*}
for any $c>0$ and some $c'>0$. Subsequently, we have
\begin{align}
    \Prob(\tilde{B}\cap Q) \ge 1 - \Prob(\tilde{B}^\setc) - \Prob(Q^\setc) \ge 1 - \frac{2}{n^c} - 2n\exp(-c't^2dn(p+q)).\label{eq:total-prob-trick}
\end{align}
We now choose
$t=\sqrt{\frac{K\log n}{dn(p+q)}}$
for any large constant $K>0$ to obtain
\[
\Prob(\tilde{B}\cap Q) \ge 1 - \frac{2}{n^{c}} - \frac{2}{n^{c'K-1}}.
\]
We now observe that on the event $\tilde{B}\cap Q$, we have
\[
    |(\Xt_i - m(i))\cdot\w| =
    O\left(\sqrt{\frac{\log n}{dn(p+q)}}\right),
\]
from which the result follows upon recalling \eqref{eq:mi-0} and \eqref{eq:mi-1}.
\end{proof}

\subsection{Rate of decay of the loss for chosen parameters}
Here we show that there exists a choice of parameters $(\wt,\bt)$ such that the loss incurred on any sample $(A,X)\sim\CSBM(n,p,q,\muv,\nuv)$ is exponentially small with a high probability.

\begin{lemma}\label{lem:loss-wt}
Consider the following parameters that satisfy the constraints of the problem in \eqref{eq:OPT}.
\begin{align*}
\wt(R) = \frac{R}{2\gamma}(\nuv-\muv),\qquad
\bt(R) = -\frac{\inner{\muv+\nuv,\wt(R)}}2,
\end{align*}
where $\gamma=\frac12\norm{\muv-\nuv}$.
Consider a sample $(A,X)\sim\CSBM(n,p,q,\muv,\nuv)$ such that $p>q$. In the regime where $\gamma=\Omega\Big(\frac{\log n}{dn(p+q)}\Big)$, we have for any $0<\beta_1,\beta_2\leq 1/2$, $R>0$, and some $c>0$ large enough that with probability at least $1 -n^{-c}$,
\[
L(A,X,\wt,\bt) = \exp\left(-R\gamma \Gamma(p,q)(1 + o(1))\right).
\]
\end{lemma}
\begin{proof}
For readability, we suppress the dependence of $R$ in $(\wt,\bt)$ when it is clear from context. Consider the loss for a single node $i$ for which we know the label $y_i$,
\begin{align*}
    L_i(A,X,\wt,\bt) &= -y_i\log(\sigma(\inner{\Xt_i,\wt}+\bt)) - (1-y_i)\log(1-\sigma(\inner{\Xt_i,\wt}+\bt))\\
    &= \log\left(1+\exp\Big((1-2\veps_i)(\inner{\Xt_i,\wt}+\bt)\Big)\right).
\end{align*}
We will work on the case where $\veps_i=0$ as the analysis for $\veps_i=1$ is symmetric.
Using \Cref{lem:xt-i-concentration}, it follows
that with probability at least $1-O(n^{-c})$ for any $c$, we have that for all $i\in[n]$,
\begin{align*}
\inner{\Xt_i,\wt} + \bt &= \frac{\inner{p\muv + q\nuv,\wt}}{p+q}(1 + o(1)) + O\left(\norm{\wt}\sqrt{\frac{\log n}{dn(p+q)}}\right) + \bt \\
&= \frac{p-q}{2(p+q)}\inner{\muv-\nuv,\wt}(1 + o(1)) + o(\norm{\wt}) \\
&= -\norm{\wt}\gamma\Gamma(p,q)(1 + o(1)),
\end{align*}
where the error terms here are uniform in $i$.
In the second equation, we have used the definition of $\tilde{b}$ and 
the fact that the error term in the first equation is $o(\norm{\wt})$ by combining \Cref{assumption:1} and \Cref{assumption:2}. The third equality follows from the definition of $\Gamma(p,q)$ and the assumption in the statement of the lemma that $\gamma=\Omega\Big(\frac{\log n}{dn(p+q)}\Big)$. The expression is symmetric for $\veps_i=1$. Hence, with probability at least $1 - O(n^{-c})$, we have that for all $i$, and all $R>0$,
\begin{equation}\label{eq:xt-wt-bt}
    \inner{\Xt_i,\wt(R)} + \bt(R) = (2\veps_i-1)R\gamma\Gamma(p,q)(1 + o_d(1)),
\end{equation}
where the error term is uniform in $i$.
On this event, we have for each $i$ that
\[
L_i(A,X,\wt(R),\bt(R)) = \log\left(1+\exp\Big(-R\gamma\Gamma(p,q)(1+o(1))\Big)\right).
\]
Thus the total loss is given by
\begin{align*}
    L(A,X,\wt,\bt) = \frac1{|S|}\sum_{i\in S}L_i(A,X,\wt,\bt)
    = \log\left(1+\exp\Big(-R\gamma\Gamma(p,q)(1+o(1))\Big)\right).
\end{align*}
Observe that for $x<0$, we have that
\begin{equation}\label{eq:log-1-ex-property}
e^{x-1}\le \log(1+e^x)\le e^x.
\end{equation}
Hence, we conclude that
\[
L(A,X,\wt(R),\bt(R)) = \exp\left(-R\gamma \Gamma(p,q)(1 + o(1))\right). \qedhere
\]
\end{proof}

\subsection{Proof of part 2 of  Theorem 1 }
We now turn to show the improvement achieved through the graph convolution.
\begin{proof}[Proof of \Cref{thm:thresholds} part 2]
We begin by observing that conditionally on $A$ and $(\veps_k)_{k\in[n]}$, $\Xt_i$ are Gaussian vectors with independent entries and have mean and covariance
\begin{align*}
    \E\Big(\tilde{X}_i\mid A,\eps\Big) &= m(i) = \frac{1}{D_{ii}}\Big(\sum_{j\in[n]} \at_{ij}\E[X_j\mid \veps]\Big),\\
    \Cov(\tilde{X}_i\mid A,\veps) &= \frac{1}{dD_{ii}}I.
\end{align*}
To have linear separability, we need that there is some unit vector $\vv\in\R^d$ and $b\in\R$ such that
\begin{equation}\label{eq:lin-sep-convolved}
\begin{cases}
\inner{m(i),\vv} + \frac{1}{\sqrt{dD_{ii}}} g_i(\vv)+ b < 0 & i\in S_0,\\
\inner{m(i),\vv} + \frac{1}{\sqrt{dD_{ii}}} g_i(\vv)+ b > 0 & i \in S_1.
\end{cases}
\end{equation}
We now turn to the event $\tilde{B}(\delta,\delta)$ from \eqref{eq:B-prob-with-ngbd}. On this event we have that $m(i)$ is given by \eqref{eq:mi-0} and \eqref{eq:mi-1}.
Note that for some $C,c>0$ we have
\[
\Prob(\max|g_i(\vv)|>K\sqrt{\log n})\leq C\exp(-cK^2),
\]
where we have used Borell's inequality \cite[Section 2.1]{Adler:2007} and the fact that for $n$ standard Gaussians $(z_i)$,
\[
\E[\max_{i\in [n]} |g_i(\vv)|]=\E\max_i |z_i|\leq c\sqrt{\log n}.
\]
We take $K=C'\sqrt{\log n}$ for some large constant $C'>0$ so that this probability is $O(1/n^\alpha)$ for some $\alpha>0$. Fix $\vv=\frac{\nuv-\muv}{\norm{\muv-\nuv}}$ and $b=-\frac12\inner{\muv+\nuv,\vv}$. Then using the degree concentration from \eqref{eq:B-prob-with-ngbd} and the fact that $\gamma = \omega(\frac{\log n}{\sqrt{dn(p+q)}})$, we have
\[
\begin{cases}
    -\gamma\Gamma(p,q)(1 + o(1)) + O(\frac{\log n}{\sqrt{d n(p+q)}}) < 0
    & i\in S_0,\\
    \gamma\Gamma(p,q)(1 + o(1)) + O(\frac{\log n}{\sqrt{d n(p+q)}}) > 0
    & i\in S_1.
\end{cases}
\]
Now to bound the loss, we take a multiple of the unit vector above, $\vv=\frac{R}{2\gamma}(\nuv-\muv)$, where $R$ is the norm constraint in \eqref{eq:OPT}. Then the bound on the loss follows directly from \Cref{lem:loss-wt}.
\end{proof}

\subsection{Proof of part 3 of Theorem 1}
We now provide the non-separability threshold for the convolved data, $\tilde{X}$.
\begin{proof}[Proof of \Cref{thm:thresholds} part 3]
There are $N$ examples in the dataset, each drawn iid from the same CSBM. For the $k$th example, define $\Xt_{k,i} = m_k(i) + \frac1{\sqrt{d}D^{(k)}_{ii}}\sum_{j\in[n]}g_{k,j}$, where $g_{k,j}$ are standard normal random vectors. For the sake of clarity we drop the index $k$ and look at the expressions for a fixed $k$. We will bound the probability of linear separability for a single example and use it to calculate the bound for all $N$ examples.

Consider a single example, $k$, out of the $N$ examples. We know from before that with probability at least $1-1/\poly(n)$ with any choice of degree for the polynomial, the value of $m(i)$ is given by \cref{eq:mi-0,eq:mi-1}. Note that for successful classification of this example, we require that for some fixed unit vector $\w$ and bias $b$ we have $\w\cdot\Xt_i + b < 0\;\text{for}\; i\in C_0$ and $\w\cdot\Xt_i + b > 0\;\text{for}\; i\in C_1$ over all $k\in [N]$. These conditions are equivalent to the event that for all $k\in[N]$,
\begin{align}
\frac{1}{p+q}\inner{\w, p\muv + q\nuv}(1\pm o(1)) + \max_{i\in C_0} \frac1{\sqrt{d}D_{ii}}\sum_{j\in[n]}a_{ij}\inner{g_j,\w} + b < 0,\\
\frac{1}{p+q}\inner{\w, q\muv + p\nuv}(1\pm o(1)) + \min_{i\in C_1} \frac1{\sqrt{d}D_{ii}}\sum_{j\in[n]}a_{ij}\inner{g_j,\w} + b > 0,
\end{align}
where the error term $o(1) = O(\frac1{\sqrt{\log n}})$.
Set $b=-\frac12\inner{\w,\muv+\nuv} + b'$ and observe that the above conditions are equivalent to:
\begin{align}
\left(\frac{p-q}{p+q}\right)\inner{\w,\muv - \nuv}(1\pm o(1)) + \max_{i\in C_0} \frac1{\sqrt{d}D_{ii}}\sum_{j\in[n]}a_{ij}\inner{g_j,\w} + b' < 0,\\
\left(\frac{q-p}{p+q}\right)\inner{\w,\nuv - \muv}(1\pm o(1)) + \min_{i\in C_1} \frac1{\sqrt{d}D_{ii}}\sum_{j\in[n]}a_{ij}\inner{g_j,\w} + b' > 0,
\end{align}
Now we observe that the above two conditions imply that at least one of them holds with $b'=0$. Denote $\Delta=\frac{n}{2}(p+q)$ and $T=C_0$ if $|C_0|\le|C_1|$, $T=C_1$ otherwise. Then we can bound the probability of the above event by the probability:
\begin{align}
& \Prob\left(\max_{i\in T} \frac1{\sqrt{d}D_{ii}}\sum_{j\in[n]}a_{ij}\inner{g_j,\w} \le \Gamma(p,q)\abs{\inner{\w,\nuv-\muv}}(1\pm o(1))\right)&& \text{using } \Gamma(p,q)=\frac{p-q}{p+q}\\
&\le \Prob\left(\max_{i\in T} \frac1{\sqrt{d}D_{ii}}\sum_{j\in[n]}a_{ij}\inner{g_j,\w} \le \Gamma(p,q) \norm{\muv-\nuv}_2(1\pm o(1))\right)&& \text{using Cauchy-Schwarz}\\
&\le \Prob\left(\max_{i\in T} \inner{Z_i,\w} \le \frac{K\Gamma(p,q)}{\sqrt{\Delta}}(1\pm o(1))\right)&& \text{since } \gamma\le \frac{K}{\sqrt{d\Delta}},\label{eq:prob-sep-graph-conv}
\end{align}
where the random vectors $Z_i = \frac1{D_{ii}}\sum_{j\in[n]}a_{ij}g_j\sim\Nc(\zero, \frac{1}{D_{ii}}I_d)$.
We will now utilize Sudakov's minoration inequality~\cite[Section 7.4]{Vershynin:2018} to obtain a lower bound on the expected supremum of the random process $\{\inner{Z_i,\w}\}_{i\in C_0}$, and then use Borell's inequality~\cite[Section 2.1]{Adler:2007} to upper bound the probability in \cref{eq:prob-sep-graph-conv}.

Denote the set $J_{ij} = (N_i\cup N_j)\setminus (N_i\cap N_j)$, and note that
\[\inner{Z_i,\w}-\inner{Z_j,\w} = \frac{(1\pm o(1))}{\Delta}\sum_{l\in J_{ij}}\inner{g_l, \w}.\]
To apply Sudakov's minoration result, we also define the canonical metric over the index set $T$ for any $i,j\in T$:
\begin{align}
d_T(i,j) =\sqrt{\E[(\inner{Z_i,\w} - \inner{Z_j,\w})^2]} = \frac{\sqrt{|J_{ij}|}}{\Delta}(1\pm o(1)).
\end{align}
For any $i,j\in T$ with $i\neq j$ and a node $l$, the probability of $l$ being a neighbor of exactly one of $i,j$ is $2p(1-p)$ if $l\in C_0$ and $2q(1-q)$ if $l\in C_1$. 
Thus  $|J_{ij}|$ is a sum of independent Bernoulli random variables and $\E|J_{ij}| = n(p(1-p) + q(1-q))$. Hence, by the multiplicative Chernoff bound we obtain that for any $\delta\in(0, 1)$,
\[
    \Prob(\abs{|J_{ij}| - \E|J_{ij}|} > \delta\E|J_{ij}|) \le 2\exp\left(-\frac{\delta^2\E|J_{ij}|}{3}\right).
\]
Since $p, q = \omega(\frac{\log^2 n}{n})$, we have that $\E|J_{ij}| = n(p(1-p) + q(1-q))=\omega(\log^2 n)$. Therefore, choosing $\delta = \frac{\sqrt{C\log n}}{\E|J_{ij}|}$ for any large constant $C$, we obtain that with probability at least $1 - 1/\poly(n)$, $|J_{ij}| \ge n(p+q - p^2 -q^2)(1-\delta) = n(p+q - p^2 -q^2)(1-o(1)) = \Omega(\Delta)$. Therefore, we have that
\begin{equation}
d_T(i, j) = \frac{\sqrt{|J_{ij}|}}{\Delta}(1\pm o(1)) = \Omega\left(\frac{1}{\sqrt{\Delta}}\right).
\end{equation}
Under the above mentioned event, let $\eps_0 = \min_{i,j\in S}d_T(i,j)$. Then for an $\eps_0$-covering of the set, we need every point of the set, i.e., $N(T, d_T, \eps_0) = |T|$. Putting this information in Sudakov's minoration inequality, we obtain that
$\E[\max_i \inner{Z_i,\w}] \ge c\eps_0\sqrt{\log n}$.
Since $\eps_0 = \Omega(1/\sqrt{\Delta})$, this gives us that for some suitable $c>0$,
\begin{equation}\label{eq:minoration}
\E[\max_i \inner{Z_i,\w}] \ge c\sqrt{\frac{\log n}{\Delta}}.
\end{equation}
We now use Borell's inequality~\cite[Section 2.1]{Adler:2007} to obtain that for any $t>0$,
\begin{align}
    &\Prob(\max_{i\in C_0}\inner{Z_i,\w} \le \E \max_{i\in C_0}\inner{Z_i,\w} - t) \le 2\exp(-t^2D_{ii})\\
    \implies \quad &\Prob\left(\max_{i\in C_0}\inner{Z_i,\w} \le c\sqrt{\frac{\log n}{\Delta}} - t\right) \le 2\exp(-t^2D_{ii})\quad \text{using \cref{eq:minoration}.}
\end{align}
Choose $t=c\sqrt{\frac{\log n}{\Delta}} - \frac{K\Gamma(p,q)}{\sqrt{\Delta}} = \Omega\left(\sqrt{\frac{\log n}{\Delta}}\right)$ and combine with the event for class-size and degree concentration from \cref{eq:B-prob-with-ngbd}, so that for some constant $c'>0$,
\begin{align}
    \Prob\left(\max_{i\in T}\inner{Z_i,\w} \le \frac{K\Gamma(p,q)}{\sqrt{\Delta}}\right)\le 2n^{-c} = o(1).
\end{align}
Recall that we have $N$ independent examples in the dataset, hence, the probability that all examples are correctly classified by a fixed unit vector $\w$ is given by
\begin{align}
    \Prob\left(\max_{i\in T}\inner{Z_i,\w} \le \frac{K\Gamma(p,q)}{\sqrt{\Delta}}\right)^N \le \left(\frac{2}{n^{c}}\right)^N.\label{eq:prob-sep-w}
\end{align}
We now reintroduce the index $k\in[N]$ of the example in the given dataset. Let $\{\inner{Z_{k,i},\w}\}_{(k,i)\in[N]\times T}$ be the process we're looking at and define the following event for any $t>0$:
\begin{align}
    E_{t,n}(\w) &= \{\max_{i\in T}:\inner{Z_{k,i},\w} \le t \; \forall k\in [N]\}.
\end{align}
We now follow a proof strategy similar to \cref{lem:isonormal-process}. Consider a sufficiently large constant $L$ and a fixed $\eps<t$. Define $\eps'=\frac{\eps \sqrt{\Delta}}{L\sqrt{d}}$. Let $\Sigma_{\eps',d}$ denote an $\eps'$-net of the unit sphere $\bS^{d-1}$. Also define the event $Q=\{\norm{Z_{k,i}}\le L\sqrt{d/\Delta}\;\; \forall (k,i)\in [N]\times T\}$ and note that
\begin{align}
    &\Prob\left(\exists \w\in \bS^{d-1}:\max_{i\in C_0}\inner{Z_{k,i},\w} \le t\;\forall k\in [N]\right)\\
    &= \Prob\left(\bigcup_{\w\in \bS^{d-1}} E_{t,n}(\w) \cap Q^\setc\right) + \Prob\left(\bigcup_{\w\in \bS^{d-1}} E_{t,n}(\w)\cap Q\right)\\
    &\le \Prob(Q^\setc) + \Prob\left(\bigcup_{\vv\in \Sigma_{\eps',d}} E_{t+\eps,n}(\vv)\right),\label{eq:sum-of-terms-prob-sep}
\end{align}
where in the last inequality we used the following fact: For any $\w\in\bS^{d-1}$, if $\vv\in \Sigma_{\eps',d}$ is the vector nearest to $\w$ in the $\eps'$-net, then we have that $\norm{\vv-\w}_2\le \eps'$, so the event $E_{t,n}(\w)\cap Q$ implies that
\[\inner{Z_i, \vv-\w}\le \norm{Z_i}_2\norm{\vv-\w}_2\le L\sqrt{\frac{d}{\Delta}}\eps' = \eps.\]
Thus, we obtain that $\inner{Z_i, \vv}\le \inner{Z_i,\w}+\eps \le t+\eps$, and so
\[
\bigcup_{\w\in \bS^{d-1}}E_{t,n}(\w)\cap Q\; \bigsubseteq \bigcup_{\vv\in \Sigma_{\eps',d}}E_{t+\gamma,n}(\vv).\]
We now bound the two terms in \cref{eq:sum-of-terms-prob-sep} separately. For the first term, observe that by standard Gaussian concentration and a union bound over all pairs $(k,i)\in [N]\times T$, we have that there exist $C,C_1,C_2,c_1,c_2>0$ such that for any $c>0$,
\begin{align*}
\Prob(Q^\setc) &= \Prob\left(\exists k\in [N], i\in T: \norm{Z_{k,i}}_2 > L\sqrt{\frac{d}{\Delta}}\right)\\
&\le  C_1Nn\exp(-c_1dL^2) + \frac{C_2}{(nN)^{c_2}}\le \frac{C}{(nN)^{c}},
\end{align*}
where we've used degree concentration from \cref{eq:B-prob-with-ngbd} and the assumption that $nN=O(\poly(d))$. For the second term, observe that
\begin{align}
    \Prob\left(\bigcup_{\vv\in \Sigma_{\eps',d}} E_{t+\eps,n}(\vv)\right) &\le |\Sigma_{\eps',d}|\Prob(E_{t+\eps,n}(\vv))
    \le \exp\left(C'd\log\Big(\frac{d}{\eps\sqrt{\Delta}}\Big)\right) \Prob(E_{t+\eps,n}(\vv)),
\end{align}
where we used a union bound in the first inequality and an upperbound on the size of the $\eps'$-net in the second inequality.
Since the above holds for any $t>0$ and $\eps\le t$, we set $t=\eps=\frac{K\Gamma(p,q)}{2\sqrt{\Delta}}$ to obtain that the event $E_{t+\eps,n}(\vv)$ is equivalent to the event from \cref{eq:prob-sep-w}. Hence, we bound the probability above by
\begin{align}
    \Prob\left(\bigcup_{\vv\in \Sigma_{\eps',d}} E_{t+\eps,n}(\vv)\right) \le \exp\left(C'd\log\Big(\frac{2d}{K\Gamma(p,q)}\Big)\right) \left(\frac{2}{n^c}\right)^N \le \exp(-cd\log d)
\end{align}
for some $c>0$. In the last inequality above, we used that $\Gamma(p,q)=\Theta(1)$ and $K$ is a constant, along with the assumption that $N\log n = \omega(d\log d)$.
\end{proof}

\section{Generalization}
\label{generalization}
In this section we provide the proof for \Cref{thm:generalization}.
\subsection{Characterizing the optimizer}
We begin by characterizing $\w^*$, the optimizer of \eqref{eq:OPT}.
Define the following quantities.
\begin{equation}\label{eq:m0-m1}
    m_0 = \frac{p\muv + q\nuv}{p+q}, \qquad
    m_1 = \frac{q\muv + p\nuv}{p+q}.
\end{equation}
\begin{lemma}\label{lem:w*-aligns-with-means}
For any $R>0$, let $(\w^*(R),b^*(R))$ be the optimizer to the problem in \eqref{eq:OPT} for a given training sample $(A,X)\sim\CSBM(n,p,q,\muv,\nuv)$ with $\muv,\nuv\in\R^d$ and with norm constraint $R$. Consider the regime where $\gamma=\frac12\norm{\muv-\nuv}_2=\Omega\Big(\frac{\log n}{dn(p+q)}\Big)$. Then for any $c>0$ fixed but  large enough, with probability at least $1 -n^{-c}$ we have that for any $R>0$,
\begin{equation}\label{eq:w-star}
\w^*(R) = \frac{R}{2\gamma}(\nuv-\muv)(1 - o(1)),
\end{equation}
and that
\begin{align}
\inner{m_0,\w^*(R)} + b^*(R) &\le -R\gamma\Gamma(p,q)(1 - o(1)),\label{eq:m0-train}\\
\inner{m_1,\w^*(R)} + b^*(R) &\ge R\gamma\Gamma(p,q)(1 - o(1)).\label{eq:m1-train}
\end{align}
\end{lemma}
\begin{proof}
Fix $R>0$ and let $(\w^*(R),b^*(R))$ be the solutions to the problem in \eqref{eq:OPT} with norm constraint $R$. Let the training sample be $(A,X)\sim\CSBM(n,p,q,\muv,\nuv)$. Then we have that
\[
\OPT_d(A,X,R) = L(A,X,\w^*,b^*) \le L(A,X,\wt,\bt),
\]
where $(\wt,\bt)$ are defined in \Cref{lem:loss-wt}. Let $\Xt_i = (D^{-1}\At X)_i$. Now we focus our scope to the event that for every $i\in[n]$, and $R>0$
\[
    \inner{\Xt_i,\wt} + \bt = (2\veps_i-1)R\gamma\Gamma(p,q)(1+o(1)).
\]
Note that from \eqref{eq:xt-wt-bt} this event occurs with probability at least $1 - n^{-c}$ for $c$ large but $O(1)$. Since $(\w^*,b^*)$ are solutions to \eqref{eq:OPT}, on this event we have for all $i$ that
\begin{align*}
    \inner{\Xt_i,\w^*} + b^* &\le -R\gamma\Gamma(p,q)(1-o(1)) && \text{ for } \veps_i=0,\\
    \inner{\Xt_i,\w^*} + b^* &\ge R\gamma\Gamma(p,q)(1-o(1)) && \text{ for } \veps_i=1.
\end{align*}
Note that \Cref{lem:xt-i-concentration} implies that with probability at least $1 -  n^{-c}$, for all $i$ we also have
\[
|\inner{\Xt_i-m_{\veps_i}(1+o(1)),\w^*}| \le O\left(\norm{\wt}\sqrt{\frac{\log n}{dn(p+q)}}\right)
\]
Since $\norm{\w^*}\le R$, we conclude that
\begin{align*}
\inner{m_0,\w^*} + b^* &\le -R\gamma\Gamma(p,q)(1 - o(1)),\\
\inner{m_1,\w^*} + b^* &\ge R\gamma\Gamma(p,q)(1 - o(1)).
\end{align*}
that is, \eqref{eq:m0-train} and \eqref{eq:m1-train} hold as desired. It remains to show \eqref{eq:w-star}. Subtracting \eqref{eq:m0-train} from \eqref{eq:m1-train} we obtain
\begin{equation}\label{eq:m1-mminus-m0}
    \inner{m_1 - m_0,\w^*} = \frac{p-q}{p+q}\inner{\nuv-\muv,\w^*} \ge 2R\gamma\Gamma(p,q)(1 - o(1)).
\end{equation}
This implies that $\norm{\w^*} \ge R(1 - o(1))$. Since $\norm{\w^*}\le R$ from the optimization constraint, we have
\[
1 - o(1)\le \frac{\inner{\nuv-\muv,\w^*}}{\norm{\muv-\nuv}\norm{\w^*}} \le 1.\qedhere
\]
\end{proof}

\subsection{Proof of Theorem 2}
Now we turn to the proof of \Cref{thm:generalization}.

\begin{proof}[Proof of \Cref{thm:generalization}]
Consider a test sample $(A',X')\sim\CSBM(n',p',q',\muv,\nuv)$. Let $\Xt'$ be the corresponding convolution $D'^{-1}\At' X'$. Similar to \eqref{eq:mi-0}, \eqref{eq:mi-1} and \eqref{eq:m0-m1} we also define $m'(i)$, $m'_0$ and $m'_1$ corresponding to the sample $(A',X')$.
We restrict our calculations to the case where $\veps_i=0$. Note that
\[
    m'_0 - m_0 = \frac{qp' - pq'}{(p+q)(p'+q')}(\muv-\nuv).
\]
From \Cref{lem:xt-i-concentration,lem:w*-aligns-with-means}, we see that for $c'>0$ large but $O(1)$, with probability at least $1 - (n')^{-c'}$ we have that for any $R>0$
\[
\inner{\Xt'_i,\w^*} = \inner{m'_{\veps_i}, \w^*}(1 + o(1))
\]
for some $\eta>0$ and for all $i\in[n']$.

Let $\gamma=\frac{\norm{\muv-\nuv}}2$. Therefore, by the same lemmas, we have that for any $c,c'>0$ large enough, with probability $1 - O(1/(n')^{c'}+1/n^{c})$, when $\veps_i=0$
\begin{align*}
    \inner{\Xt'_i,\w^*} + b^* &= \inner{m_0', \w^*}(1 + o(1)) + b^*\\
    &= \inner{m_0'-m_0, \w^*}(1 + o(1)) + \inner{m_0, \w^*}(1 + o(1)) + b^*\\
    &\le \frac{qp' - pq'}{(p+q)(p'+q')}\inner{\muv-\nuv,\w^*} - R\gamma\Gamma(p,q)(1 - o(1))\\
    &\le -\frac{4d\gamma^2(qp' - pq')}{2\gamma(p+q)(p'+q')}(1 - o(1)) - R\gamma\Gamma(p,q)(1 - o(1))\\
    &= R\gamma\left(\frac{2(pq' - qp')}{(p+q)(p'+q')} - \frac{p-q}{p+q}\right)(1 - o(1))\\
    &= -R\gamma\Gamma(p',q')(1 - o(1)).
\end{align*}
The first inequality above uses \eqref{eq:m0-train}, while the second inequality follows from \Cref{lem:w*-aligns-with-means}. Similarly, for $\veps_i=1$ we obtain
\begin{align*}
    \inner{\Xt'_i,\w^*} + b^* &\ge R\gamma\Gamma(p',q')(1 - o(1)).
\end{align*}
The loss is then given by
\begin{align*}
    L(A',X',\w^*,b^*) &= \frac1{n'}\sum_{i\in[n']}\log\left(1+\exp\Big((1-2\veps_i)(\inner{\Xt',\w^*} + b^*)\Big)\right)\\
    &\le \frac1{n'}\sum_{i\in[n']}\log\left(1+\exp\Big(-R\gamma\Gamma(p',q')(1 - o(1))\Big)\right)
\end{align*}
Now it follows from \eqref{eq:log-1-ex-property} that on this event,
\[
L(A',X',\w^*,b^*) \le C\exp\left(-R\gamma \Gamma(p',q')(1 - o(1))\right).\qedhere
\]
\end{proof}

\section{Conclusion and Future Work}
\label{conclusion}
In this work we study the benefits of graph convolution for the problem of semi-supervised classification of data. Using the contextual stochastic block model we show that graph convolution can transform data which is not linearly separable into data which is linearly separable. However, we also show empirically that graph convolution can be disadvantageous if the intra-class edge probability is close to the inter-class edge probability. Furthermore, we show that a classifier trained on the convolved data can generalize to out-of-distribution data which have different intra- and inter-class edge probabilities. 

Our work is only the first step in understanding the effects of graph convolution for semi-supervised classification. There is still a lot of future work to be done. Below we indicate two questions that need to be addressed.
\begin{enumerate}
    \item Graph neural networks~\cite{HamilBook} have recently dominated practical aspects of relational machine learning. A lot of these models utilize graph convolution in the same way that we do in this paper. However, the key point of these models is to utilize more than $1$ layers in the graph neural network. It is still an open question to understand the benefits of graph convolution for these highly non-linear models for semi-supervised node classification.
    \item Our analysis holds for graphs with average number of neighbors at least $\omega(\log^2 n)$. Since a lot of large-scale data consist of sparse graphs it is still an open question to extend our results to sparser graphs where the average number of neighbors per node is $O(1)$.
\end{enumerate}

We end by noting here that while we only study the two class setting, we expect that our arguments  extend to the $k$-class setting with $k=O(1)$ with only minor modifications under natural assumptions.\footnote{ 
Let us briefly sketch how our arguments extend to the $k$-class setting. As we are only interested in the relative performance of logistic regression as compared to graph convolutions, the key issue is to show that $k-1$ one-v.s.-all linear classifiers fail or perform well. The failure threshold for logistic regression evidently generalizes as if any two means are closer than $1/\sqrt{d}$, a one-v.s.-all classifier must incorrectly classify a large fraction of the corresponding samples, a similar argument holds if one of the means is not an extreme point of the convex hull of the means as can be seen, e.g., in the case of a mixture of three Gaussians whose means are co-linear. Furthermore, if the means are sufficiently far apart for the graph convolution to work in the two-class setting, a similar argument yields that the one-v.s.-all classifier separates there as well provided the means satisfy certain simple geometric constraints, e.g., the convex hull condition mentioned earlier. Once one shows that the classifiers perform well or poorly, the corresponding loss bounds are immediate as in the two-component case. }

\section*{Acknowledgements}
\label{acknowledgements}
K.F. would like to acknowledge the support of the Natural Sciences and Engineering Research Council of Canada (NSERC). Cette recherche a \'et\'e financ\'ee par le Conseil de recherches en sciences naturelles et en g\'enie du Canada (CRSNG), [RGPIN-2019-04067, DGECR-2019-00147].

A.J. acknowledges the support of the Natural Sciences and Engineering Research Council of Canada (NSERC). Cette recherche a \'et\'e financ\'ee par le Conseil de recherches en sciences naturelles et en g\'enie du Canada (CRSNG),  [RGPIN-2020-04597, DGECR-2020-00199].

\bibliography{references}

\begin{thebibliography}{10}

\bibitem{Abbe2018}
E.~Abbe.
\newblock Community detection and stochastic block models: Recent developments.
\newblock {\em Journal of Machine Learning Research}, 18:1--86, 2018.

\bibitem{abbe2015exact}
E.~Abbe, A.~S. Bandeira, and G.~Hall.
\newblock Exact recovery in the stochastic block model.
\newblock {\em IEEE Transactions on Information Theory}, 62(1):471--487, 2015.

\bibitem{abbe2015community}
E.~Abbe and C.~Sandon.
\newblock Community detection in general stochastic block models: Fundamental
  limits and efficient algorithms for recovery.
\newblock In {\em 2015 IEEE 56th Annual Symposium on Foundations of Computer
  Science}, pages 670--688, 2015.

\bibitem{abbe2018proof}
E.~Abbe and C.~Sandon.
\newblock Proof of the achievability conjectures for the general stochastic
  block model.
\newblock {\em Communications on Pure and Applied Mathematics},
  71(7):1334--1406, 2018.

\bibitem{Adler:2007}
R~J Adler and J~E Taylor.
\newblock Gaussian inequalities.
\newblock In {\em Random Fields and Geometry}, chapter~2, pages 49--64.
  Springer New York, New York, NY, 2007.

\bibitem{banks2016information}
J.~Banks, C.~Moore, J.~Neeman, and P.~Netrapalli.
\newblock Information-theoretic thresholds for community detection in sparse
  networks.
\newblock In {\em Conference on Learning Theory}, pages 383--416. PMLR, 2016.

\bibitem{battaglia:graphnets}
P.~Battaglia, R.~Pascanu, M.~Lai, D.~J. Rezende, and K.~Kavukcuoglu.
\newblock {Interaction Networks for Learning about Objects, Relations and
  Physics}.
\newblock In {\em Advances in Neural Information Processing Systems (NeurIPS)},
  2016.

\bibitem{BVR17}
N.~Binkiewicz, J.~T. Vogelstein, and K.~Rohe.
\newblock Covariate-assisted spectral clustering.
\newblock {\em Biometrika}, 104:361--377, 2017.

\bibitem{GB08}
V.~Blondel, S.~Boyd, and H.~Kimura.
\newblock Graph implementations for nonsmooth convex programs, recent advances
  in learning and control (a tribute to {M}. {V}idyasagar).
\newblock {\em Lecture Notes in Control and Information Sciences, Springer},
  pages 95--110, 2008.

\bibitem{bordenave2015non}
C.~Bordenave, M.~Lelarge, and L.~Massouli{\'e}.
\newblock Non-backtracking spectrum of random graphs: community detection and
  non-regular ramanujan graphs.
\newblock In {\em 2015 IEEE 56th Annual Symposium on Foundations of Computer
  Science}, pages 1347--1357. IEEE, 2015.

\bibitem{CLB19}
Z.~Chen, L.~Li, and J.~Bruna.
\newblock Supervised community detection with line graph neural networks.
\newblock In {\em International Conference on Learning Representations (ICLR)},
  2019.

\bibitem{CZY2011}
H.~Cheng, Y.~Zhou, and J.~X. Yu.
\newblock Clustering large attributed graphs: A balance between structural and
  attribute similarities.
\newblock {\em ACM Transactions on Knowledge Discovery from Data}, 12, 2011.

\bibitem{DV2012}
T.~A. Dang and E.~Viennet.
\newblock Community detection based on structural and attribute similarities.
\newblock In {\em The Sixth International Conference on Digital Society
  (ICDS)}, 2012.

\bibitem{decelle2011asymptotic}
A.~Decelle, F.~Krzakala, C.~Moore, and L.~Zdeborov{\'a}.
\newblock Asymptotic analysis of the stochastic block model for modular
  networks and its algorithmic applications.
\newblock {\em Physical Review E}, 84(6):066106, 2011.

\bibitem{deshpande2015asymptotic}
Y.~Deshpande, E.~Abbe, and A.~Montanari.
\newblock Asymptotic mutual information for the two-groups stochastic block
  model.
\newblock {\em ArXiv}, 2015.
\newblock arXiv:1507.08685.

\bibitem{DSM18}
Y.~Deshpande, A.~Montanari S.~Sen, and E.~Mossel.
\newblock Contextual stochastic block models.
\newblock In {\em Advances in Neural Information Processing Systems (NeurIPS)},
  2018.

\bibitem{FL2019}
M.~Fey and J.~E. Lenssen.
\newblock Fast graph representation learning with {PyTorch Geometric}.
\newblock In {\em ICLR Workshop on Representation Lexarning on Graphs and
  Manifolds}, 2019.

\bibitem{GJJ20}
V.~Garg, S.~Jegelka, and T.~Jaakkola.
\newblock Generalization and representational limits of graph neural networks.
\newblock In {\em Advances in Neural Information Processing Systems (NeurIPS)},
  volume 119, pages 3419--3430, 2020.

\bibitem{GVB2012}
J.~Gilbert, E.~Valveny, and H.~Bunke.
\newblock Graph embedding in vector spaces by node attribute statistics.
\newblock {\em Pattern Recognition}, 45(9):3072--3083, 2012.

\bibitem{gilmer:quantum}
J.~Gilmer, S.~S. Schoenholz, P.~F. Riley, O.~Vinyals, and G.~E. Dahl.
\newblock Neural message passing for quantum chemistry.
\newblock In {\em Proceedings of the 34th International Conference on Machine
  Learning}, 2017.

\bibitem{cvx}
M.~Grant and S.~Boyd.
\newblock {CVX}: Matlab software for disciplined convex programming, version
  2.0 beta.
\newblock http://cvxr.com/cvx, 2013.

\bibitem{GFRT13}
S.~G{\"u}nnemann, I~F{\"a}rber, S.~Raubach, and T.~Seidl.
\newblock Spectral subspace clustering for graphs with feature vectors.
\newblock In {\em IEEE 13th International Conference on Data Mining}, 2013.

\bibitem{HamilBook}
L.~W. Hamilton.
\newblock Graph representation learning.
\newblock {\em Synthesis Lectures on Artificial Intelligence and Machine
  Learning}, 14(3):1--159, 2020.

\bibitem{HYL17}
W.~L. Hamilton, R.~Ying, and J.~Leskovec.
\newblock Inductive representation learning on large graphs.
\newblock {\em NIPS'17: Proceedings of the 31st International Conference on
  Neural Information Processing Systems}, pages 1025--1035, 2017.

\bibitem{holland83stochastic}
P.~W. Holland, K.~B. Laskey, and S.~Leinhardt.
\newblock Stochastic blockmodels: First steps.
\newblock {\em Social networks}, 5(2):109--137, 1983.

\bibitem{JLLHZ19}
D.~Jin, Z.~Liu, W.~Li, D.~He, and W.~Zhang.
\newblock Graph convolutional networks meet markov random fields:
  Semi-supervised community detection in attribute networks.
\newblock {\em Proceedings of the AAAI Conference on Artificial Intelligence},
  3(1):152--159, 2019.

\bibitem{kipf:gcn}
T.~N. Kipf and M.~Welling.
\newblock Semi-supervised classification with graph convolutional networks.
\newblock In {\em International Conference on Learning Representations (ICLR)},
  2017.

\bibitem{A2020}
A.~Loukas.
\newblock How hard is to distinguish graphs with graph neural networks?
\newblock In {\em Advances in Neural Information Processing Systems (NeurIPS)},
  2020.

\bibitem{ALoukas2020}
A.~Loukas.
\newblock What graph neural networks cannot learn: Depth vs width.
\newblock In {\em International Conference on Learning Representations (ICLR)},
  2020.

\bibitem{massoulie2014community}
Laurent Massouli\'{e}.
\newblock Community detection thresholds and the weak ramanujan property.
\newblock In {\em Proceedings of the Forty-Sixth Annual ACM Symposium on Theory
  of Computing}, page 694–703, 2014.

\bibitem{pmlr-v97-mehta19a}
N.~Mehta, C.~L. Duke, and P.~Rai.
\newblock Stochastic blockmodels meet graph neural networks.
\newblock In {\em Proceedings of the 36th International Conference on Machine
  Learning}, volume~97, pages 4466--4474, 2019.

\bibitem{montanari2016semidefinite}
A.~Montanari and S.~Sen.
\newblock Semidefinite programs on sparse random graphs and their application
  to community detection.
\newblock In {\em Proceedings of the forty-eighth annual ACM symposium on
  Theory of Computing}, pages 814--827, 2016.

\bibitem{moore2017csphysics}
C.~Moore.
\newblock The computer science and physics of community detection: Landscapes,
  phase transitions, and hardness.
\newblock {\em Bulletin of The European Association for Theoretical Computer
  Science}, 1(121), 2017.

\bibitem{mossel2015consistency}
E.~Mossel, J.~Neeman, and A.~Sly.
\newblock Consistency thresholds for the planted bisection model.
\newblock In {\em Proceedings of the forty-seventh annual ACM symposium on
  Theory of computing}, pages 69--75, 2015.

\bibitem{mossel2018proof}
E.~Mossel, J.~Neeman, and A.~Sly.
\newblock A proof of the block model threshold conjecture.
\newblock {\em Combinatorica}, 38(3):665--708, 2018.

\bibitem{scarselli:gnn}
F.~Scarselli, M.~Gori, A.~C. Tsoi, M.~Hagenbuchner, and G.~Monfardini.
\newblock The graph neural network model.
\newblock {\em IEEE Transactions on Neural Networks}, 20(1), 2009.

\bibitem{Vershynin:2018}
R.~Vershynin.
\newblock {\em High-Dimensional Probability: An Introduction with Applications
  in Data Science}, volume~47.
\newblock Cambridge University Press, 2018.

\bibitem{XHLJ19}
K.~Xu, W.~Hu, J.~Leskovec, and S.~Jegelka.
\newblock How powerful are graph neural networks?
\newblock {\em In International Conference on Learning Representations (ICLR)},
  2019.

\bibitem{yan:2021:two-sides}
Yujun Yan, Milad Hashemi, Kevin Swersky, Yaoqing Yang, and Danai Koutra.
\newblock Two sides of the same coin: Heterophily and oversmoothing in graph
  convolutional neural networks, 2021.

\bibitem{YML13}
J.~{Yang}, J.~{McAuley}, and J.~{Leskovec}.
\newblock Community detection in networks with node attributes.
\newblock In {\em 2013 IEEE 13th International Conference on Data Mining},
  pages 1151--1156, 2013.

\bibitem{YHCEHL18}
R.~Ying, R.~He, K.~Chen, P.~Eksombatchai, W.~L. Hamilton, and J.~Leskovec.
\newblock Graph convolutional neural networks for web-scale recommender
  systems.
\newblock {\em KDD '18: Proceedings of the 24th ACM SIGKDD International
  Conference on Knowledge Discovery \& Data Mining}, pages 974--983, 2018.

\end{thebibliography}
\bibliographystyle{plain}

\appendix
\section{Additional experiments}
\label{additional-experiments}
In this section we present the additional experiments. Our conclusions are similar to the ones made by the experiments in \cref{experiments}.

\subsection{Out-of-distribution generalization}
In this experiment we test the performance of the trained classifier on out-of-distribution datasets.
We perform this experiment for two different distances between the means, $16/\sqrt{d}$ and $2/\sqrt{d}$.
We train on a CSBM using various combinations of $p_{\train}$ and $q_{\train}$, while $p_{\train}>q_{\train}$. In all experiments we set $n=400$ and $d=60$. We test on CSBMs with $n=400$, $d=60$ and varying $p_{\test}$ and $q_{\test}$ while $p_{\test}>q_{\test}$. The results for distance of means equal to $2/\sqrt{d}$ are presented in \Cref{fig:CSBM-1}, and the results for distance between the means equal to $16/\sqrt{d}$ are presented in \Cref{fig:CSBM-2}.
\begin{figure}[ht!]
	\centering
	\begin{subfigure}[t]{1.5in}
		\centering
		\includegraphics[width=\columnwidth]{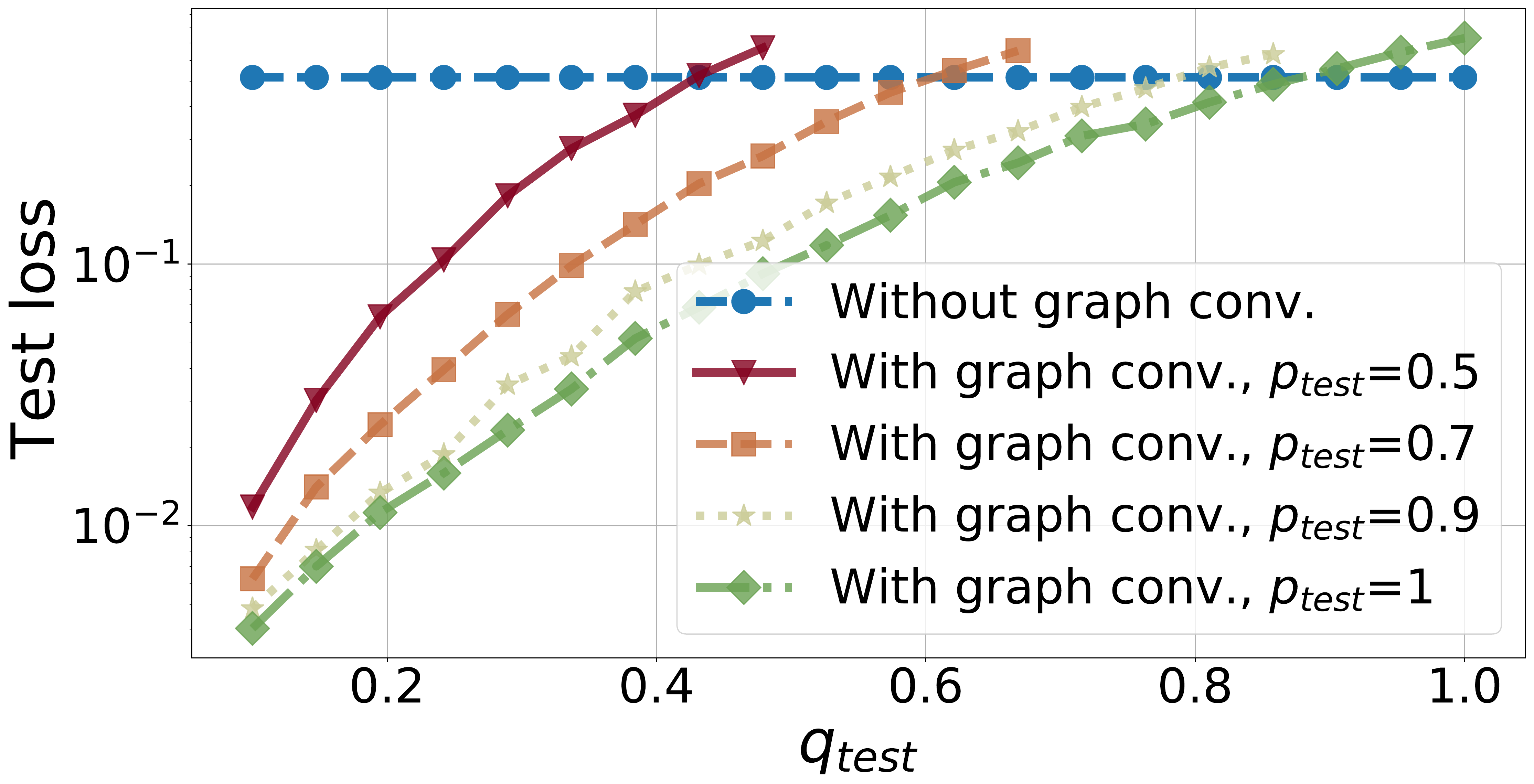}
		\caption{$p=0.5,q=0.38$}\label{fig:CSBM-1-1}	
	\end{subfigure}
	\begin{subfigure}[t]{1.5in}
		\includegraphics[width=\columnwidth]{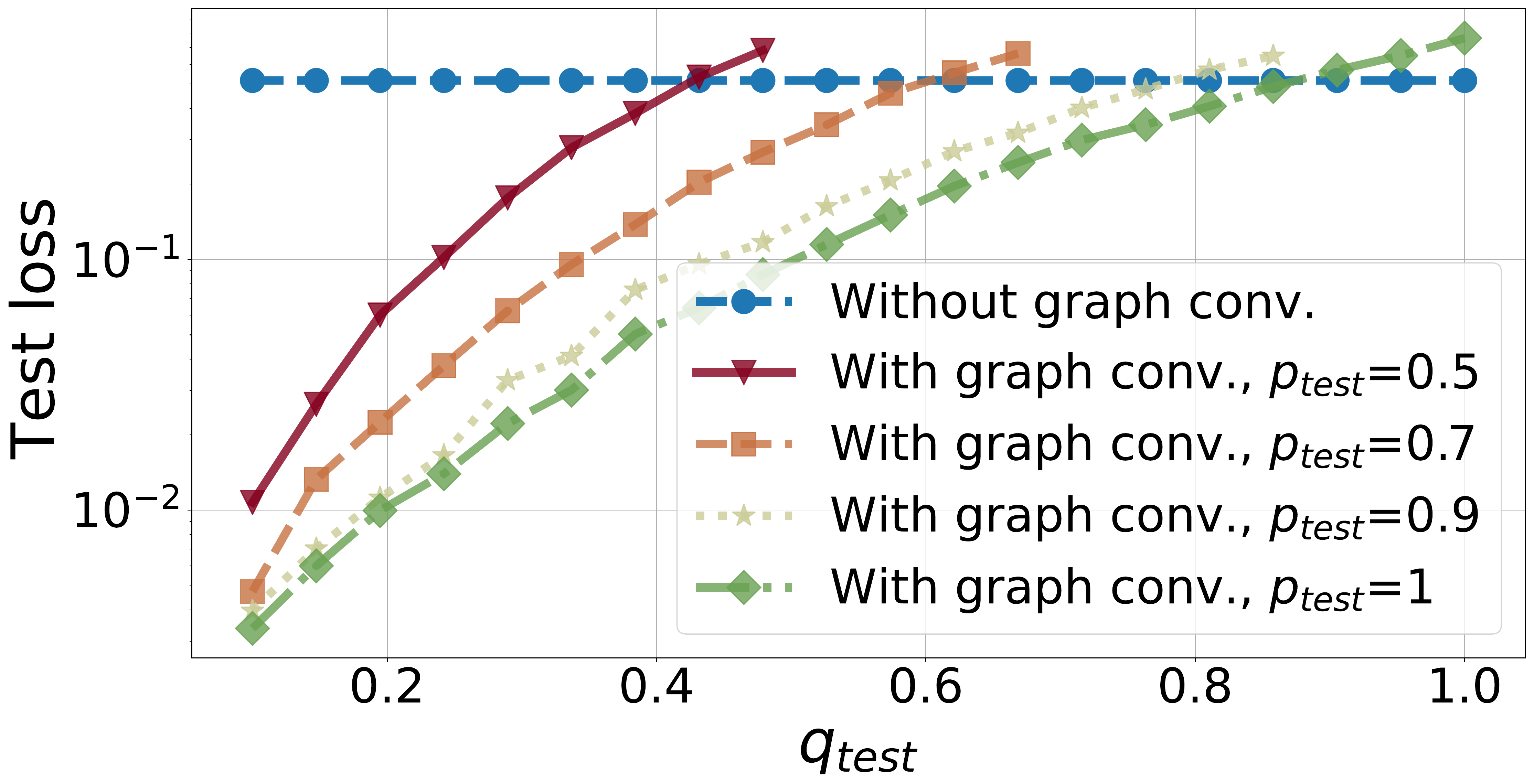}
		\caption{$p=0.7,q=0.1$}\label{fig:CSBM-1-2}
	\end{subfigure}
	\begin{subfigure}[t]{1.5in}
		\centering
		\includegraphics[width=\columnwidth]{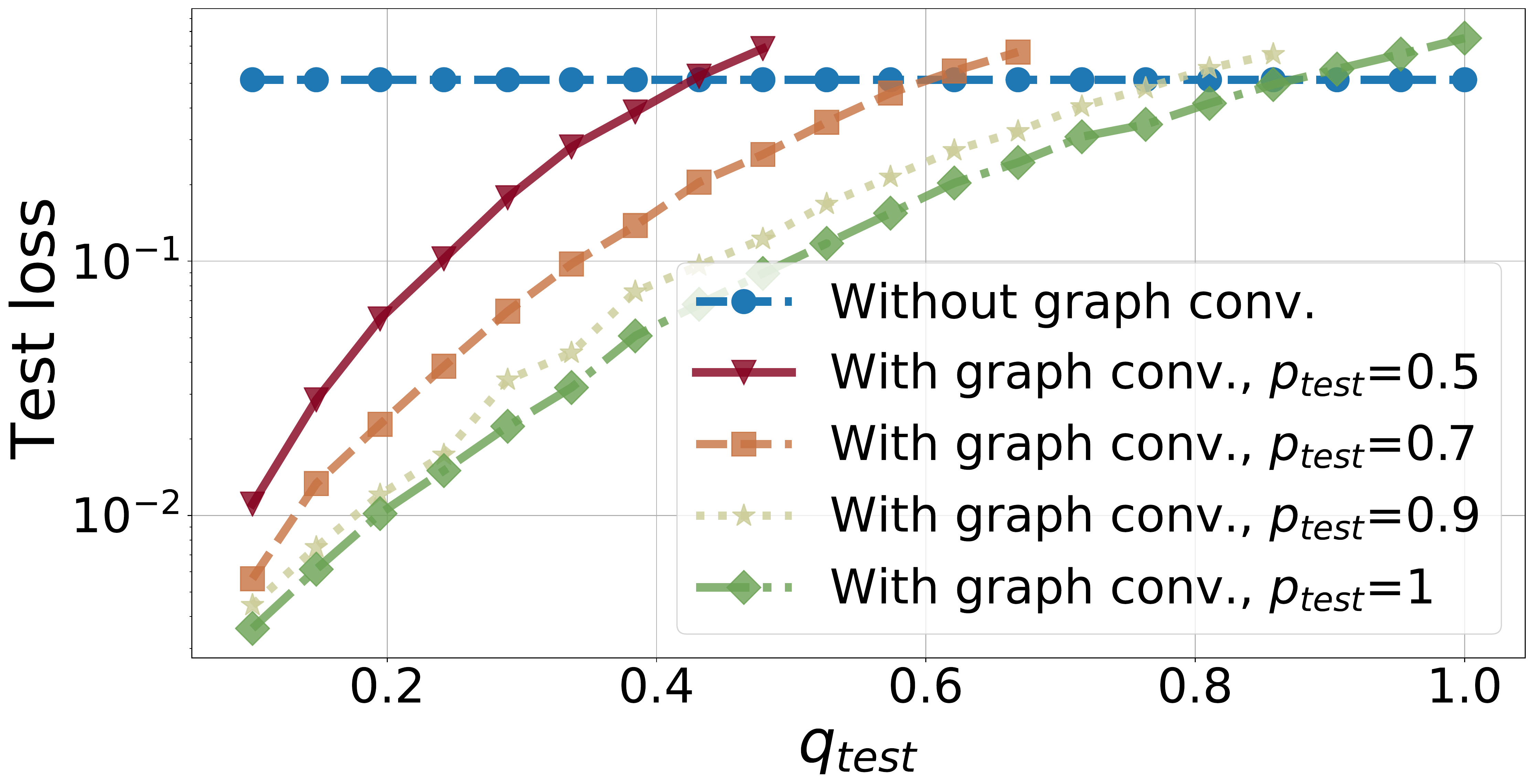}
		\caption{$p=0.7,q=0.38$}\label{fig:CSBM-1-3}	
	\end{subfigure}
	\begin{subfigure}[t]{1.5in}
		\includegraphics[width=\columnwidth]{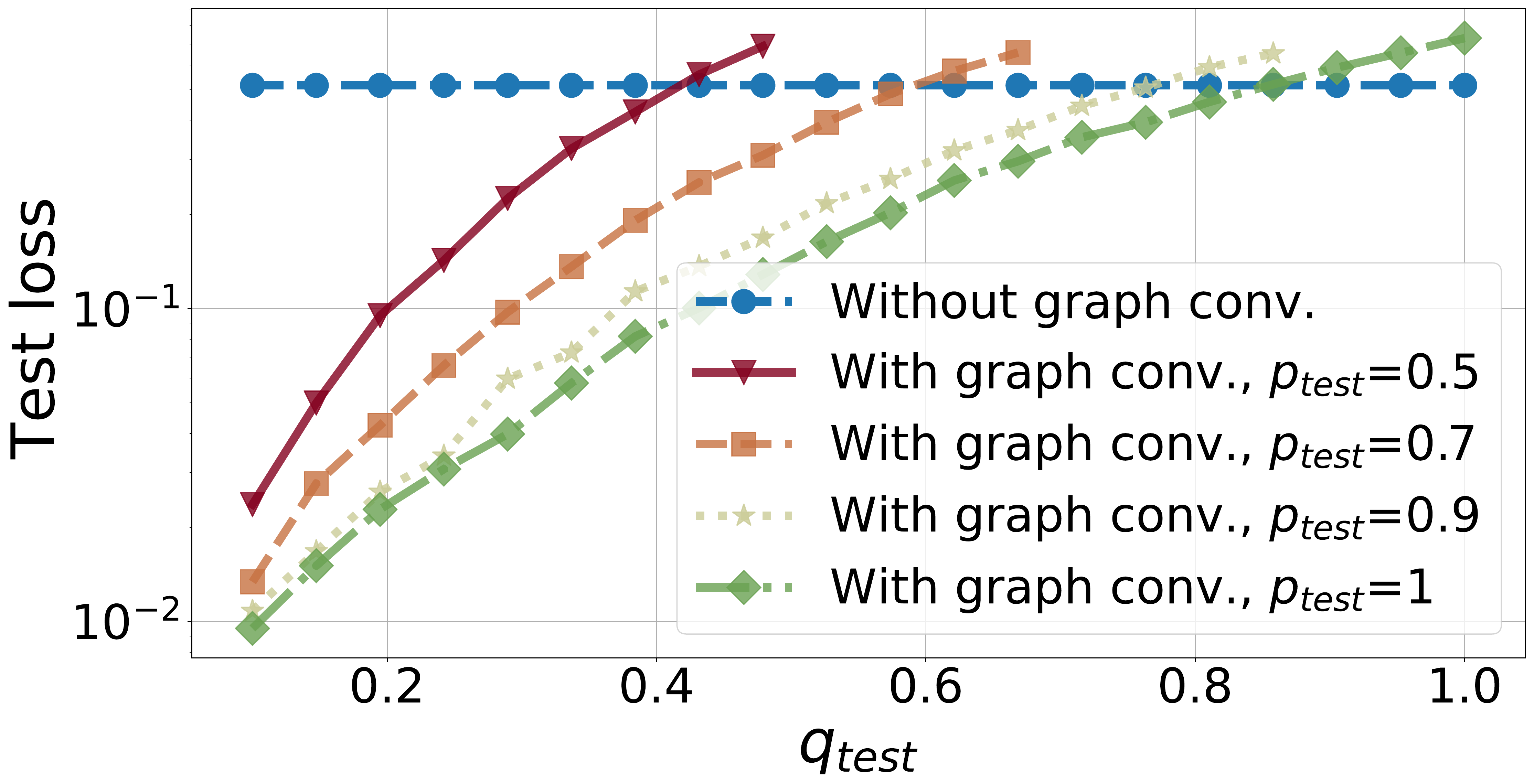}
		\caption{$p=0.7,q=0.67$}\label{fig:CSBM-1-4}
	\end{subfigure}
	\\\vspace{0.2cm}
	\begin{subfigure}[t]{1.5in}
		\centering
		\includegraphics[width=\columnwidth]{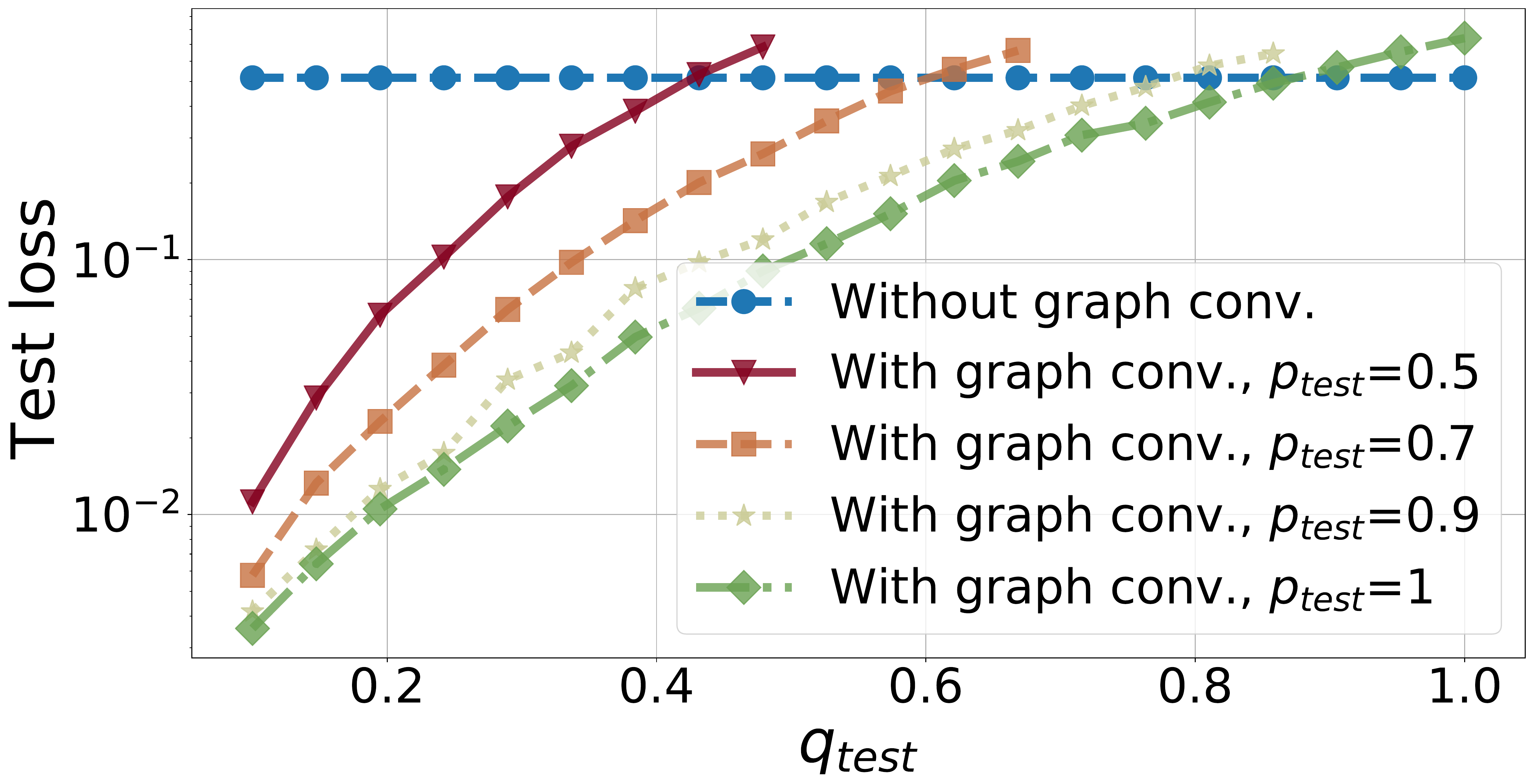}
		\caption{$p=0.9,q=0.1$}\label{fig:CSBM-1-5}	
	\end{subfigure}
	\begin{subfigure}[t]{1.5in}
		\includegraphics[width=\columnwidth]{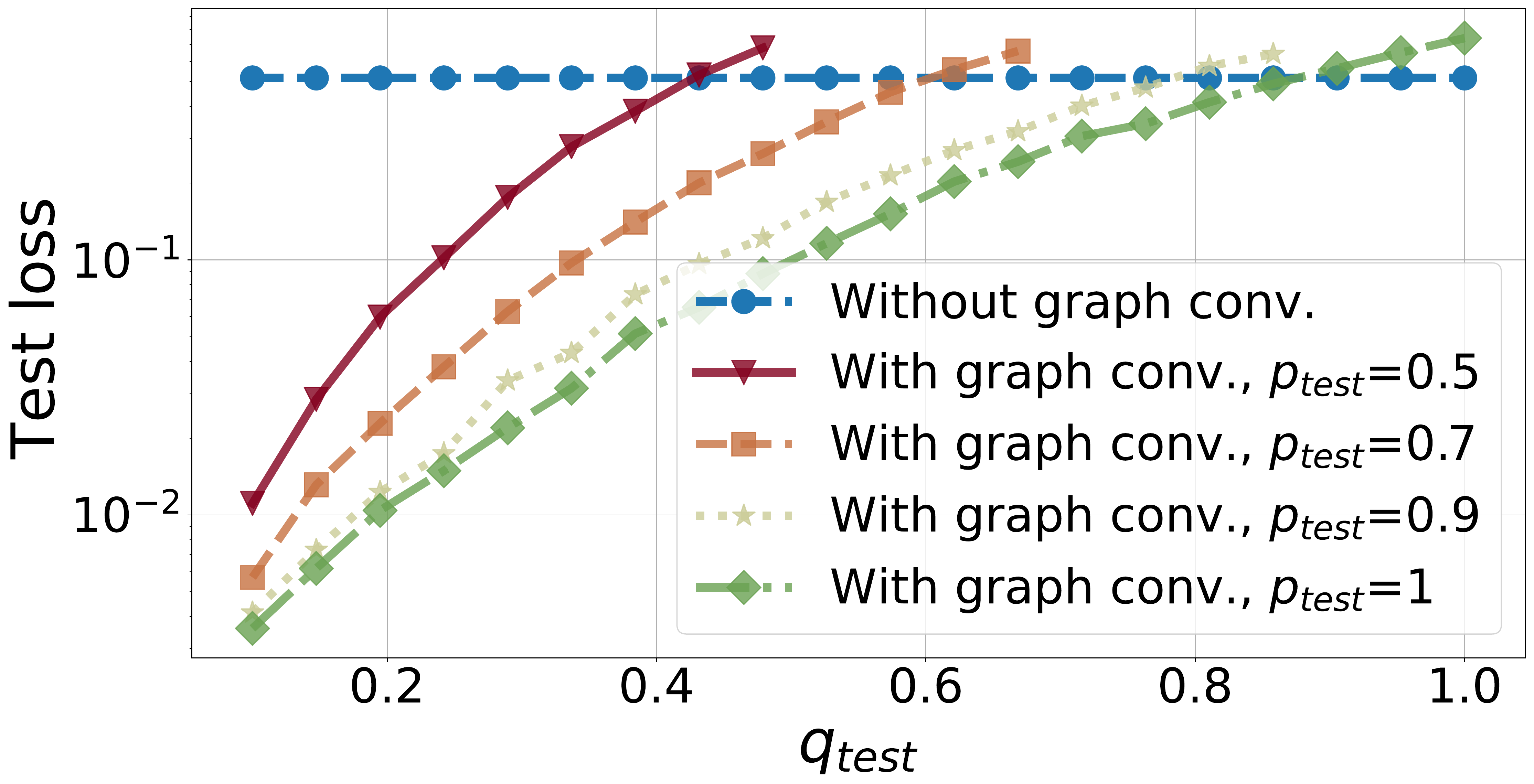}
		\caption{$p=0.9,q=0.38$}\label{fig:CSBM-1-6}
	\end{subfigure}
	\begin{subfigure}[t]{1.5in}
		\centering
		\includegraphics[width=\columnwidth]{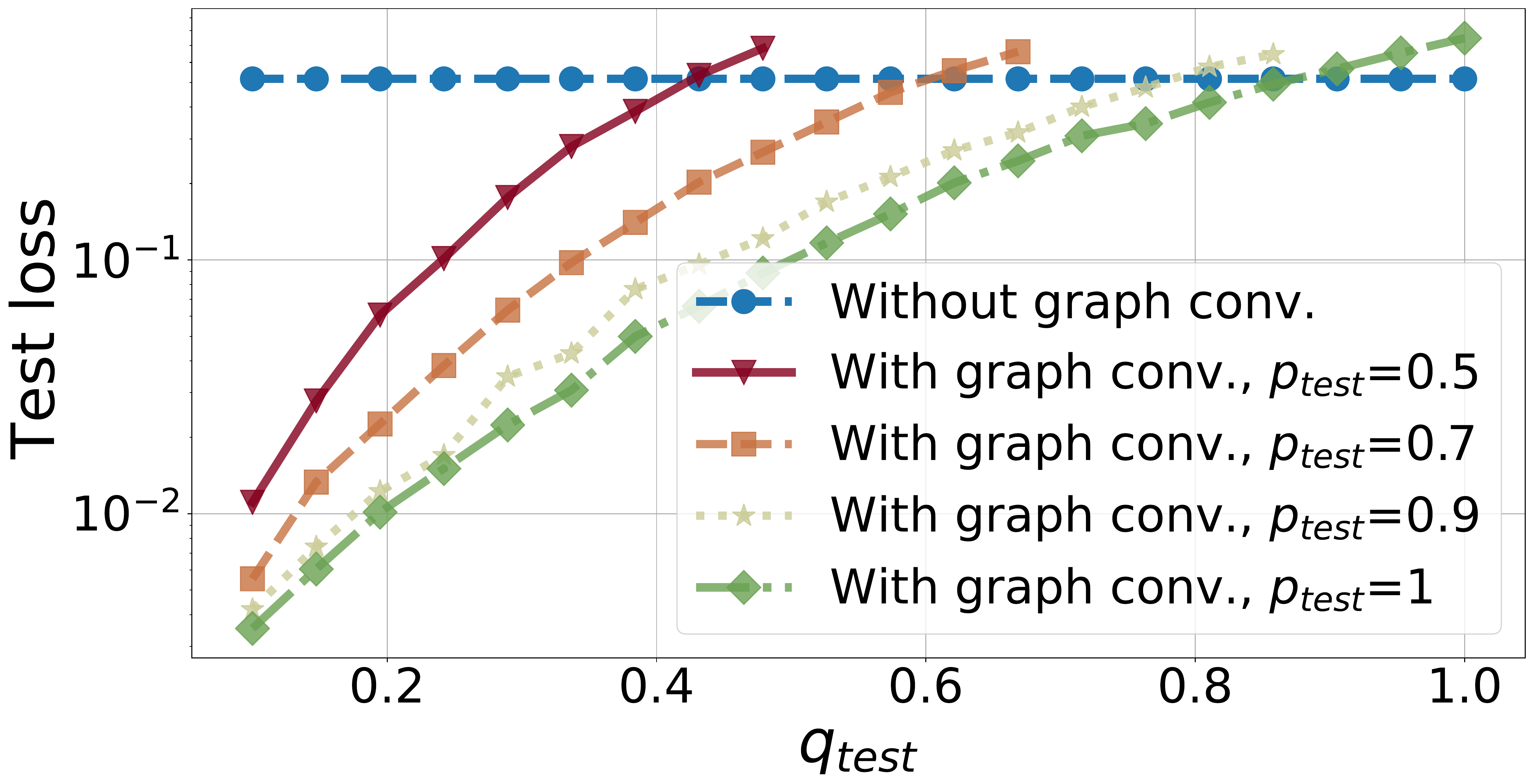}
		\caption{$p=0.9,q=0.67$}\label{fig:CSBM-1-7}	
	\end{subfigure}
	\begin{subfigure}[t]{1.5in}
		\includegraphics[width=\columnwidth]{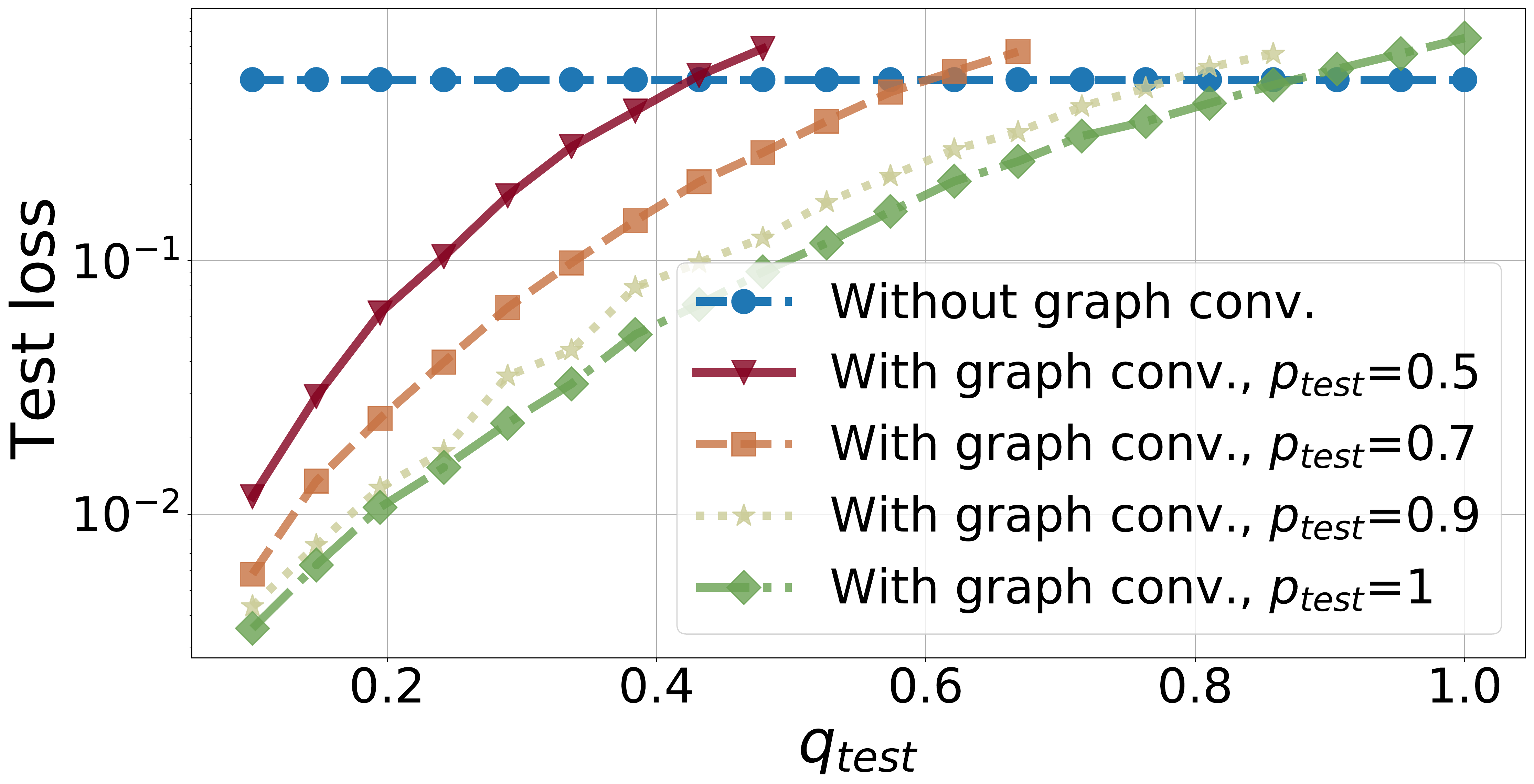}
		\caption{$p=1,q=0.1$}\label{fig:CSBM-1-8}
	\end{subfigure}
	\\\vspace{0.2cm}
	\begin{subfigure}[t]{1.5in}
		\centering
		\includegraphics[width=\columnwidth]{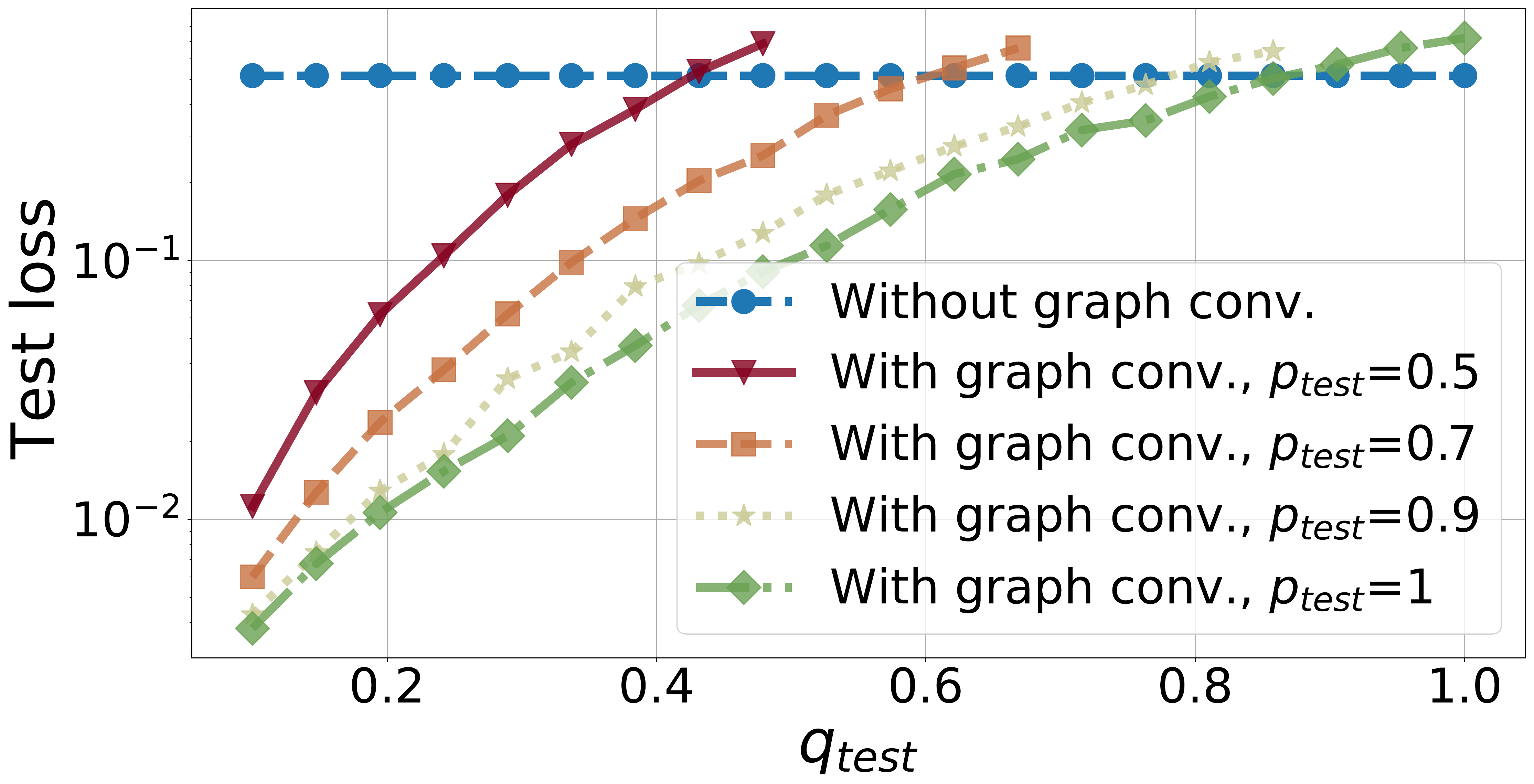}
		\caption{$p=1,q=0.38$}\label{fig:CSBM-1-9}	
	\end{subfigure}
	\begin{subfigure}[t]{1.5in}
		\includegraphics[width=\columnwidth]{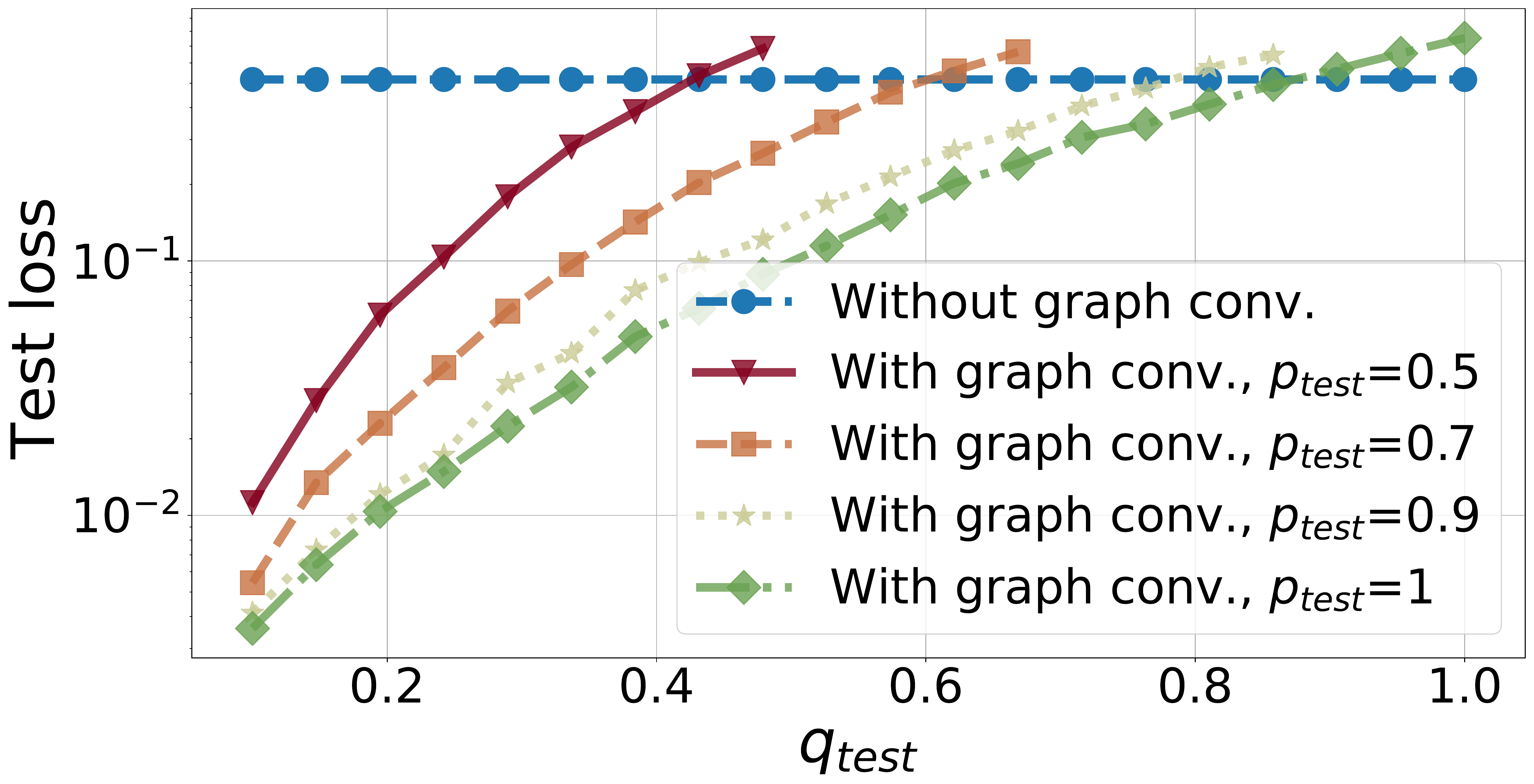}
		\caption{$p=1,q=0.67$}\label{fig:CSBM-1-10}
	\end{subfigure}
	\begin{subfigure}[t]{1.5in}
		\centering
		\includegraphics[width=\columnwidth]{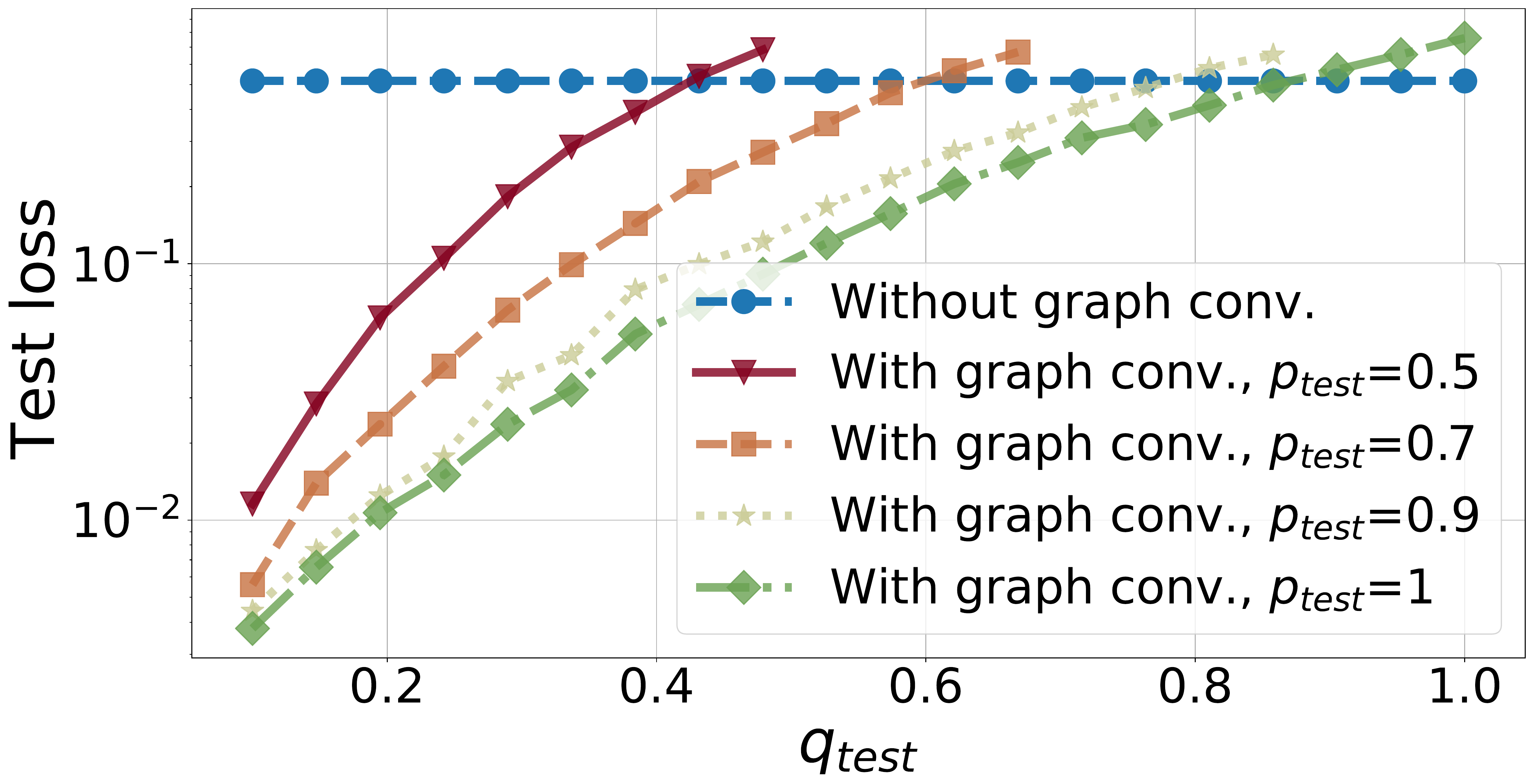}
		\caption{$p=1,q=0.95$}\label{fig:CSBM-1-11}	
	\end{subfigure}
	\caption{Out-of-distribution generalization for distance between the means equal to $2/\sqrt{d}$. The subcaption of each figure is the $p_{\train}$ and $q_{\train}$ pair. Note that we omit the sub-index $train$ from $p$ and $q$ in the subcaption due to space limitation. We test on CSBMs with $n=400$, $d=60$ and varying $p_{\test}$ and $q_{\test}$ while $p_{\test}>q_{\test}$ and fixed means. The $y$-axis is in log-scale.}\label{fig:CSBM-1}
\end{figure}

\begin{figure}[ht!]
	\centering
	\begin{subfigure}[t]{1.5in}
		\includegraphics[width=\columnwidth]{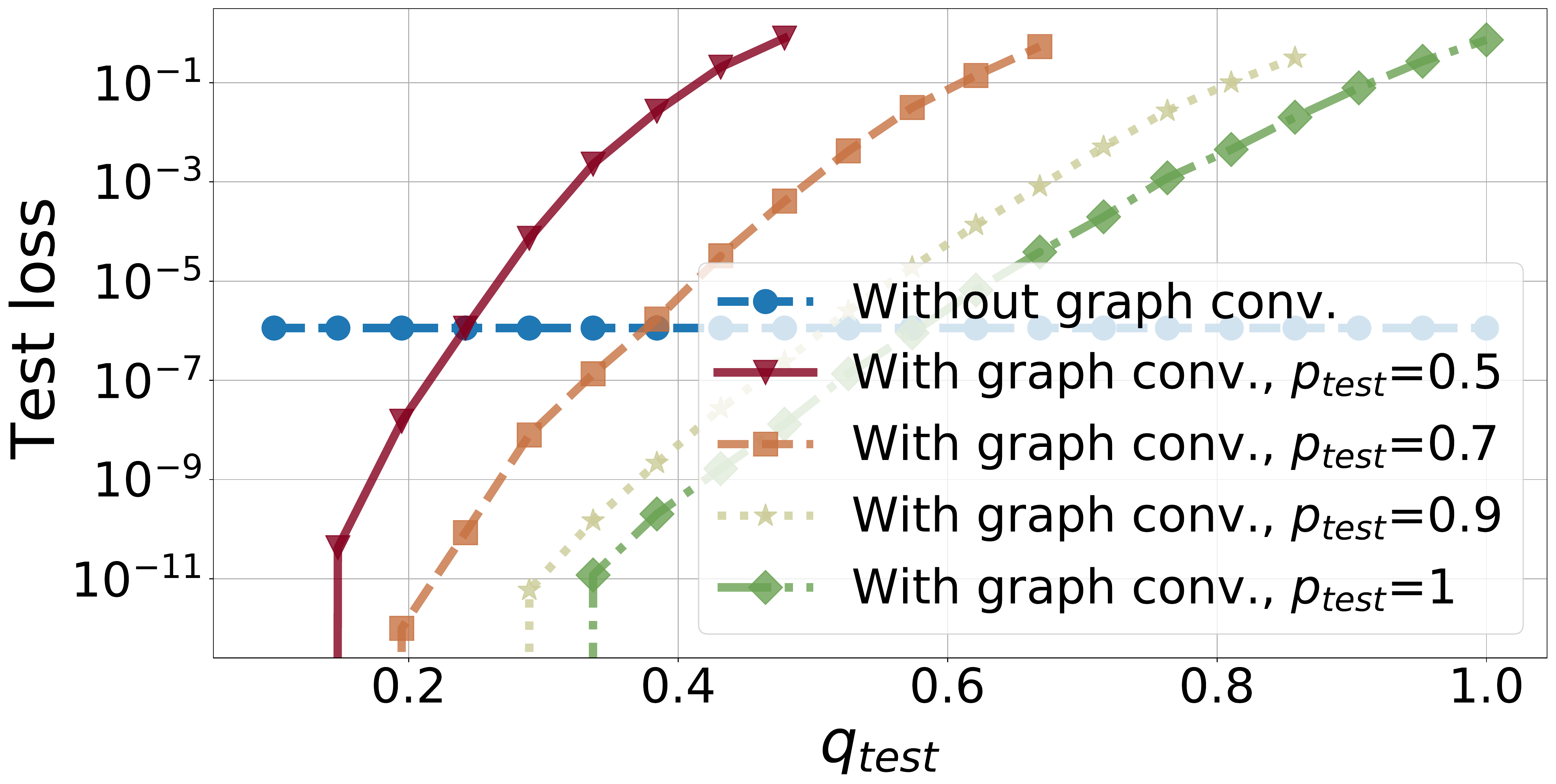}
		\caption{$p=0.5,q=0.38$}\label{fig:CSBM-2-1}
	\end{subfigure}
	\begin{subfigure}[t]{1.5in}
		\includegraphics[width=\columnwidth]{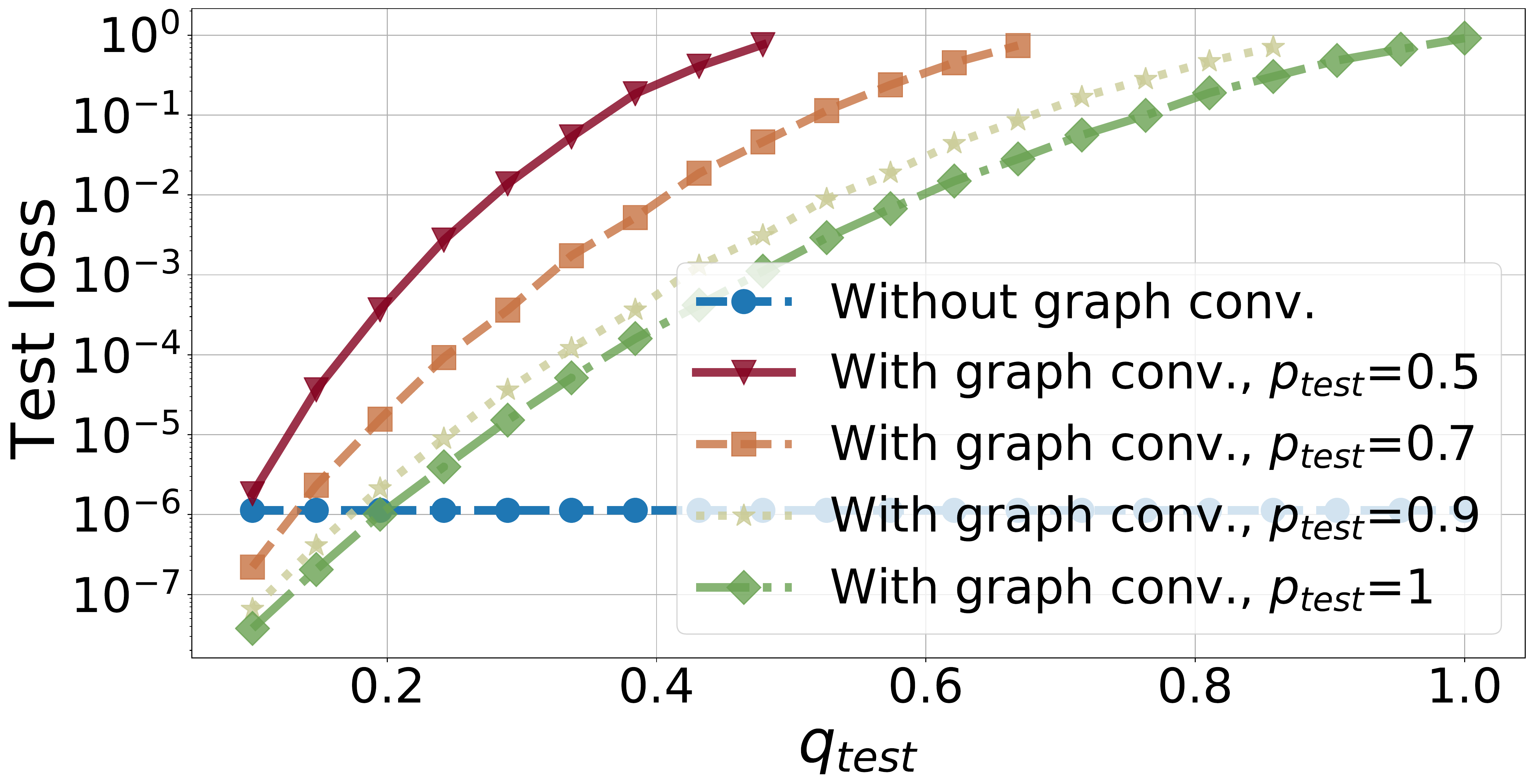}
		\caption{$p=0.7,q=0.1$}\label{fig:CSBM-2-2}
	\end{subfigure}
	\begin{subfigure}[t]{1.5in}
		\includegraphics[width=\columnwidth]{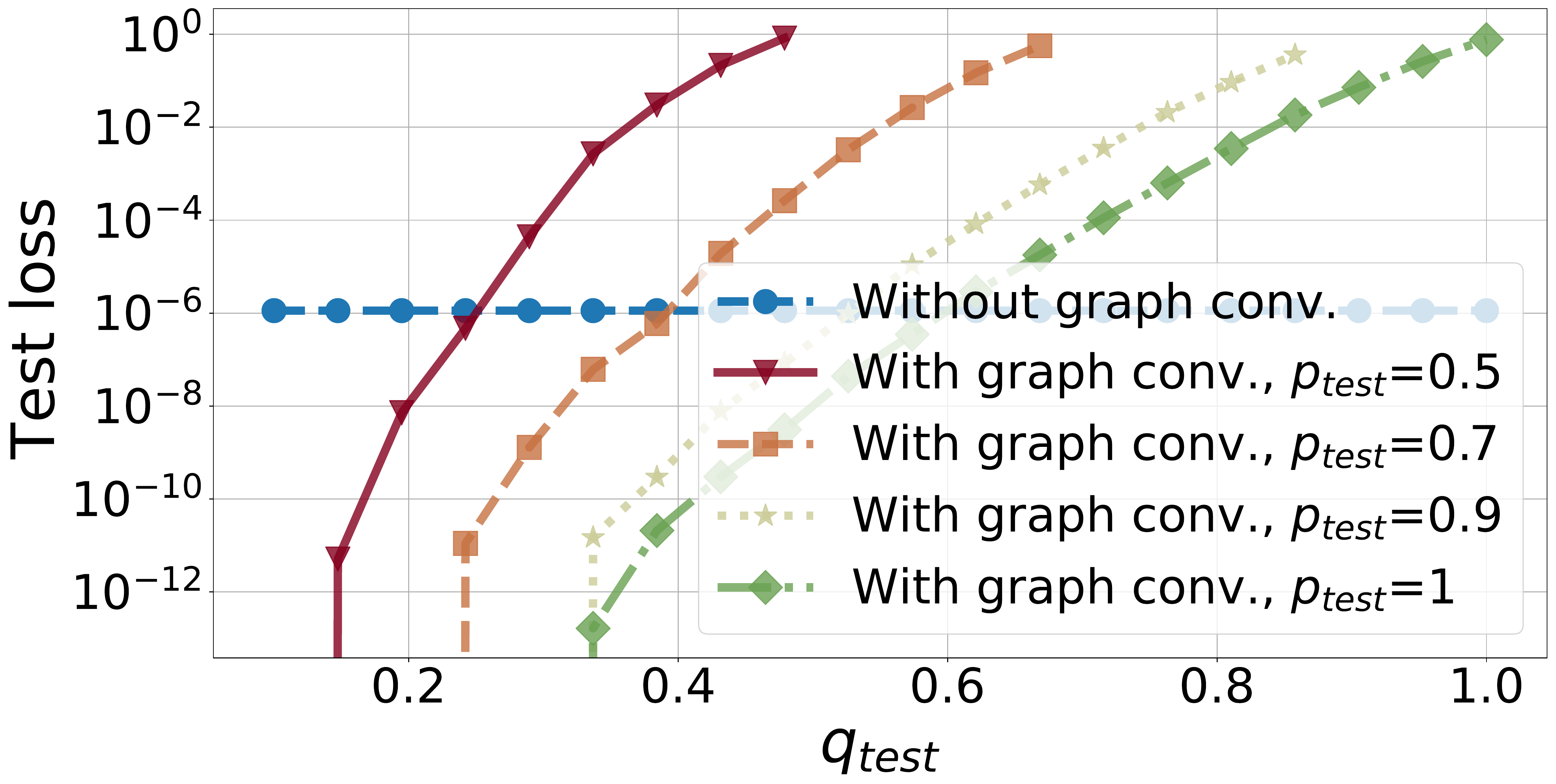}
		\caption{$p=0.7,q=0.38$}\label{fig:CSBM-2-3}
	\end{subfigure}
	\begin{subfigure}[t]{1.5in}
		\includegraphics[width=\columnwidth]{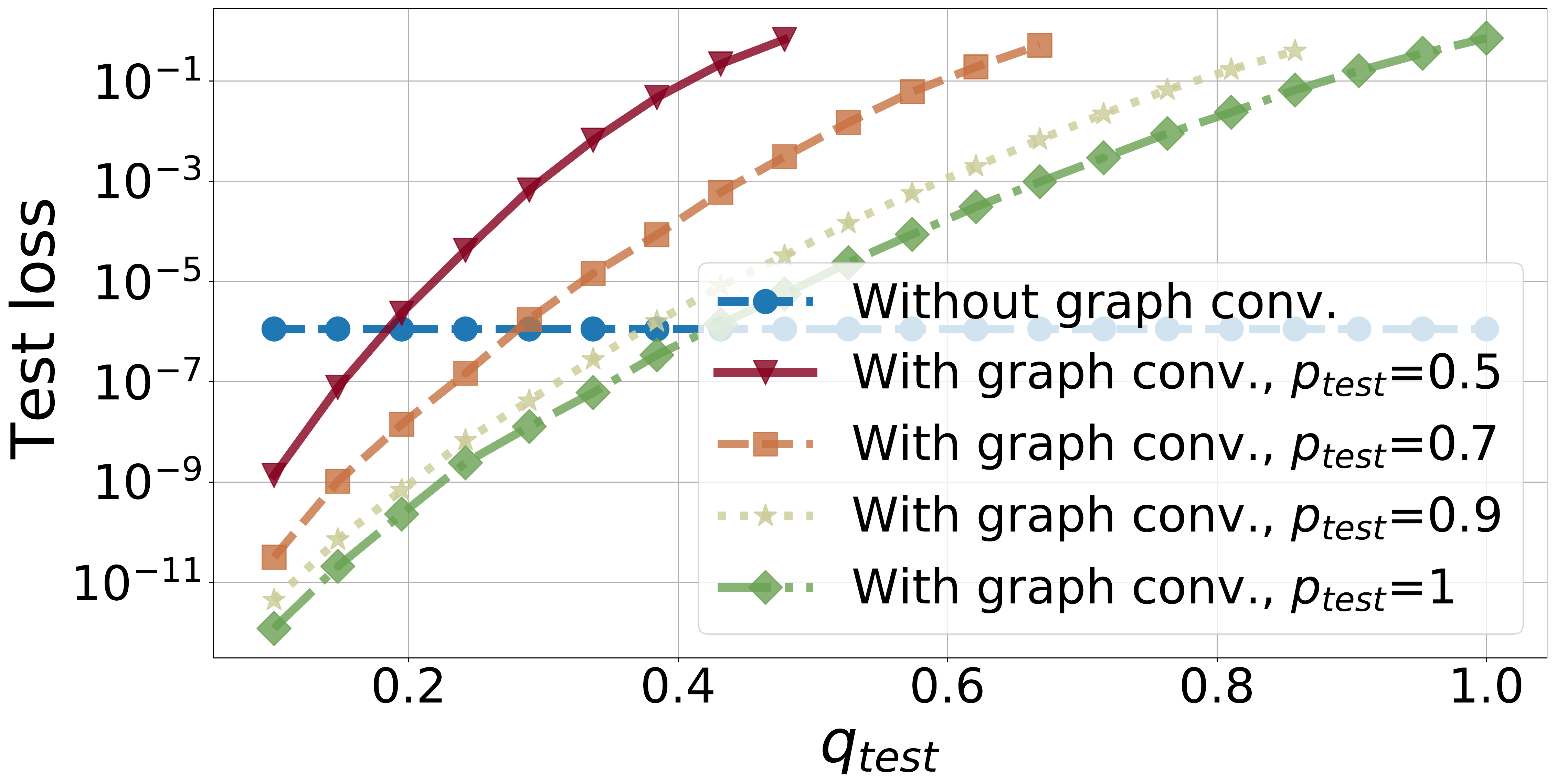}
		\caption{$p=0.7,q=0.67$}\label{fig:CSBM-2-4}
	\end{subfigure}
	\\\vspace{0.2cm}
	\begin{subfigure}[t]{1.5in}
		\includegraphics[width=\columnwidth]{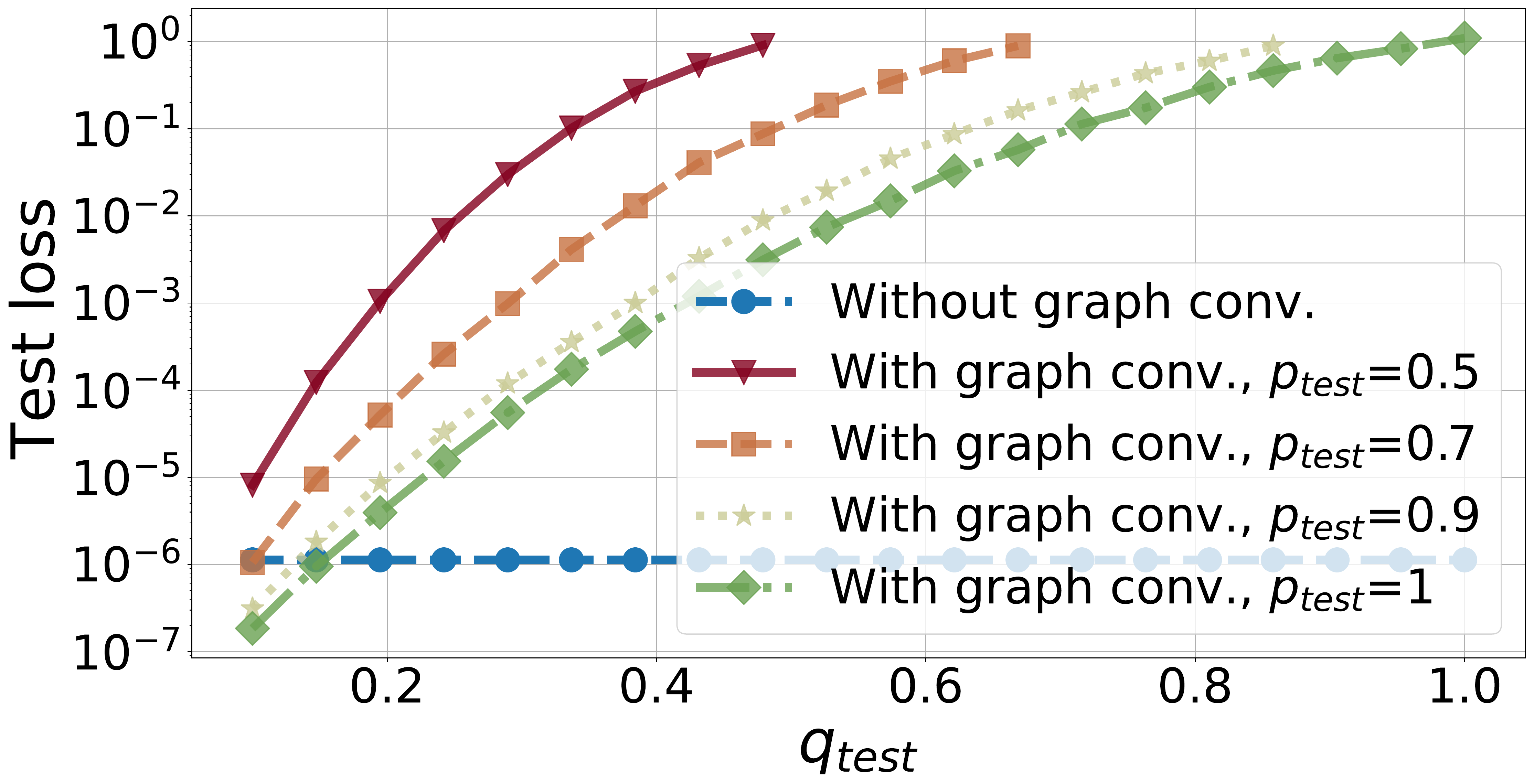}
		\caption{$p=0.9,q=0.1$}\label{fig:CSBM-2-5}
	\end{subfigure}
	\begin{subfigure}[t]{1.5in}
		\includegraphics[width=\columnwidth]{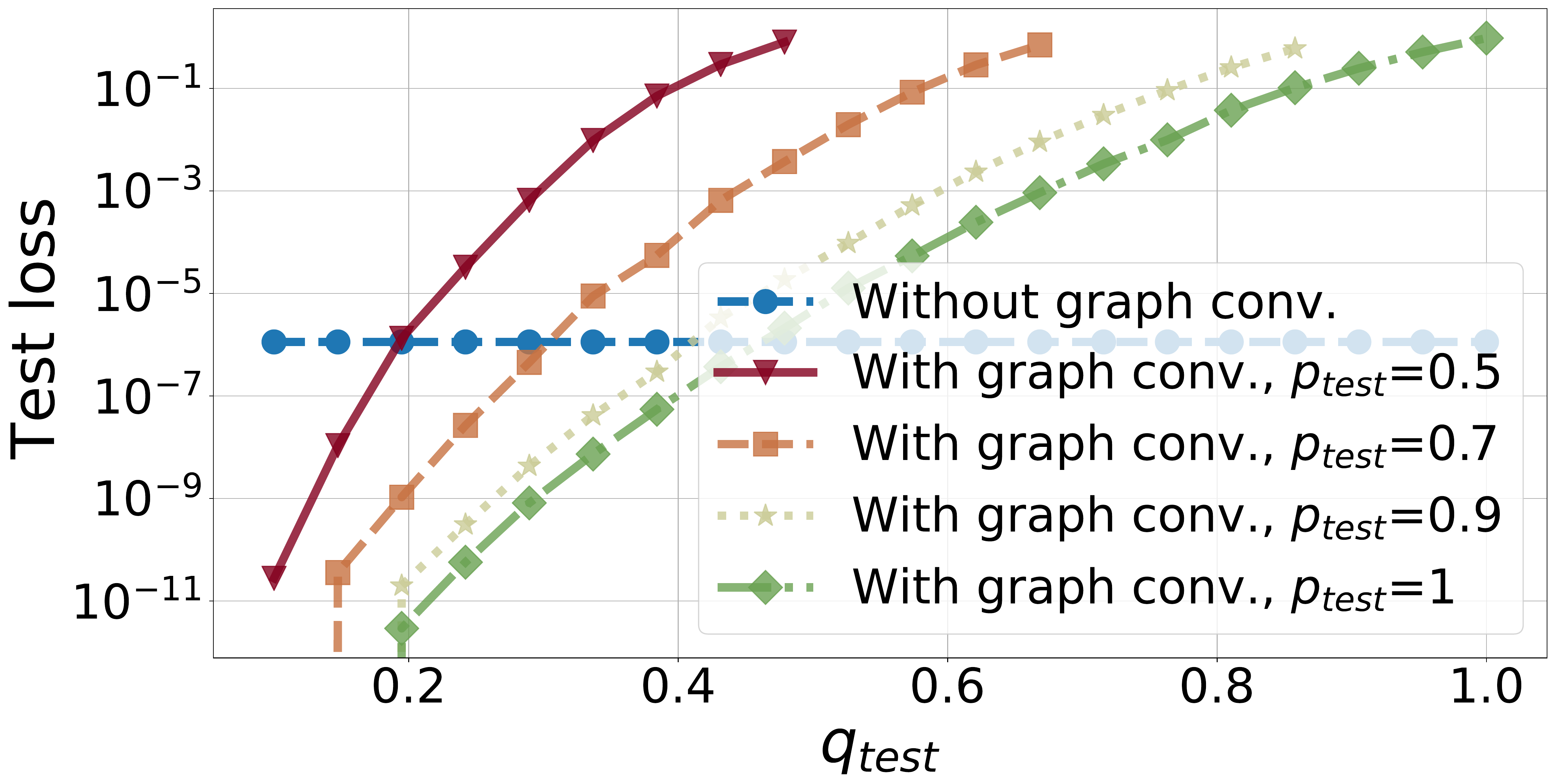}
		\caption{$p=0.9,q=0.38$}\label{fig:CSBM-2-6}
	\end{subfigure}
	\begin{subfigure}[t]{1.5in}
		\includegraphics[width=\columnwidth]{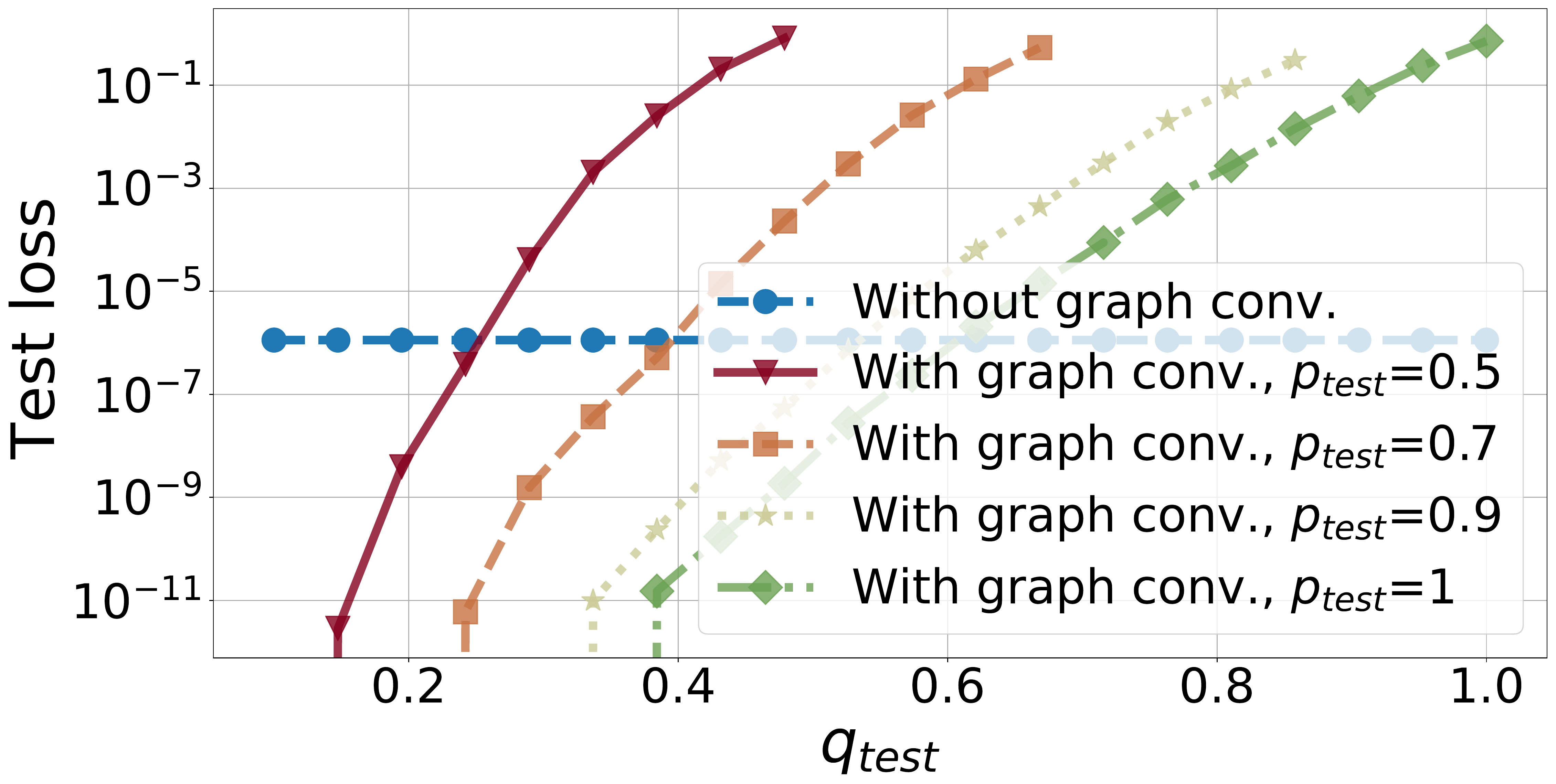}
		\caption{$p=0.9,q=0.67$}\label{fig:CSBM-2-7}
	\end{subfigure}
	\begin{subfigure}[t]{1.5in}
		\includegraphics[width=\columnwidth]{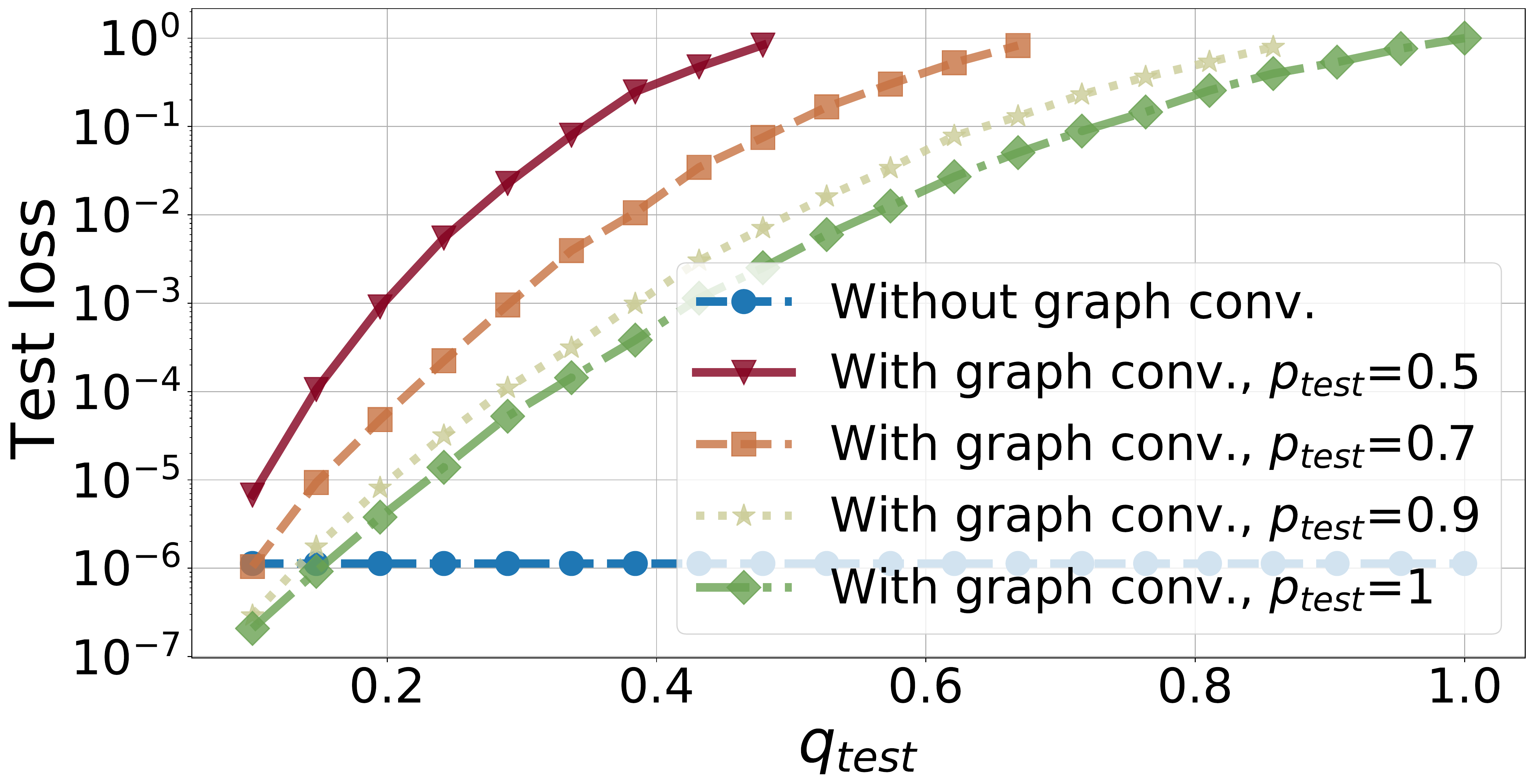}
		\caption{$p=1,q=0.1$}\label{fig:CSBM-2-8}
	\end{subfigure}
	\\\vspace{0.2cm}
	\begin{subfigure}[t]{1.5in}
		\includegraphics[width=\columnwidth]{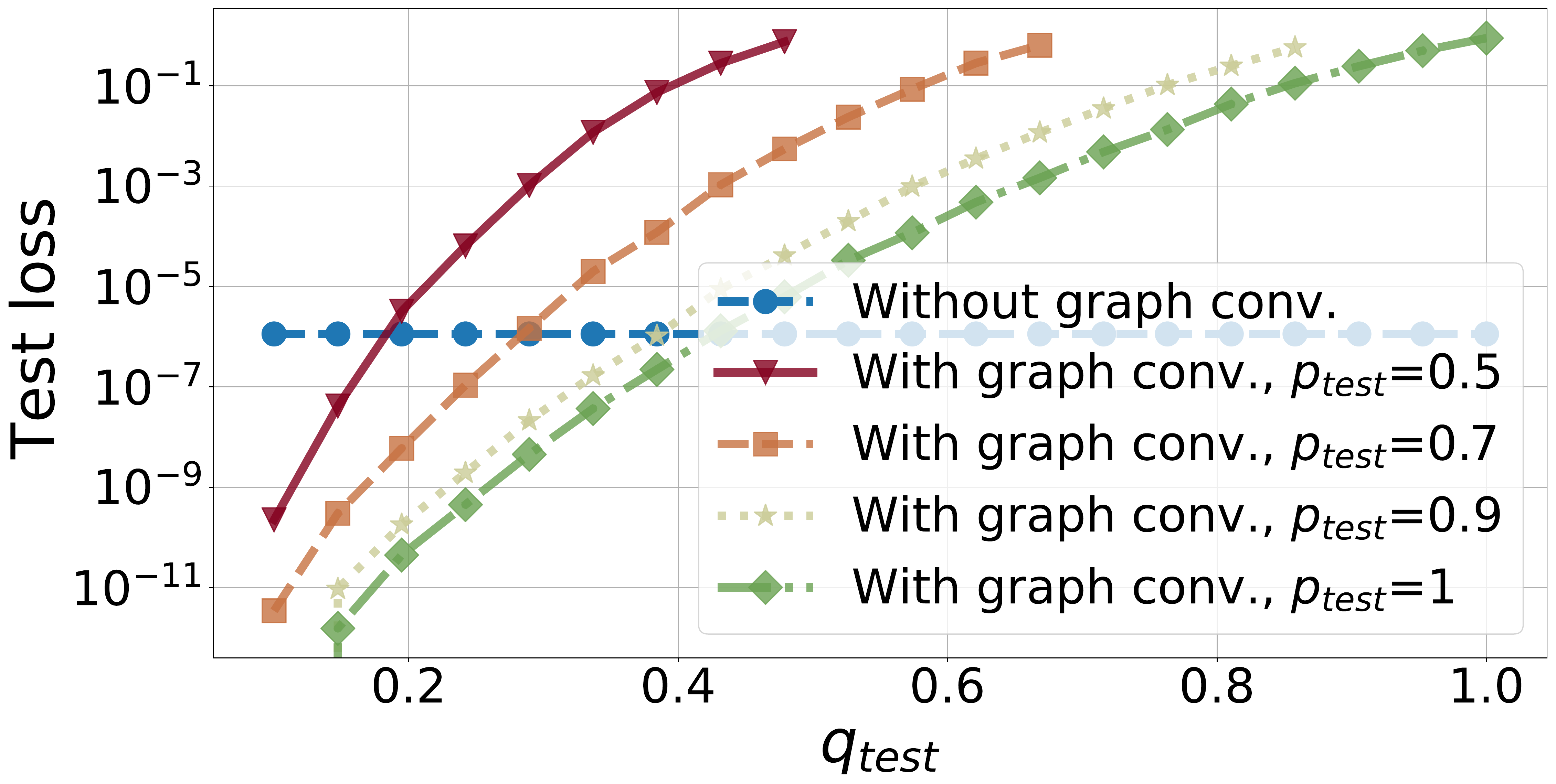}
		\caption{$p=1,q=0.38$}\label{fig:CSBM-2-9}
	\end{subfigure}
	\begin{subfigure}[t]{1.5in}
		\includegraphics[width=\columnwidth]{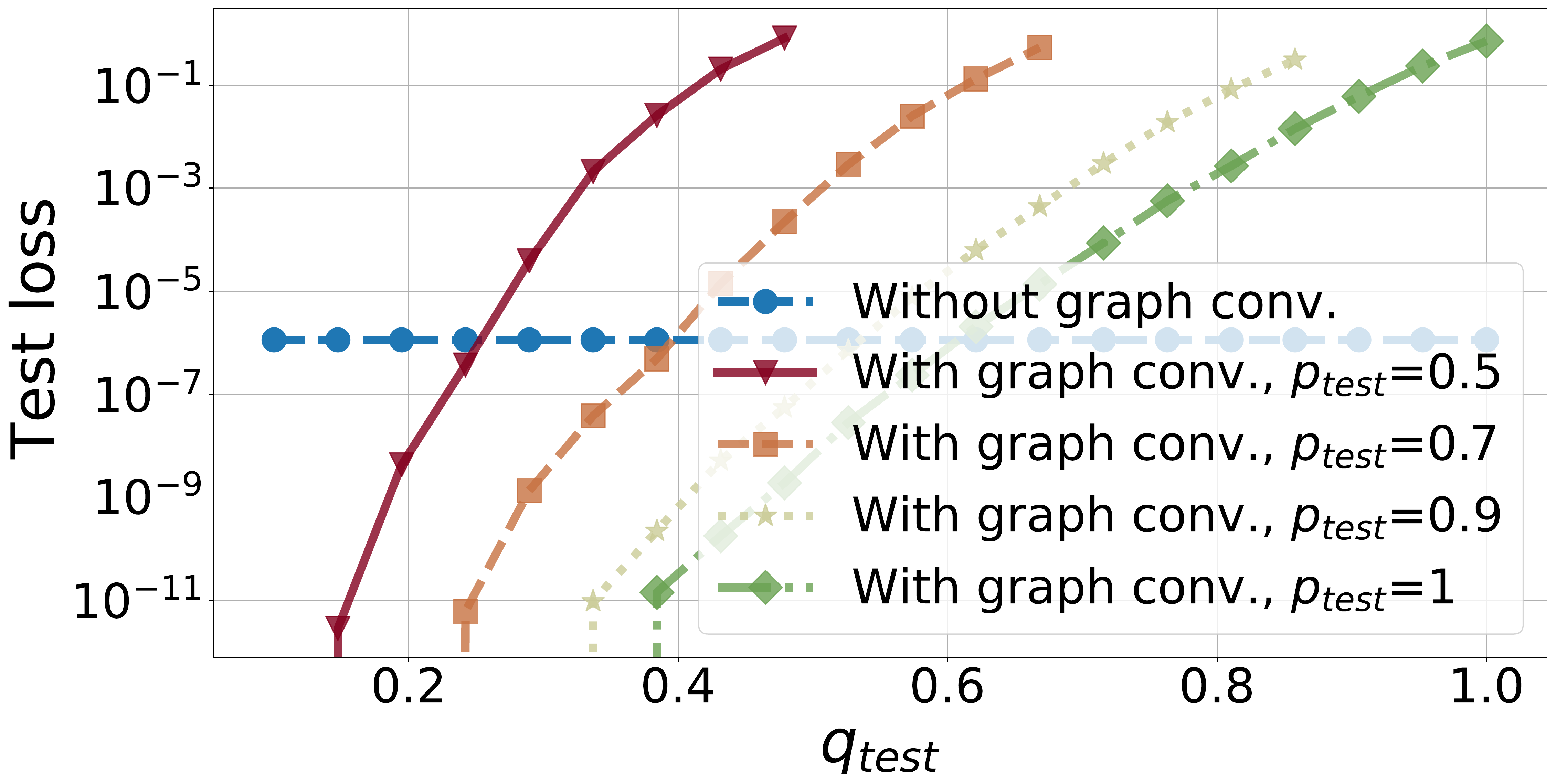}
		\caption{$p=1,q=0.67$}\label{fig:CSBM-2-10}
	\end{subfigure}
	\begin{subfigure}[t]{1.5in}
		\includegraphics[width=\columnwidth]{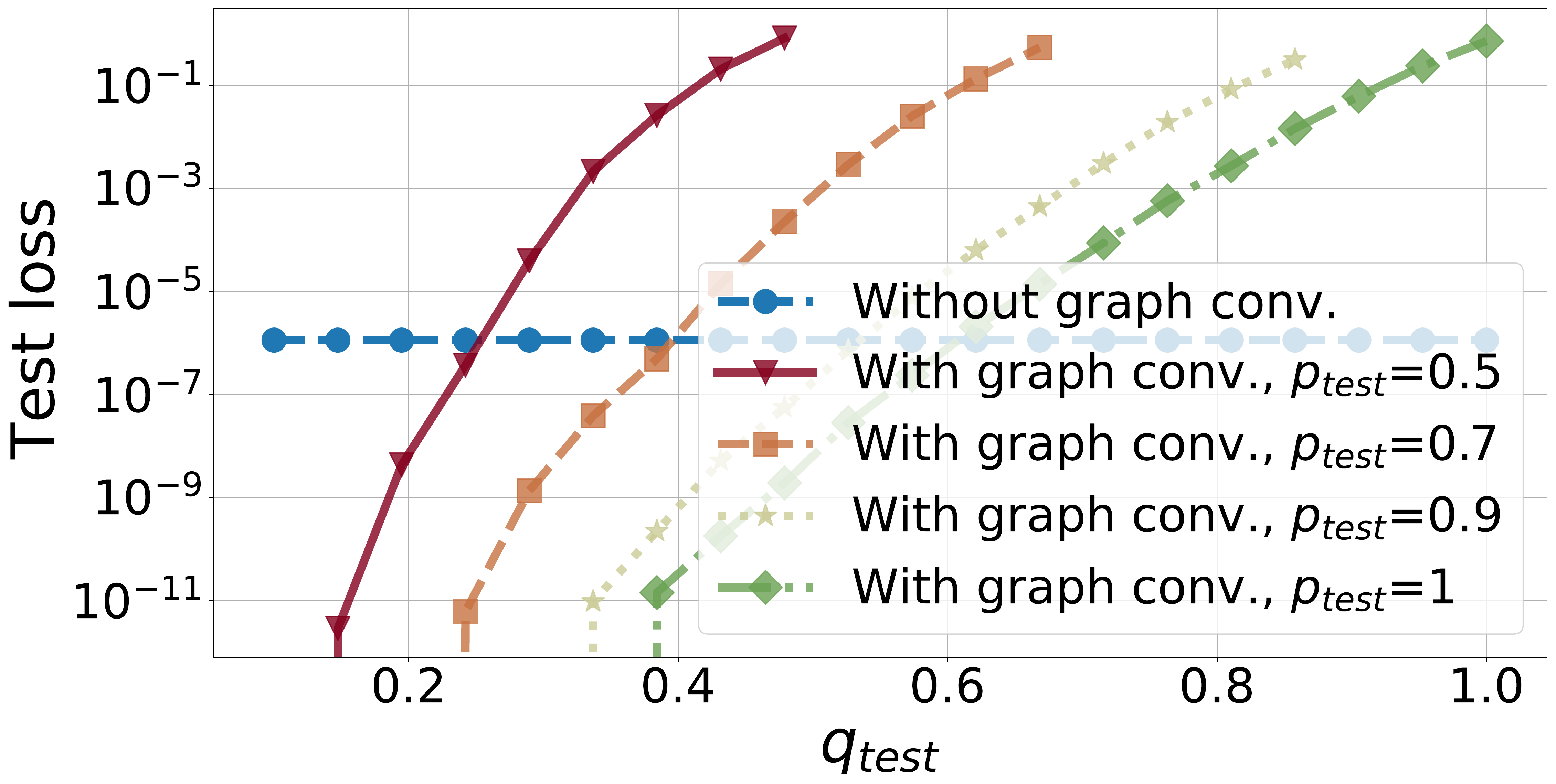}
		\caption{$p=1,q=0.95$}\label{fig:CSBM-2-11}
	\end{subfigure}
	\caption{Out-of-distribution generalization for distance between the means equal to $16/\sqrt{d}$. The subcaption of each figure is the $p_{\train}$ and $q_{\train}$ pair. Note that we omit the sub-index $train$ from $p$ and $q$ in the subcaption due to space limitation. We test on CSBMs with $n=400$, $d=60$ and varying $p_{\test}$ and $q_{\test}$ while $p_{\test}>q_{\test}$ and fixed means. The $y$-axis is in log-scale.}\label{fig:CSBM-2}
\end{figure}

\subsection{Out-of-distribution generalization on real data}

In this experiment we illustrate the generalization performance on real data for the linear classifier obtained by minimizing cross-entropy; see details about the optimization problem in the main paper. In particular, we use the partially labelled real data to train two linear classifiers, with and without graph convolution. We generate new graphs by adding inter-class edges uniformly at random. Then we test the performance of the trained classifiers on the noisy graphs with the original attributes. Therefore, the only thing that changes in the new unseen data are the graphs, the attributes remain the same.

We use the popular real data Cora, PubMed and WikipediaNetwork. These data are publicly available and can be downloaded from~\cite{FL2019}. The datasets come with multiple classes, however, for each of our experiments we do a one-v.s.-all  classification for a single class. WikipediaNetwork comes with multiple masks for the labels, in our experiments we use the first mask. Moreover, this is a semi-supervised problem, meaning that only a fraction of the training nodes have labels. Details about the classes of the datasets that were omitted from the main paper are given in \Cref{table:2}. The results of the experiments are shown in \Cref{fig:Cora,fig:PubMed,fig:Wiki}.
\begin{table}[ht!]
\caption{Information about the classes of the datasets. Here, the letter of the class refers to the original class of the dataset. Then number of nodes and attributes for Cora dataset is $2708$ and $1433$, respectively. The number of nodes and attributes for PubMed dataset is $19717$ and $500$, respectively. The number of nodes and attributes for Wiki.Net dataset is $2277$ and $2325$, respectively.} \vspace{0.1cm}
\centering 
 \begin{tabular}{||c c c c c||} 
 \hline
 Dataset &  Class & $\beta_0$ & $\beta_1$ & $\norm{\muv-\nuv}$ \\ [0.5ex] 
 \hline\hline
  \multirow{5}{*}{Cora}   & C & $5.2$e-$02$ & $4.8$e-$02$ & $9.2$e-$01$\\
                          & D & $6.3$e-$02$ & $2.4$e-$02$ & $6.8$e-$01$\\
                          & E & $5.3$e-$02$ & $4.7$e-$02$ & $7.7$e-$01$\\
                          & F & $5.0$e-$02$ & $6.7$e-$02$ & $8.5$e-$01$\\
                          & G & $4.7$e-$02$ & $1.1$e-$01$ & $8.6$e-$01$\\
 \hline
  \multirow{1}{*}{PubMed} & C& $3.4$e-$03$ & $2.5$e-$03$ & $7.0$e-$02$\\
 \hline
  \multirow{3}{*}{Wiki.Net} & C & $4.8$e-$01$ & $4.8$e-$01$ & $2.1$e-$01$\\
                            & D & $4.8$e-$01$ & $4.6$e-$01$ & $4.3$e-$01$\\
                            & E & $4.8$e-$01$ & $4.9$e-$01$ & $5.3$e-$01$\\ [1ex] 
 \hline
\end{tabular}
\label{table:2}
\end{table}

\begin{figure}[ht!]
	\centering
	\begin{subfigure}[t]{2.1in}
		\centering
		\includegraphics[width=\columnwidth]{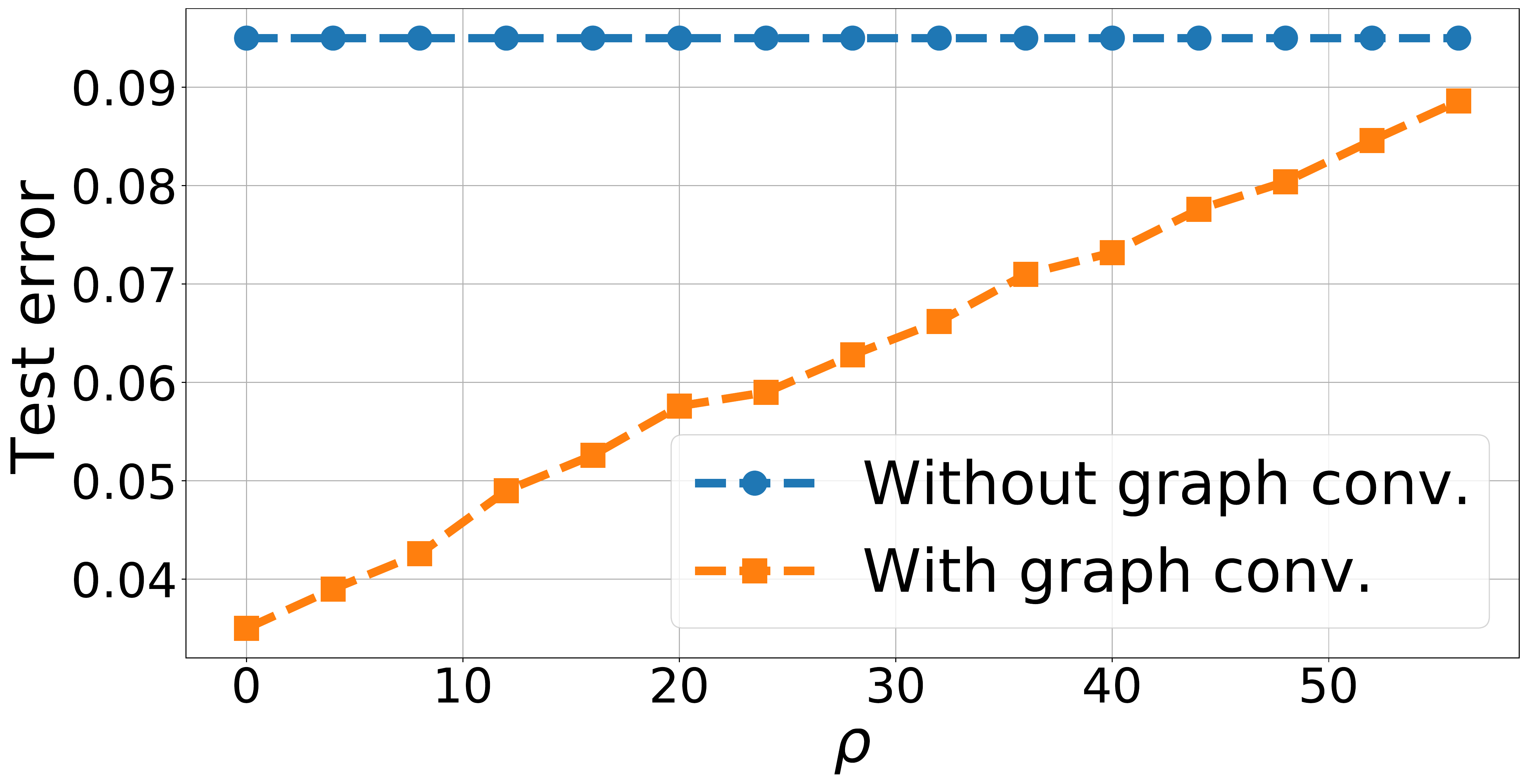}
		\caption{Cora, class $C$}\label{fig:Cora-c}
	\end{subfigure}
	\begin{subfigure}[t]{2.1in}
		\centering
		\includegraphics[width=\columnwidth]{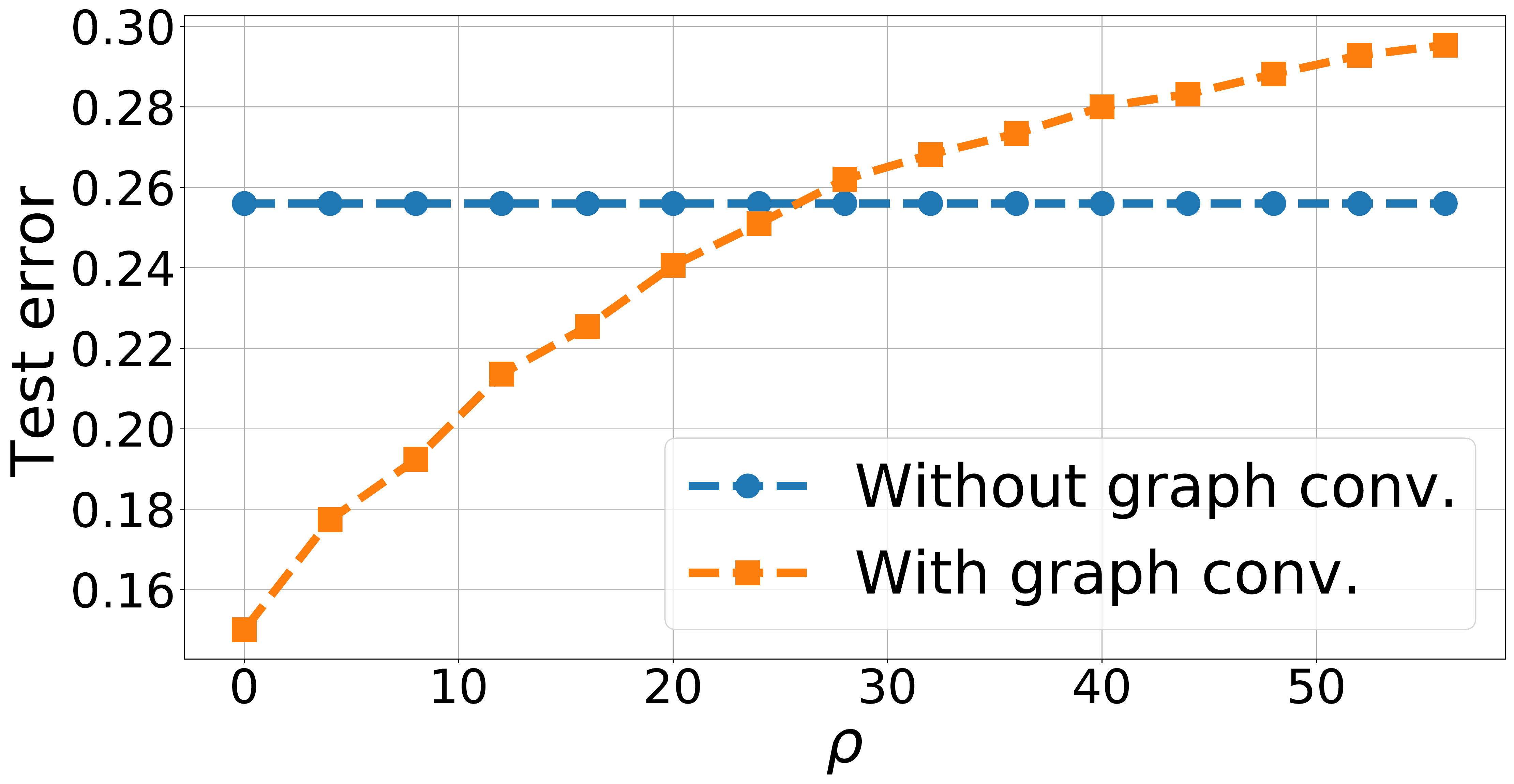}
		\caption{Cora, class $D$}\label{fig:Cora-d}
	\end{subfigure}
	\begin{subfigure}[t]{2.1in}
		\centering
		\includegraphics[width=\columnwidth]{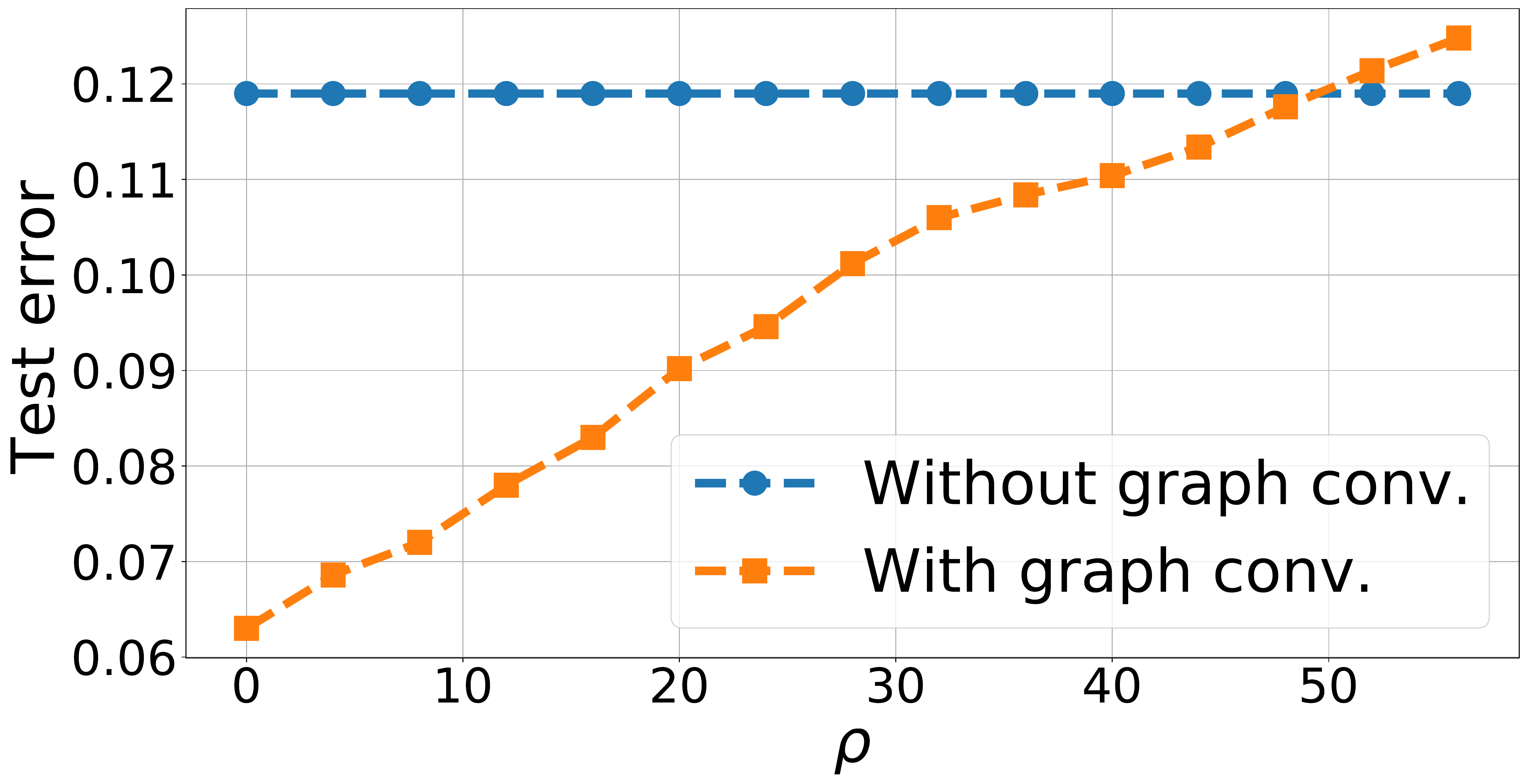}
		\caption{Cora, class $E$}\label{fig:Cora-e}
	\end{subfigure}
	\\\vspace{0.2cm}
	\begin{subfigure}[t]{2.1in}
		\centering
		\includegraphics[width=\columnwidth]{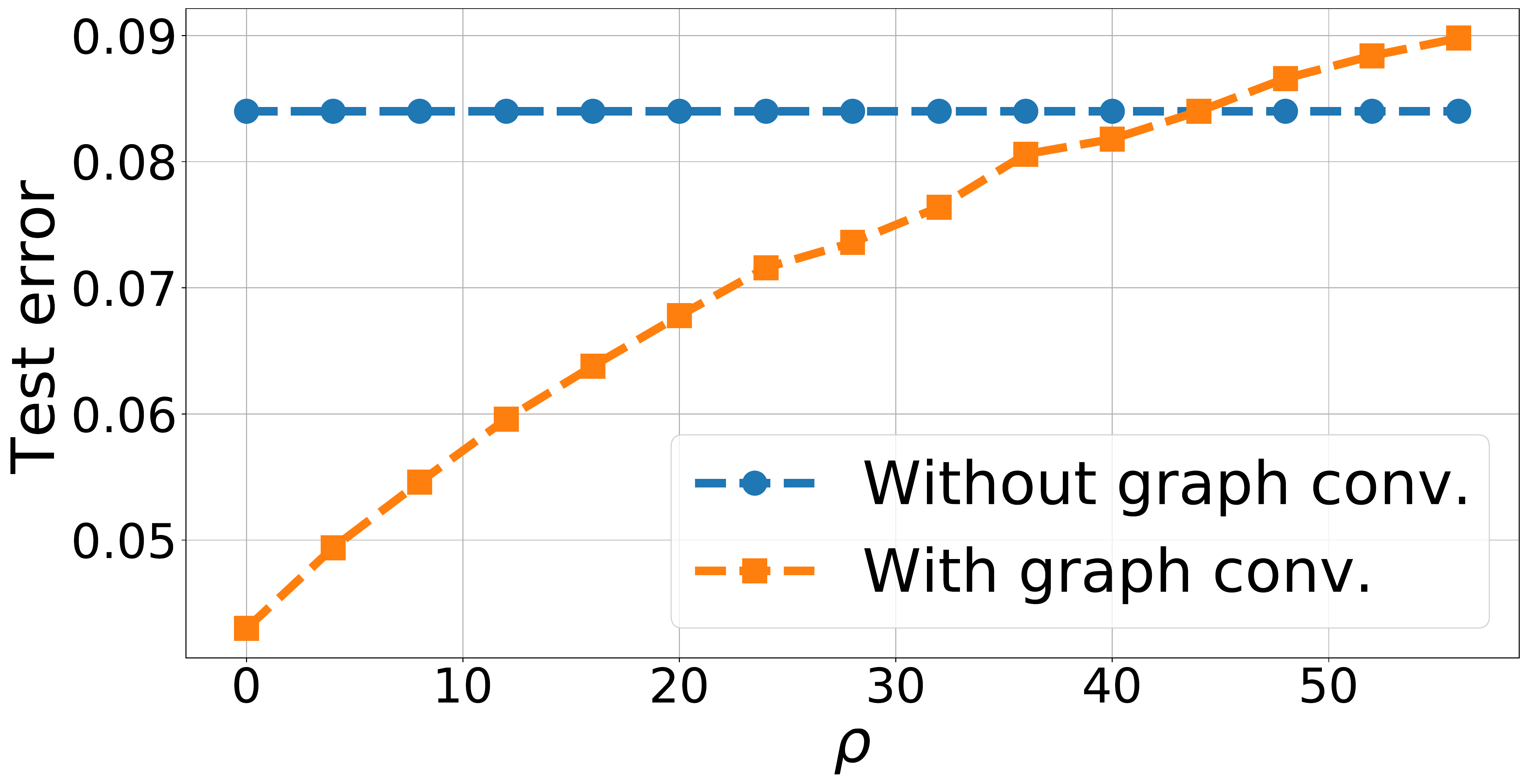}
		\caption{Cora, class $F$}\label{fig:Cora-f}
	\end{subfigure}
	\begin{subfigure}[t]{2.1in}
		\centering
		\includegraphics[width=\columnwidth]{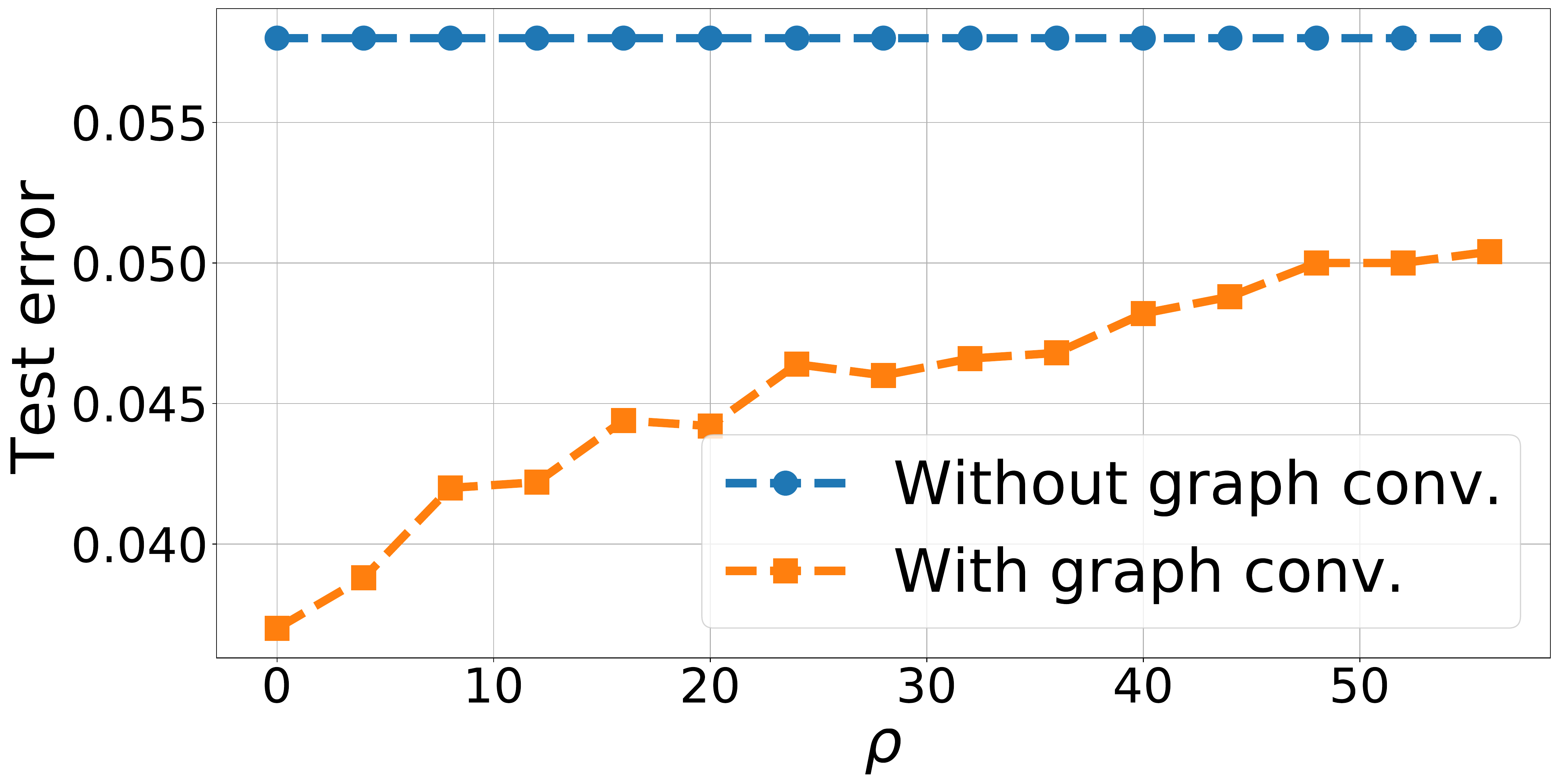}
		\caption{Cora, class $G$}\label{fig:Cora-g}
	\end{subfigure}
	\caption{Test loss as the number of nodes increases for Cora. The test error measures the number of misclassified nodes over the number of nodes in the graph. Here, $\rho$ denotes the ratio of added inter-class edges over the number of inter-class edges of the original graph. The $y$-axis is in log-scale.}\label{fig:Cora}
\end{figure}

\begin{figure}[ht!]
	\centering
	\begin{subfigure}[t]{2.1in}
		\centering
		\includegraphics[width=\columnwidth]{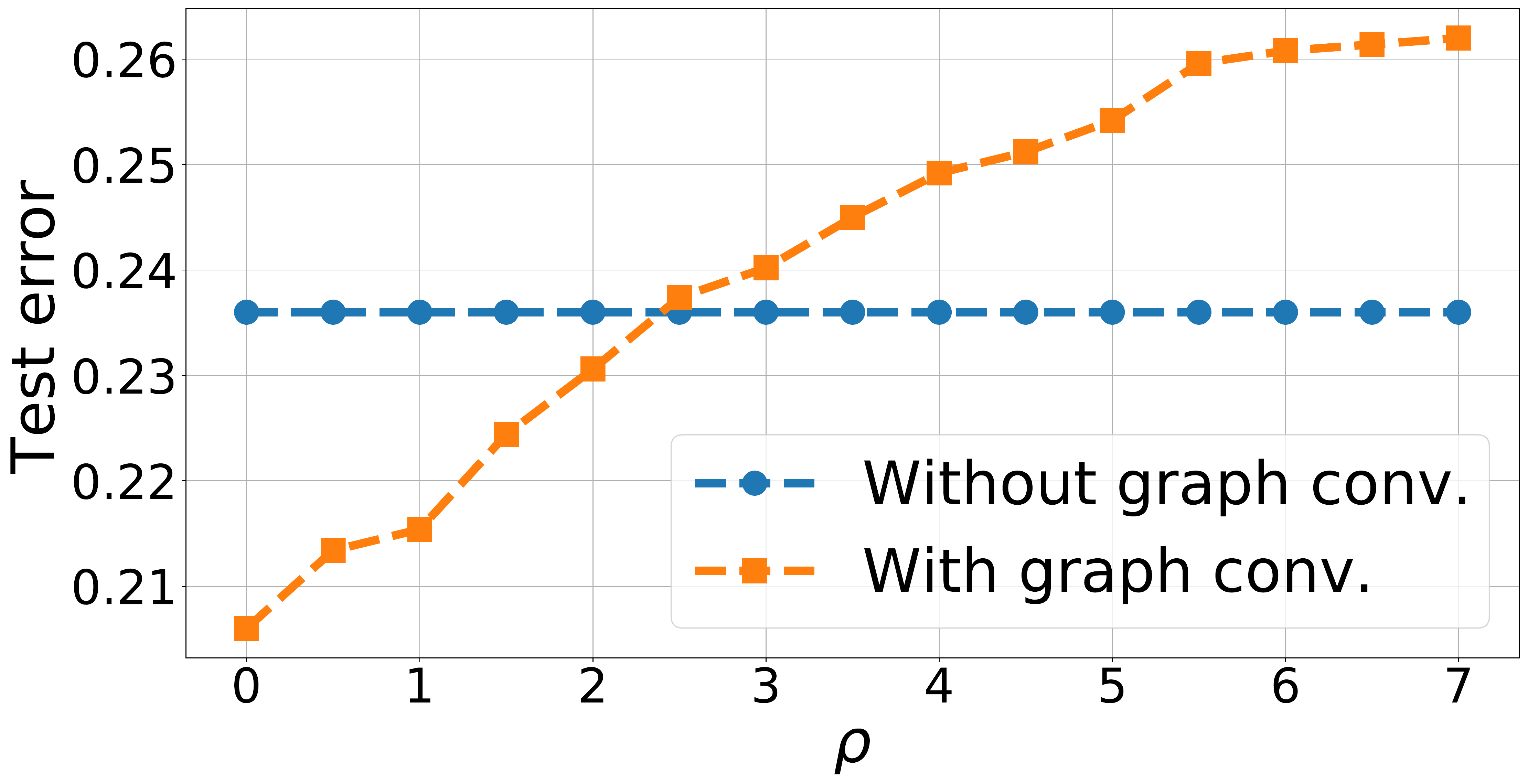}
		\caption{PubMed, class $C$}\label{fig:PubMed-c}
	\end{subfigure}
	\caption{Test loss as the number of nodes increases for PubMed.}\label{fig:PubMed}
\end{figure}

\begin{figure}[ht!]
	\centering
	\begin{subfigure}[t]{2.1in}
		\centering
		\includegraphics[width=\columnwidth]{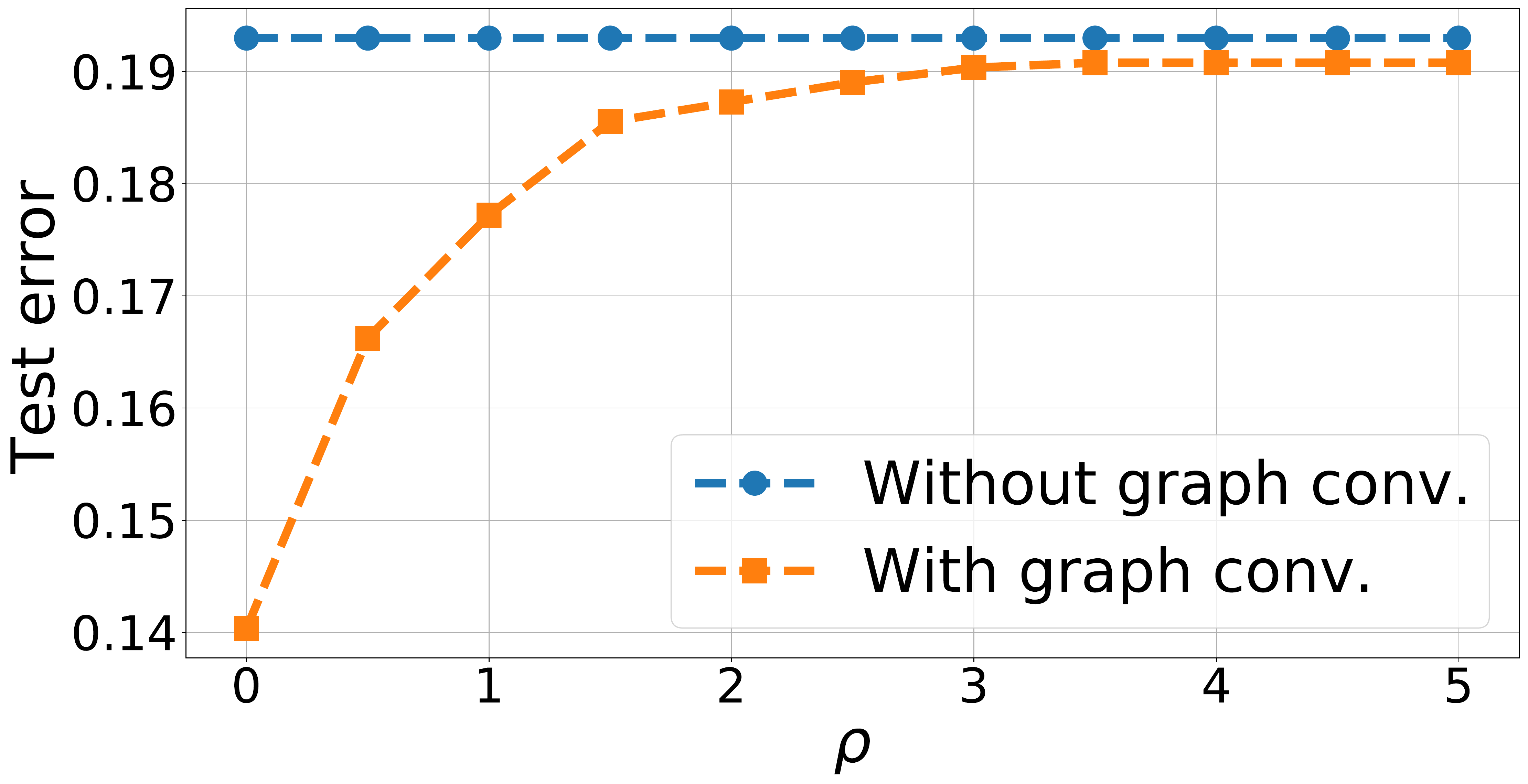}
		\caption{Wiki.Net, class $C$}\label{fig:Wiki-c}
	\end{subfigure}
	\begin{subfigure}[t]{2.1in}
		\centering
		\includegraphics[width=\columnwidth]{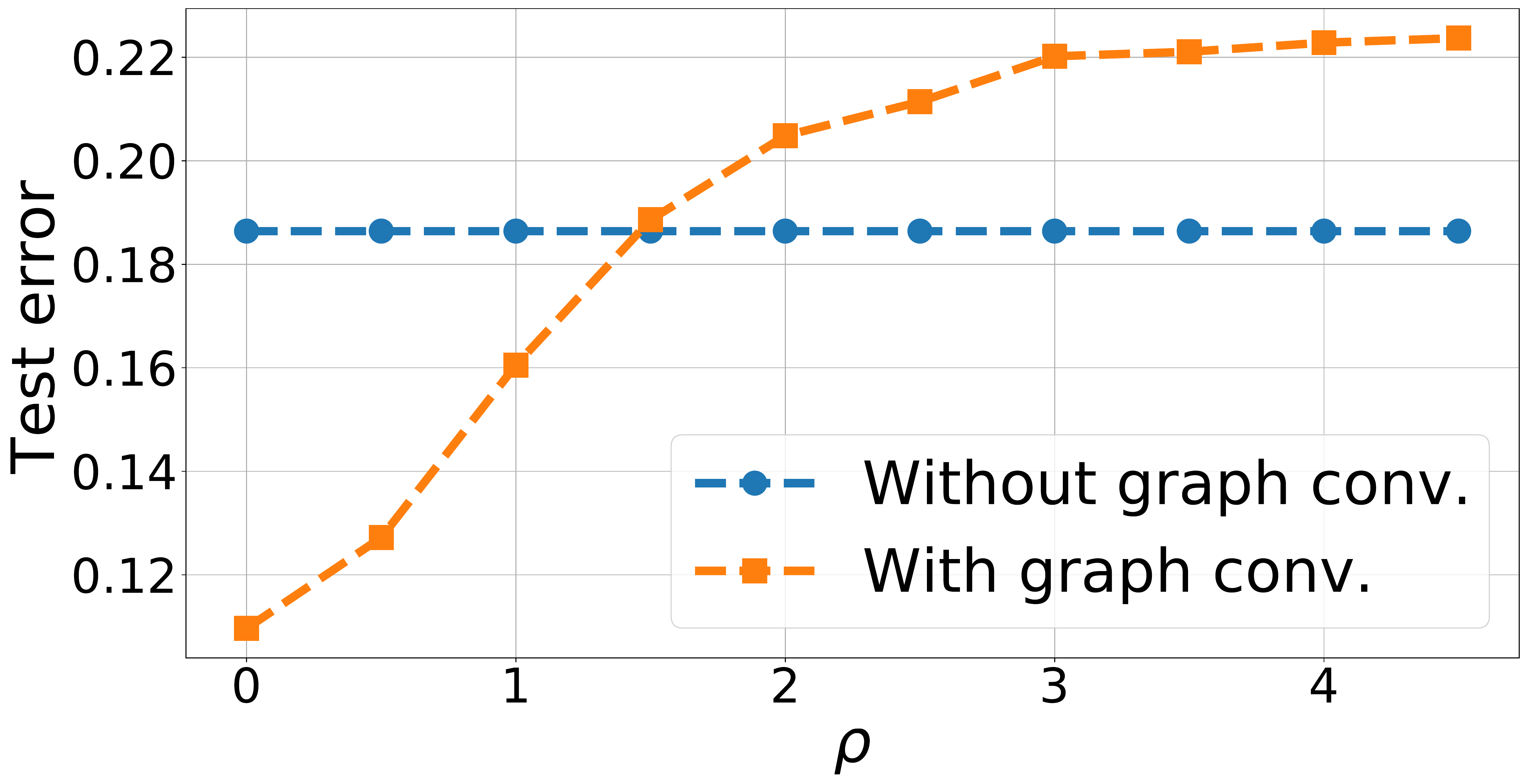}
		\caption{Wiki.Net, class $D$}\label{fig:Wiki-d}
	\end{subfigure}
	\begin{subfigure}[t]{2.1in}
		\centering
		\includegraphics[width=\columnwidth]{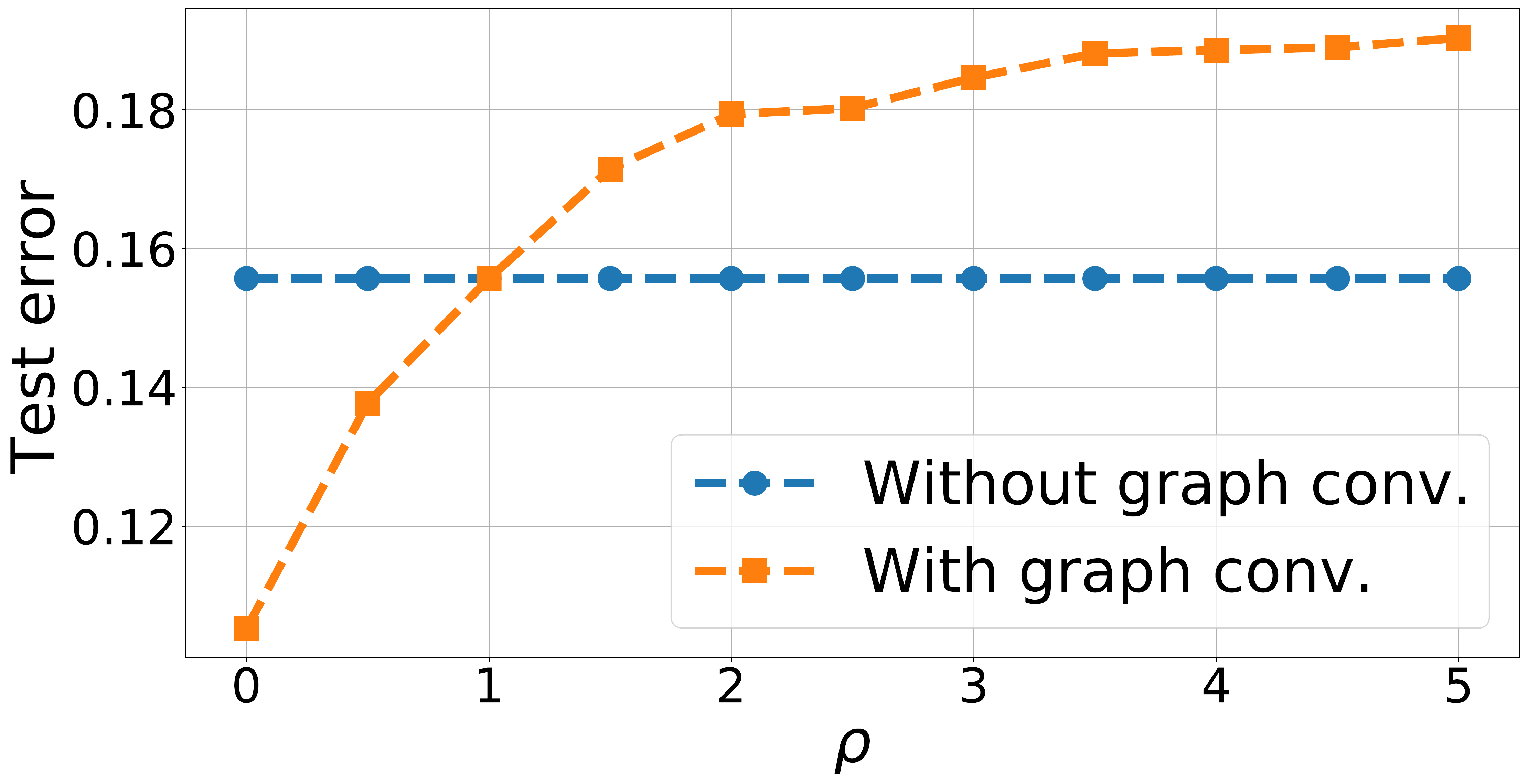}
		\caption{Wiki.Net, class $E$}\label{fig:Wiki-e}
	\end{subfigure}
	\caption{Test loss as the number of nodes increases for Wiki.Net.}\label{fig:Wiki}
\end{figure}

\end{document}